\documentclass{article}

\usepackage{fullpage}
\usepackage{amsfonts}
\usepackage{graphicx}
\usepackage[utf8]{inputenc}
\usepackage[T1]{fontenc}
\usepackage{url}
\usepackage{booktabs}
\usepackage{nicefrac}
\usepackage{microtype}
\usepackage{subfigure}
\usepackage{hyperref}
\usepackage[normalem]{ulem}
\usepackage{amsmath}
\usepackage{amssymb}
\usepackage{amsthm}
\usepackage{mathtools}
\usepackage[square,numbers,compress]{natbib}
\usepackage{amsopn}
\usepackage{bm}
\usepackage[capitalize,noabbrev]{cleveref}

\newcommand{\heaviside}{H}
\newcommand{\uniform}{U}

\def\vmu{{\bm{\mu}}}

\def\vlambda{{\bm\lambda}}
\def\vlam{{\bm\lambda}}
\def\vsig{{\bm\sigma}}

\def\va{{\bm{a}}}
\def\vb{{\bm{b}}}
\def\vc{{\bm{c}}}

\def\ve{{\bm{e}}}

\def\vh{{\bm{h}}}

\def\vm{{\bm{m}}}

\def\vq{{\bm{q}}}

\def\vu{{\bm{u}}}
\def\vv{{\bm{v}}}
\def\vw{{\bm{w}}}
\def\vx{{\bm{x}}}
\def\vy{{\bm{y}}}
\def\vz{{\bm{z}}}
\def\T{\top}

\def\mA{{\bm{A}}}
\def\mB{{\bm{B}}}
\def\mC{{\bm{C}}}
\def\mD{{\bm{D}}}

\def\mH{{\bm{H}}}

\def\mM{{\bm{M}}}

\def\mP{{\bm{P}}}

\def\mU{{\bm{U}}}
\def\mV{{\bm{V}}}
\def\mW{{\bm{W}}}

\newcommand{\mheaviside}[1]{\mH_{#1}}

\def\gD{{\mathcal{D}}}

\def\gI{{\mathcal{I}}}
\def\gJ{{\mathcal{J}}}

\def\gM{{\mathcal{M}}}
\def\gN{{\mathcal{N}}}

\def\gS{{\mathcal{S}}}
\def\Sc{\gS}

\def\gU{{\mathcal{U}}}
\def\gV{{\mathcal{V}}}
\def\gW{{\mathcal{W}}}
\def\gX{{\mathcal{X}}}

\newcommand{\R}{\mathbb{R}}
\newcommand{\st}{~~\mathrm{s.t.}~~}
\newcommand{\mathand}{~~\mathrm{and}~~}

\DeclareMathOperator*{\nullspace}{null}
\DeclareMathOperator{\sign}{sign}
\DeclareMathOperator{\rank}{rank}
\DeclareMathOperator{\vspan}{span}
\newcommand{\funcrank}[1]{\operatorname{rank}_I(#1)}
\newcommand{\epsfuncrank}[1]{\operatorname{rank}_{I,\varepsilon}(#1)}

\DeclareMathOperator{\range}{range}
\DeclareMathOperator{\Tr}{Tr}

\newcommand{\mixedvar}[2]{\mathcal{MV}\left(#1,#2\right)}
\newcommand{\mixedvarsv}[2]{\sigma_{#2}(#1)}
\newcommand{\gradcov}[1]{\mC_{#1}}

\newcommand{\setofnns}[1]{\mathcal{N}_2\left(#1\right)}
\newcommand{\densitysupp}{\mathcal{X}}
\newcommand{\interpcost}{\mathcal{I}}

\ifpdf
  \DeclareGraphicsExtensions{.eps,.pdf,.png,.jpg}
\else
  \DeclareGraphicsExtensions{.eps}
\fi

\theoremstyle{plain}
\newtheorem{theorem}{Theorem}[section]
\newtheorem{prop}[theorem]{Proposition}
\newtheorem{lemma}[theorem]{Lemma}
\newtheorem{corollary}[theorem]{Corollary}

\theoremstyle{definition}
\newtheorem{definition}[theorem]{Definition}
\newtheorem{example}[theorem]{Example}

\theoremstyle{remark}

\title{ReLU Neural Networks with Linear Layers are Biased Towards Single- and Multi-Index Models}

\author{Suzanna Parkinson\thanks{Committee on Computational and Applied Mathematics, University of Chicago, Chicago, IL 
(sueparkinson@uchicago.edu).}
\and Greg Ongie\thanks{Department of Mathematical and Statistical Sciences, Marquette University, Milwaukee, WI 
(gregory.ongie@marquette.edu).}
\and Rebecca Willett\thanks{Department of Statistics, Department of Computer Science, and Committee on Computational and Applied Mathematics, University of Chicago, Chicago, IL.}}

\begin{document}

\maketitle

\begin{abstract}
Neural networks often operate in the overparameterized regime, in which there are far more parameters than training samples, allowing the training data to be fit perfectly. That is, training the network effectively learns an interpolating function, and properties of the interpolant affect predictions the network will make on new samples. This manuscript explores how properties of such functions learned by neural networks of depth greater than two layers.
Our framework considers a family of networks of varying depths that all have the same {\em capacity} but different {\em representation costs}. 
The representation cost of a function induced by a neural network architecture is the minimum sum of squared weights needed for the network to represent the function; it reflects the function space bias associated with the architecture.
Our results show that adding additional linear layers to the input side of a shallow ReLU network yields a representation cost favoring functions 
 with low \emph{mixed variation} – that is, it has limited variation in directions orthogonal to a low-dimensional subspace and can be well approximated by a single- or multi-index model. 
This bias occurs because minimizing the sum of squared weights of the linear layers is equivalent to minimizing a low-rank promoting Schatten quasi-norm of a single ``virtual'' weight matrix.
Our experiments confirm this behavior in standard network training regimes. 
They additionally show that linear layers can improve generalization and the learned network 
is well-aligned with the true latent low-dimensional linear subspace when data is generated using a multi-index model.
\end{abstract}

\section{Introduction}
\label{sec:intro}
An outstanding problem in understanding the generalization properties of overparameterized neural networks is characterizing 
the inductive bias of various architectures -- i.e., characterizing the types of predictors learned when training networks with the capacity to represent large families of functions. 
Past work has explored this problem through the lens of {\em representation costs}.
Specifically, the representation cost of a function $f$ is the minimum sum of squared network weights necessary for the network to represent $f$. Representation costs are key to understanding how overparameterized neural networks trained with limited data are able to generalize well. For instance, imagine training a neural network to interpolate a set of training samples using weight decay regularization (i.e., $\ell^2$-regularization on the network weights); the corresponding interpolant will have low representation cost. Different network architectures are associated with different representation costs, so the network architecture will influence which interpolating function is learned, which can have a profound effect on test performance. The following key question then arises:
\textbf{How does network architecture affect which functions have minimum representation cost?} 

In this paper, we describe the representation cost associated with deep fully-connected networks having $L$ layers in which the first $L-1$ layers have linear activations and the final layer has a ReLU activation. As detailed in \Cref{sec:related}, networks related to this class play an important role in both theoretical studies of neural network generalization properties and experimental efforts.
{This is a particularly important family to study because
adding linear layers does not change the capacity or expressivity of a network, even though the number of parameters may change. This means that different behaviors for networks of different depths solely reflect the role of depth and not of capacity. \textbf{In effect, this framework isolates the effects of depth from those of expressivity.}

We show that adding linear layers to a ReLU network while using $\ell_2$-regularization (weight decay) is equivalent to fitting a two-layer ReLU network with 
nuclear or Schatten norm regularization on the innermost weight matrix and $\ell^2$-regularization on the outermost weights. 
The associated function space inductive bias corresponds to a notion of latent low-dimension structure that has close connections to multi- and single-index models, 
as illustrated in \Cref{fig:first_example}.
Specifically, we relate the function space inductive bias to the singular value spectrum of the expected gradient outer product (EGOP) matrix, where gradients are taken
with respect to the neural network inputs. We prove that the representation cost is bounded in terms of the \textit{mixed variation} and \textit{index rank} of a function, which are properties defined in terms of the EGOP singular values.
Our bounds imply that networks minimizing the representation cost must have an EGOP with low effective rank, where the rank decreases as more linear layers are added. Our numerical experiments on synthetic data show that with a moderate number of linear layers, the principal subspace of the learned function's EGOP is low-dimensional and closely approximates the principal subspace of the data-generating function's EGOP, which improves in- and out-of-distribution generalization.

\begin{figure}[ht!]
    \centering
    \subfigure[$L=2$ layers]{
    \includegraphics[width=0.3\columnwidth]{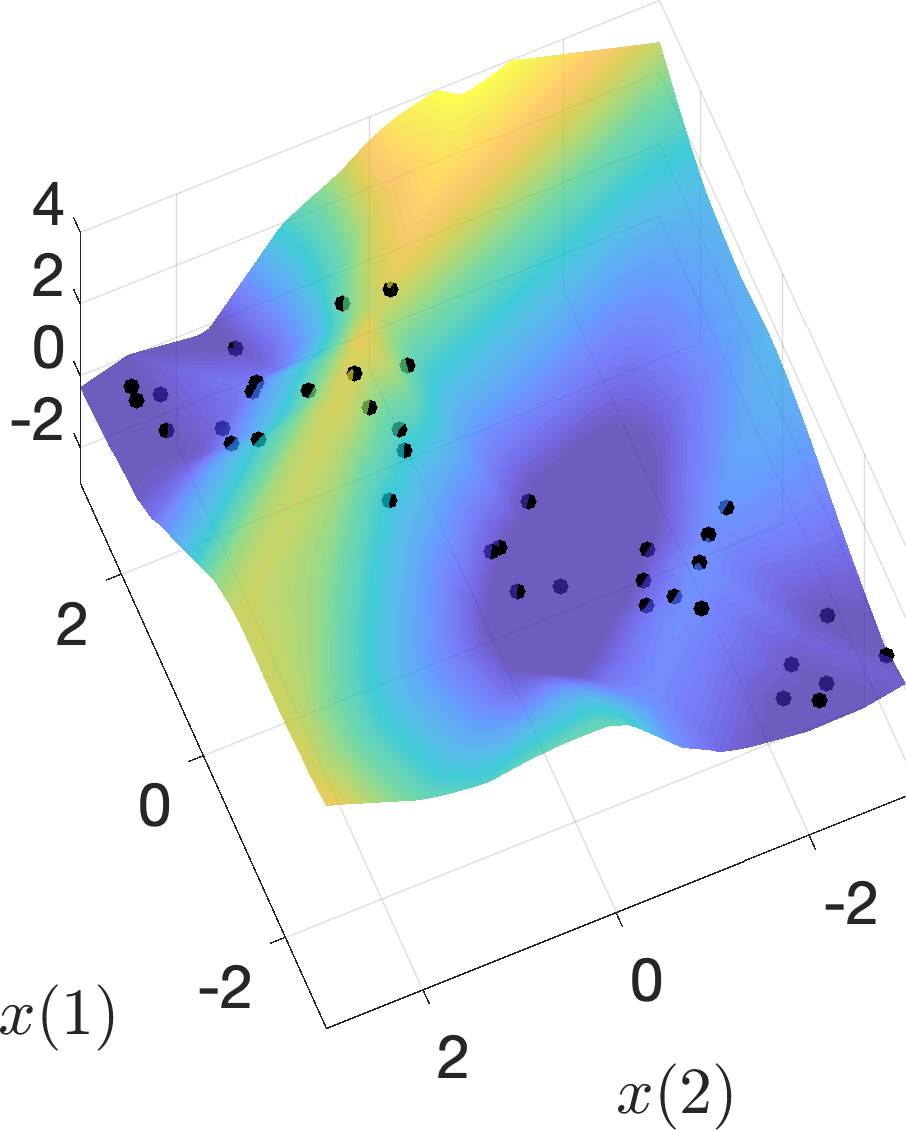}}  \hfill
    \subfigure[$L=3$ layers]{
    \includegraphics[width=0.3\columnwidth]{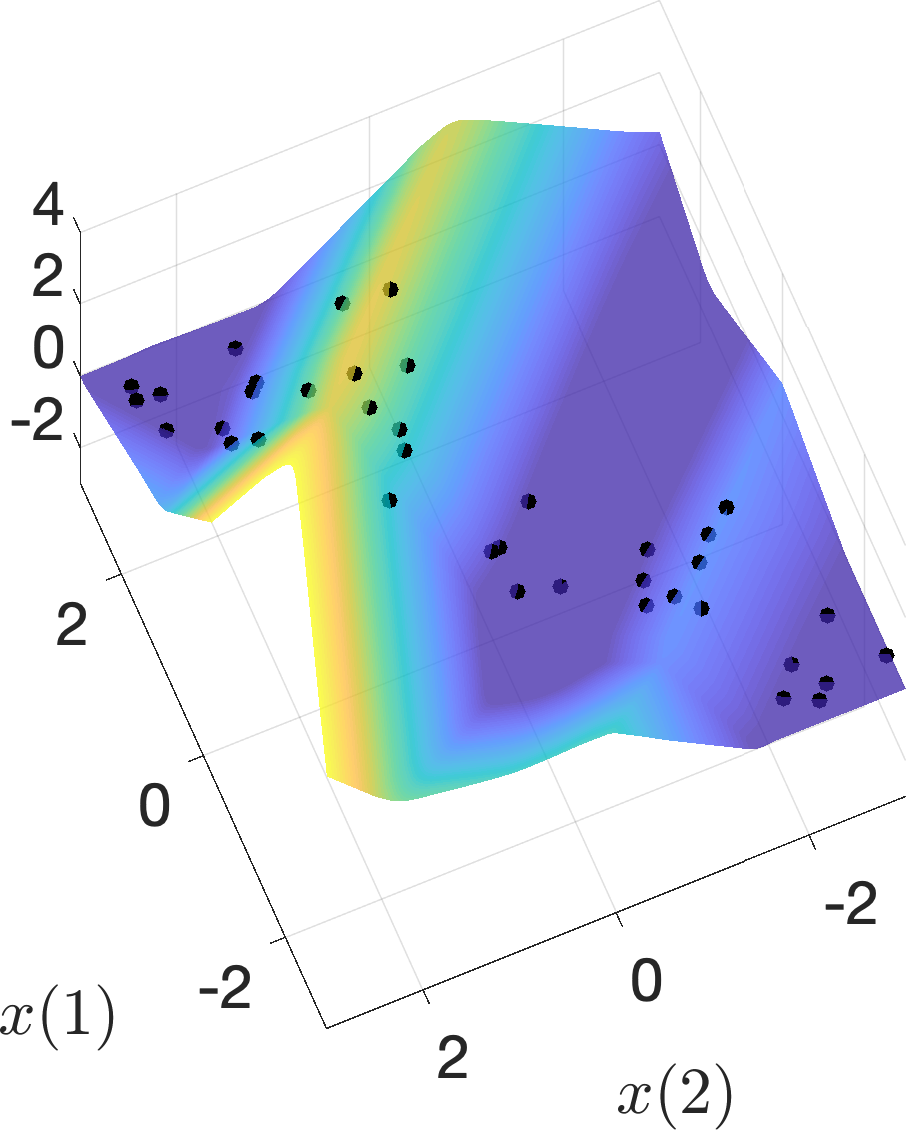}} \hfill
    \subfigure[$L=4$ layers]{
    \includegraphics[width=0.3\columnwidth]{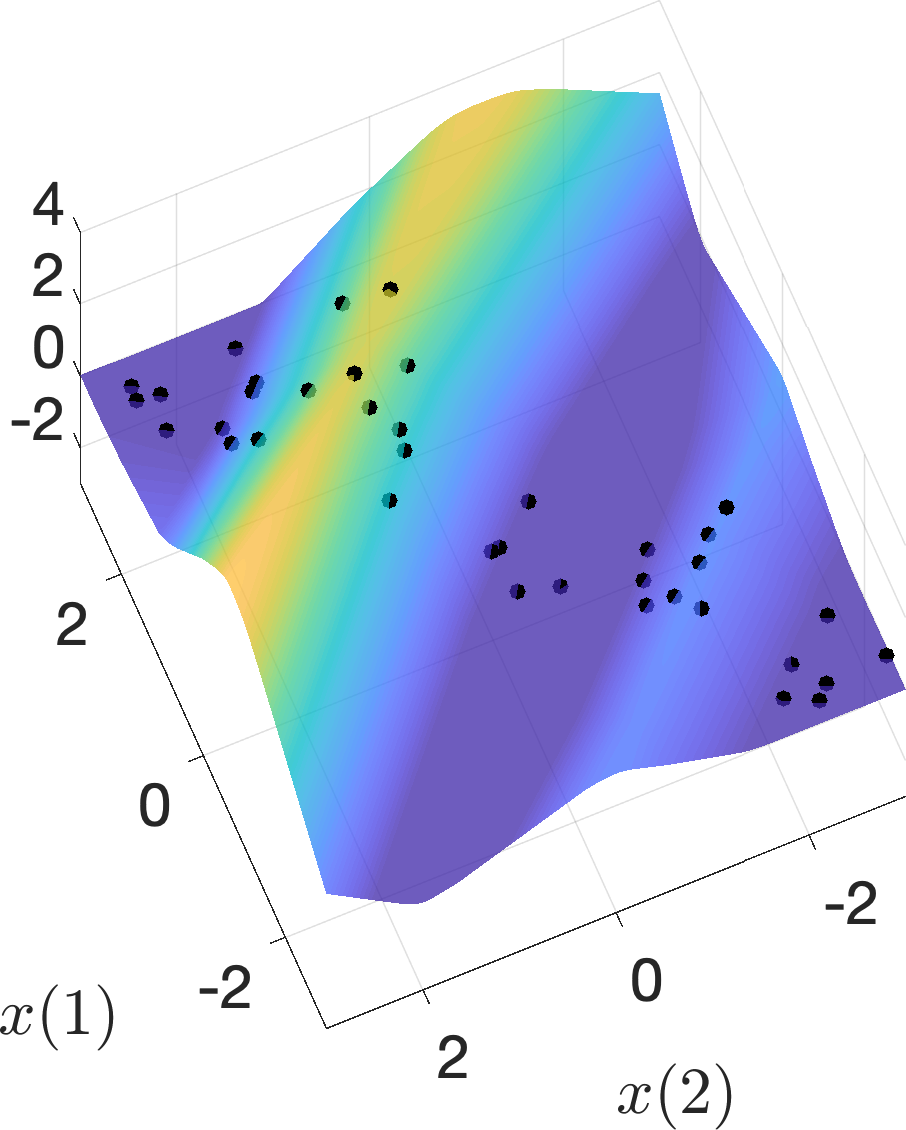}}
    \caption{\textbf{Numerical evidence that weight decay promotes unit alignment with more linear layers.} Neural networks with $L-1$ linear layers followed by one ReLU layer were trained using SGD with $\ell_2$-regularization (weight decay) to close to zero training loss on the training samples, as shown in black. Pictured in (a)-(c) are the resulting interpolating functions shown as surface plots. 
    Our theory predicts that as the number of linear layers increases, the learned interpolating function will become closer to constant in directions orthogonal to a low-dimensional subspace on which a parsimonious interpolant can be defined.}
    \label{fig:first_example}
\end{figure}

\paragraph{Contributions}
Our theoretical results show that adding linear layers to a shallow ReLU network trained with weight decay regularization results in global minimizers with low-dimensional structure, and empirically, the phenomenon persists in practical training settings in which we may not find the global minimizer. These theoretical results do not depend on the data-generating function having low-dimensional structure, contrary to past work focused on learning single- and multi-index models. Furthermore, when the data-generating function has approximate low-dimensional structure and the sample size is moderate, linear layers improve generalization in our experiments. More specifically, this
manuscript makes the following contributions:
\begin{itemize}
\item Formalizes the notions of the mixed variation and index rank of a function, establishing connections with single- and multi-index models.
\item Characterizes the representation cost as a function of the number of linear layers, and bounds this cost in terms of the function's mixed variation and index rank. 
\item Bounds the effective index rank of models that interpolate data with minimal representation cost.
\item Demonstrates empirically that training models with linear layers using standard training and optimization approaches yields models with low effective index rank and strong generalization performance. That is, linear layers are a useful form of regularization that promotes low-rank structure, which in turn can improve generalization.
\end{itemize}

\subsection{Related work}
\label{sec:related}

\paragraph{Representation costs}
In neural networks, it has been argued that ``the size [magnitude] of the weights is more
important than the size [number of weights or parameters] of the
network"  \cite{Bartlett1997}, an idea reinforced by \cite{Neyshabur14,Zhang16} and yielding insight into the generalization performance of overparameterized neural networks \citep{lyu2019gradient,nacson2019lexicographic,neyshabur2015norm,wei2018margin}. Networks trained with weight decay regularization seek weights with minimal norm required to represent a function that accurately fits the training data. Therefore, minimal norm solutions and the corresponding \emph{representation cost} of a function play an important role in generalization performance. 

The representation costs associated with shallow (i.e., two-layer) networks have been studied extensively.
In \cite{bach2017breaking}, Bach studies the variation norm, which corresponds to the representation cost associated with infinitely wide two-layer networks.
A number of papers by E, Wojtowytsch, and collaborators study the set of functions represented by finite-norm, infinitely wide two-layer networks, known as Barron space \cite{ma2019priori,e2022representation,e2022barron,wojtowytsch2024optimal}.
Savarese et al.\ \cite{savarese2019infinite} and Boursier and Flammarion \cite{Boursier_Flammarion_2023} provide a function space description of representation cost of univariate functions in the case of two-layer ReLU networks; 
Ongie et al.\ \cite{ongie2019function} extend this analysis to scalar-valued multivariate functions, while Shenouda et al.\ \cite{shenouda2023vector} consider the case of vector-valued outputs.
Parhi and Nowak \cite{parhi2021banach}, Bartolucci et al. \cite{bartolucci2023understanding}, and Unser \cite{unser2023ridges} provide representer theorems, which show that for a class of variational problems regularized using the infinite-width two-layer representation cost
there exist solutions realizable as finite-width ReLU networks. 
A line of work from Ergen and Pilanci explores the connection between two-layer representation costs and convex formulations of network training \cite{pilanci2020neural,ergen2020implicit,ergen2021revealing}.
Work by Mulayoff et al.\ \cite{mulayoff2021implicit} and Nacson et al.\ \cite{nacson2022implicit} connects the function space representation costs of two-layer ReLU networks to the stability of SGD minimizers. 
Several works by Ma, Siegel and Xu \cite{ma2022uniform,siegel2024sharp,siegel2023characterization} study shallow neural networks with ReLU$^k$ activations.

There have also been several efforts to understand the representation costs of deep non-linear networks.
Notably, Parhi and Nowak \cite{parhi2022kinds} examine deep ReLU networks with one additional \emph{linear} layer between ReLU layers, and relate the corresponding representation cost to a compositional version of the two-layer representation cost; however, an explicit characterization of the associated function space inductive bias is not given in this work. 
Jacot \cite{jacot2022implicit,jacot2024bottleneck} connects the representation costs of deep ReLU networks in the limit as the number of layers goes to infinity with certain notions of nonlinear function rank; see \Cref{sec:index-rank def} for more discussion. 
Chen \cite{chen2024neural} studies a different way to generalize Barron spaces to deep nonlinear networks as an infinite union of reproducing kernel Hilbert spaces.
Ergen and Pilanci \cite{ergen2021revealing} characterize representation cost minimizers associated with deep nonlinear networks but place strong assumptions on the data distribution (i.e., rank-1 or orthonormal training data). Additionally, recent work studies representation cost minimizers in the context of parallel deep ReLU architectures \cite{wang2022parallel}, depth-4 networks on one-dimensional data \cite{zeger2024library}, 
and path-norm regularization in place of $\ell^2$-regularization \cite{ergen2024path}.

\paragraph{Linear layers}
The inductive bias associated with fitting deep \emph{linear} networks has been studied extensively.
Gunasekar et al.\ \cite{gunasekar2018implicit} show that $L$-layer linear networks with diagonal structure 
(i.e., all weight matrices are diagonal)
induces a non-convex implicit regularization over network weights corresponding to the $\ell^q$ norm of the outer layer weights for $q = 2/L$, and similar conclusions hold for deep linear convolutional networks. Wang et al.\ \cite{wang2022parallel} show that for deep, fully-connected linear networks the associated representation cost reduces to the Schatten-$q$ penalty on a virtual single hidden layer weight matrix.
Additionally, Dai et al.\  \cite{dai2021representation} examine the representation costs of deep linear networks under various connectivity constraints from a function space perspective. Several works, including those by Ji and Telgarsky \cite{ji2018gradient}, Pesme et al.\ \cite{pesme2021implicit}, Wang and Jacot \cite{wang2024implicit}, and Even et al.\ \cite{even2023sgd}, study the (stochastic) gradient descent path of deep linear networks.

The role of linear layers in \textit{nonlinear} networks has also been explored in a number of works.
In \cite{golubeva2020wider}, Golubeva et al.\ study the role of network width when the number of parameters is held fixed; they specifically look at increasing the width without increasing the number of parameters by adding linear layers. 
This procedure seems to help with generalization
performance (as long as the training error is controlled). 
Khodak et al.\ \cite{khodak2020initialization} study how to initialize and regularize linear layers in nonlinear networks and conclude that low-rank structure emerges empirically.
One of the main contributions of this paper is an understanding of why this low-rank structure emerges and how it can improve generalization.
\sloppypar
The effect of linear layers on training speed was previously examined by Ba and Caruana \cite{ba2013deep} and Urban et al.\ \cite{urban2016deep}. Arora et al.\ \cite{arora2018optimization} consider implicit acceleration in deep nets and claim that depth induces a momentum-like term in training deep linear networks with SGD. The implicit regularization of gradient descent has been studied in the context of matrix and tensor factorization problems \cite{arora2019implicit,gunasekar2018implicit,razin2020implicit,razin2021implicit}. Similar to this work, low-rank representations play a key role in their analysis. Linear layers have also been shown to help uncover latent low-dimensional structure in dynamical systems \citep{zeng2023autoencoders}. Linear layers also play in important role in attention mechanisms and transformers \cite{vaswani2017attention}; the factoring of the key-query product matrix into two matrices can be interpreted as a linear layer, and several works have explored using linear layers to fine-tune large language models for downstream tasks \cite{hu2022lora,yaras2024compressible}.

\paragraph{Single- and multi-index models}
Multi-index models are functions of the form 
\begin{equation}
f(\vx) = g(\langle \vv_1,\vx \rangle,\langle \vv_2,\vx \rangle,\ldots,\langle \vv_r,\vx \rangle) = g(\mV^\top \vx)
\label{eq:MIM}
\end{equation}
for $\vx \in \R^d$, for some matrix $\mV := \begin{bmatrix}\vv_1 & \ldots & \vv_r\end{bmatrix} \in \mathbb{R}^{d \times r}$ with linearly independent columns,
and an unknown \textit{link function} $g:\mathbb{R}^r \rightarrow \mathbb{R}^D$. The $r$-dimensional subspace spanned by the columns of $\mV$ is often called the \emph{central subspace} associated with $f$.
Single-index models correspond to the special case where $r=1$.
(Like most work on single-index models, in this paper we assume that the output dimension $D=1$, but we generalize to the case that $D>1$ in \Cref{sec:vector valued functions}.)
Multiple works have explored learning such models (i.e., learning both the central subspace and the link function) in high dimensions \citep{bach2017breaking,ganti2017learning,ganti2015matrix,gollakota2024agnostically,kakade2011efficient,liu2020learning,xia2008multiple,yin2008successive,zhu2006fourier}. 
The link function $g$ has an $r$-dimensional domain, so the sample complexity of learning these models depends primarily on $r$ even when the dimension $d$ of the inputs is large. 
As noted in \cite{liu2020learning}, the minimax mean squared error rate for general functions $f$ defined on a $d$-dimensional input space that are $s$-H\"older smooth is $n^{-\frac{2s}{2s+d}}$, while for functions with a rank-$r$ central subspace, the minimax rate is $n^{-\frac{2s}{2s+r}}$. The difference between these rates implies that for $r \ll d$, a method that can adapt to the central subspace can achieve far smaller function estimation errors (and hence better generalization) than a non-adaptive method.

Several recent papers \cite{bach2017breaking,bietti2022learning,Damian_Lee_Soltanolkotabi_2022,mousavi2022neural,ardeshir2023intrinsic} provide bounds on generalization errors when learning single- and multi-index models using shallow neural networks. 
Bach \cite{bach2017breaking} describes learning single- or multi-index models in a function space optimization framework with the two-layer representation cost serving as a regularizer and shows that shallow neural networks can achieve the minimax estimation rate. However, this does not preclude the possibility of linear layers improving constants in generalization rates, which can have a significant impact when sample sizes are moderate.
Damien et al.\ \cite{Damian_Lee_Soltanolkotabi_2022}, Bietti et al.\ \cite{bietti2022learning}, and Mousavi et al.\ \cite{mousavi2022neural} focus on shallow networks trained via specialized variations of gradient descent or gradient flow. Contrary to the present paper, some of these works explicitly enforce single-index structure during training: Bietti et al.\ \cite{bietti2022learning} by constraining the inner weights of all hidden nodes to have the same weight vector, and Mousavi-Hosseini et al.\ \cite{mousavi2022neural} by initializing all weights to be equal and noting that gradient-based updates of the weights will maintain this symmetry.
Finally, as a negative result, Ardeshir et al.\ \cite{ardeshir2023intrinsic} prove that two-layer ReLU networks regularized with the two-layer representation cost are not well-suited to learning the parity function, which is a single-index model, suggesting that the inductive bias of the two-layer representation cost is incompatible with learning certain types of single-index models.

\paragraph{Expected Gradient Outer Products (EGOP) of neural networks}
There are several empirical works highlighting low-rank structures emerging during the training of overparameterized neural networks and hypothesizing about the role of this structure in the generalization performance of overparameterized models \cite{huh2022low,khodak2020initialization,radhakrishnan2024mechanism}.
For example, Radhakrishnan et al.\ \cite{radhakrishnan2024mechanism} examine the Expected Gradient Outer Product (EGOP) of a fitted model; specifically, for a model $f(\mathbf{x})$ the EGOP is 
\begin{equation}
\mathbb{E}_X[\nabla f(X) \nabla f(X)^\T].
\end{equation}
Their empirical study highlights how the EGOP of trained neural networks correlates with features salient to the learning task. Our work theoretically characterizes how the EGOP is influenced by linear layers in the network. Further connections between the EGOP and neural network models are explored in \cite{beaglehole2024average,radhakrishnan2024linear}. The EGOP is also central to the active subspaces dimensionality reduction technique \cite{constantine2015active,constantine2014active}, and was originally studied in the context of multi-index regression \cite{samarov1993exploring,hristache2001structure,wu2010learning,trivedi2014consistent,yuan2023efficient}.

\subsection{Outline}\label{subsec:outline}
In \Cref{sec:def} we formally define the neural network architectures we study and their representation costs. In \Cref{sec:lowrank} we define the index rank and mixed variation of a function. Our main theoretical results are in \Cref{sec:function_space}, where we connect the representation cost with index rank and mixed variation. The numerical experiments in \Cref{sec:experiments} show that our theory is predictive of practice when the data comes from a low-index-rank function. We discuss the implications and limitations of our results in \Cref{sec:discussion}.
Another expression for the representation cost can be found in \Cref{sec:simplify}. Most technical details are reserved for the remainder of the appendix. Of note, a generalization of the results to vector-valued functions can be found in \Cref{sec:vector valued functions}.

\subsection{Notation}\label{subsec:notation}
For a vector $\va\in \R^K$, we use $\|\va\|_p$ to denote its $\ell^p$ norm and $a_k$ to denote the $k$-th entry. For a matrix $\mW$, we use $\|\mW\|_{op}$ to denote the operator norm, $\|\mW\|_F$ to denote the Frobenius norm, $\|\mW\|_*$ to denote the nuclear norm (i.e., the sum of the singular values), and for $q > 0$ we use
$\|\mW\|_{\Sc^q}$ to denote the Schatten-$q$ quasi-norm (i.e., the $\ell^q$ quasi-norm of the singular values of $\mW$).
We let $\sigma_k(\mW)$ denote the $k$-th largest singular value of $\mW$ and $\vw_k$ denote row $k$ of $\mW$. Given a vector $\vlambda \in \R^{K}$, 
the matrix $\mD_{\vlambda} \in \R^{K \times K}$ is a diagonal matrix with the entries of $\vlambda$ along the diagonal. 
We write $\vlambda > 0$ to indicate that $\vlambda$ has all positive entries. 
For the weighted $L_2$-norm of a function $f:\R^d \rightarrow \R$ with respect to a probability distribution $\rho$ we write $\|f\|_{L_2(\rho)}$. 
We use $N(\mu,\sigma^2)$ for the normal distribution with mean $\mu$ and standard deviation $\sigma$ and $U(\Omega)$ for the uniform distribution over a set $\Omega$.
Finally, we use $[t]_+ = \max\{0,t\}$ to denote the ReLU activation, whose application to vectors is understood entrywise.

\section{Problem Formulation}\label{sec:def}
Let $\densitysupp \subseteq \R^d$ be either a bounded convex set with a nonempty interior or all of $\R^d$. 
Let $\setofnns{\densitysupp}$ denote the space of functions $f:\mathcal{X}\rightarrow \R$ expressible as a two-layer ReLU network having input dimension $d$; we allow the width $K$ of the single hidden layer to be unbounded. Every function in $\setofnns{\densitysupp}$ is described (non-uniquely) by a collection of weights $\theta = (\mW,\va,\vb,c)$:
\begin{align}
h_\theta^{(2)}(\vx) & = \va^\T[\mW\vx + \vb]_+ + c = \sum_{k=1}^K a_k[\vw_k^\T\vx + b_k]_+ + c
\end{align}
for some $K \in \mathbb N$, $\mW \in \R^{K\times d}, \va \in \R^{K}, \vb \in \R^K$, and $c\in \R$. We denote the set of all such parameter vectors $\theta$ by $\Theta_2$.

In this work, we consider a re-parameterization of networks in $\setofnns{\densitysupp}$. Specifically, we replace the linear input layer $\mW$  with $L-1$ linear layers:
\begin{align}
\label{eq:L layers nn model}
    h_\theta^{(L)}(\vx) & = \va^\T[\mW_{L-1}\cdots\mW_2\mW_1 \vx + \vb]_+ + c
\end{align}
where now $\theta = (\mW_1,\mW_2,...,\mW_{L-1},\va,\vb,c)$. Again, we allow the widths of all layers to be arbitrarily large. Let $\Theta_L$ denote the set of all such parameter vectors. With any $\theta \in \Theta_L$  we associate the $\ell_2$-regularization penalty:
\begin{equation}
    C_L(\theta) = \frac{1}{L}\left(\|\va\|_2^2 + \|\mW_{1}\|_F^2 + \cdots +  \|\mW_{L-1}\|_F^2\right),
\end{equation}
i.e., the squared Euclidean norm of all non-bias weights\footnote{Similar to \cite{ongie2019function}, we do not regularize the bias terms in our definition of the cost $C_L$. This simplifies the theoretical analysis; for example, our formulation makes the representation cost translation invariant, a property that is lost when one regularizes the bias terms. Though, we note, regularizing biases may change the inductive bias. For example, as shown in \cite{Boursier_Flammarion_2023}, regularizing the biases in univariate shallow ReLU networks yields unique interpolating representation cost minimizers, while uniqueness is not guaranteed when bias is unregularized.}. This type of regularization penalty is also known as \emph{weight decay} in the machine learning literature \cite{hanson1988comparing,loshchilov2017decoupled}.

Given training pairs $\{(\vx_i,y_i)\}_{i=1}^n$ with $\vx_i \in \densitysupp$ and $y_i \in \R$, and a loss function $\ell(\cdot,\cdot):\R\times \R \rightarrow [0,\infty)$, consider the problem of finding an $L$-layer network that minimizes the 
 $\ell^2$-regularized empirical risk:
\begin{equation}\label{eq:opt1}
\min_{\theta \in \Theta_L} \frac{1}{n}\sum_{i=1}^n \ell( h^{(L)}_\theta(\vx_i), y_i) + \lambda C_L(\theta),
\end{equation}
where $\lambda > 0$ is a regularization parameter. 
We may recast \eqref{eq:opt1} as an optimization problem in function space: for any $f \in \setofnns{\densitysupp}$, define its $L$-layer \emph{representation cost} $R_L(f)$ by
\begin{equation}\label{eq:RLdef}
    R_L(f) = \inf_{\theta\in \Theta_L} C_L(\theta)\st f = h^{(L)}_\theta|_\densitysupp.
\end{equation}
Then \eqref{eq:opt1} is equivalent to the function space optimization problem
\begin{equation}\label{eq:opt2}
\min_{f\in \setofnns{\densitysupp}}\frac{1}{n}\sum_{i=1}^n \ell( f(\vx_i), y_i) + \lambda R_L(f).
\end{equation}
Therefore, $R_L$ is the function space regularizer induced by the parameter space regularizer $C_L$. 

In practice, the regularization strength parameter $\lambda$ in \eqref{eq:opt2} is often taken to be sufficiently small such that the empirical risk dominates the overall cost during the early phases of training. In this case, any minimizer $f$ of \eqref{eq:opt2} will satisfy $\ell(f(\vx_i),y_i) \approx 0$ for all $i \in [n]$. Assuming this implies $f(\vx_i) \approx y_i$, we see that $f$ approximately interpolates the training data while achieving low $R_L$ cost. This motivates us to consider the minimum $R_L$-cost interpolation problem:
\begin{equation}\label{eq:opt3}
\min_{f\in \setofnns{\densitysupp}}R_L(f)\st f(\vx_i)= y_i \;\; \forall i \in [n].
\end{equation}
Informally, \eqref{eq:opt3} can be thought of as the limit of \eqref{eq:opt2} as the regularization strength $\lambda \rightarrow 0$.
\textit{One goal of this paper is to describe how the set of global minimizers to \eqref{eq:opt3} changes with $L$, providing insight into the role of linear layers in nonlinear ReLU networks.}

\subsection{Simplifying the representation cost}
Earlier work, such as \cite{savarese2019infinite}, has shown that the two-layer representation cost reduces to
\begin{align}\label{eq:R2}
R_2(f) =& \inf_{\theta \in \Theta_2} \sum_{k=1}^K |a_k| \st \|\vw_k\|_2 =1,\quad\forall k\in [K]~\text{and}~f = h^{(2)}_\theta.
\end{align}
This shows that minimizing the $2$-layer representation cost is equivalent to minimizing the $\ell^1$-norm of the outer-layer weights in the network, subject to a unit norm constraint on the inner-layer weights. Since minimizing the $\ell^1$ norm promotes sparsity, this suggests functions realizable as a sparse linear combination of ReLU units will have low $R_2$-cost, a perspective explored in many recent works \cite{bach2017breaking,savarese2019infinite,ongie2019function,parhi2021banach,Boursier_Flammarion_2023}.
\textit{A key goal of this paper is to characterize the representation cost $R_L$ for different numbers of linear layers $L \geq 3$, and identify which functions have low $R_L$ cost.
} 

As a step in this direction, we first prove the general $R_L$ cost can be re-cast as an optimization over two-layer networks, but where the cost associated with the inner-layer weight matrix $\mW$ changes with $L$:

\begin{lemma}\label{lem:schatten} 
Suppose $f\in \setofnns{\densitysupp}$. Then 
\begin{equation}\label{eq:nucnormex}
R_L(f) = \inf_{\theta \in \Theta_2} \tfrac{1}{L}\|\va\|_2^2 + \tfrac{L-1}{L}\|\mW\|^{q}_{\Sc^{q}}\st f = h_\theta^{(2)}|_\densitysupp
\end{equation}
where  $q:=2/(L-1)$ and $\|\mW\|_{\Sc^q}$ is the Schatten-$q$ quasi-norm, i.e., the $\ell^q$ quasi-norm of the singular values of $\mW$.
\end{lemma}
\begin{proof}
 The result is a direct consequence of the following variational characterization of the Schatten-$q$ quasi-norm for $q = 2/P$ with $P$ a positive integer:
\begin{equation}
\|\mW\|^{2/P}_{\Sc^{2/P}} = \min_{\mW = \mW_1\mW_2\cdots\mW_{P}} \frac{1}{P}\left(\|\mW_1\|_F^2 + \|\mW_2\|_F^2 + \cdots +  \|\mW_P\|_F^2\right),
\end{equation}
where the minimization is over all matrices $\mW_1,...,\mW_P$ of compatible dimensions. The case $P=2$ is well-known (see, e.g., \cite{srebro2004maximum}). The general case for $P \geq 3$ is established in \cite[Corollary 3]{shang2020unified}. See also \cite[Proposition 2]{wang2022parallel}.
\end{proof}
Note that Schatten-$q$ quasi-norms with $0 < q \leq 1$ are a widely used surrogate for the rank penalty \cite{nie2012low,shang2017bilinear,shang2020unified}. Intuitively, this shows that minimizing the $R_L$-cost promotes functions realizable by shallow networks having low-rank weight matrices $\mW$, or equivalently, a multi-index model with a low-dimensional central subspace.

However, one deficiency of the characterization of the $R_L$ cost given in \eqref{eq:nucnormex} is that the objective varies under different sets of parameters realizing the same function. In particular, trivial re-scalings of inner- and outer-layer weight pairs lead to different objective values. In \Cref{sec:simplify}, we derive a scale invariant form of the $R_L$-cost, similar to the characterization of the $R_2$-cost given in \eqref{eq:R2}. This characterization is used to prove our main results in \Cref{sec:function_space}.

\section{Index rank and mixed variation}
\label{sec:lowrank}

We will see that adding linear layers induces a representation cost that favors functions well-approximated by a low-dimensional multi-index model, in which case we say the function has low \emph{index rank} or low \emph{mixed variation}.
In this section, we formalize the notions of the index rank and the mixed variation of a function as well as their connections to related concepts in the literature. 

\subsection{Low-index-rank functions}\label{sec:index-rank def}
We define the index rank of a function using its expected gradient outer product (EGOP), a tool used in \emph{multi-index regression} \cite{samarov1993exploring,hristache2001structure,wu2010learning,trivedi2014consistent,yuan2023efficient} as well as the \emph{active subspaces} literature \cite{constantine2014active,constantine2015active}.  
Given a function $f: \densitysupp \rightarrow \R$ whose gradient $\nabla f$ exists almost everywhere on $\gX$, the EGOP matrix $\mC_f \in \R^{d\times d}$ is defined by
\begin{equation}
\gradcov{f} := \mathbb{E}_{X}[\nabla f(X) \nabla f(X)^\T] = \int_\densitysupp \nabla f(\vx) \nabla f(\vx)^\T \rho(\vx)\, d\vx,
\label{eq:EGOP}\end{equation}
where $\rho$ is a probability density function defined over $\gX$. For technical convenience, throughout the paper we assume that $\rho$ is strictly positive, i.e., $\rho(\vx) > 0$ for all $\vx \in \gX$. 
Note that the EGOP matrix (and all related definitions in the sequel) depend on the density $\rho$, but for ease of presentation we suppress this dependency.

An eigendecomposition of the EGOP reveals directions in which the function has large (or small) variation on average. To see this, suppose $\vv$ is a unit-norm eigenvector of $\mC_f$ with eigenvalue $\lambda$.
Then observe that
\[
\lambda = \vv^\top \mC_f \vv = \int_\gX (\vv^\top \nabla f(\vx))^2\rho(\vx)d\vx = \|\partial_{\vv} f\|_{L^2(\rho)}^2,
\]
where $\partial_{\vv} f := \vv^\top \nabla f$ denotes the directional derivative of $f$ in the direction of $\vv$. 
This shows that eigenvectors of $\mC_f$ with large eigenvalues correspond to directions for which the directional derivative of $f$ is large in a $L^2(\rho)$-norm sense. On the other hand, eigenvectors with zero eigenvalue correspond to directions for which the directional derivative of $f$ vanishes almost everywhere on $\gX$, which implies $f$ is constant in these directions, i.e., $f(\vx) = f(\vx + \sigma \vv)$ for almost all $\vx \in \gX$ and $\sigma \in \R$. In particular, if the EGOP $\mC_f$ is low-rank, this implies $f$ is constant in directions orthogonal to a low-dimensional subspace. This observation motivates the following definition:

\begin{definition}[Index rank]
\label{def:rank}
We define the \emph{index rank} of a function, denoted $\funcrank{f}$, as the rank of its EGOP matrix
$\gradcov{f}$. 
\end{definition}

Additionally, we use the term \emph{principal subspace} to refer to the range of $\mC_f$, which coincides with the span of eigenvectors of $\mC_f$ associated with non-zero eigenvalues. Therefore, the index rank of a function coincides with the dimension of its principal subspace. 

Our definition of index rank is closely related to multi-index models \eqref{eq:MIM}. 
To see this, note that if $f:\densitysupp \rightarrow \R$ is a multi-index model of the form $f(\vx) = g(\mV^\top \vx)$, then $\nabla f(\vx) = \mV \nabla g(\mV^\top \vx)$ and so 
\begin{equation}
\gradcov{f} = \mV \mathbb{E}_X\left[ \nabla g(\mV^\top X) \nabla g(\mV^\top X)^\T \right] \mV^\top.
\end{equation}
This implies that the principal subspace of $f$ will lie within its central subspace, and will be equal to the central subspace if $\mathbb{E}_X\left[ \nabla g(\mV^\top X) \nabla g(\mV^\top X)^\T \right]$ is full rank.

We note that the index rank is distinct from other notions of nonlinear function rank proposed by Jacot in \cite{jacot2022implicit,jacot2024bottleneck}. 
Specifically, Jacot defines
the \emph{Jacobian rank} as 
$
\max_{\vx} \rank(Jf(\vx))$ 
where $Jf$ is the Jacobian of $f$ and the
\emph{bottleneck rank} as 
the smallest integer $k$ such that $f$ can be factorized as $f = h \circ g$
with inner dimension $k$ where $h$ and $g$ are continuous and piecewise linear.
These notions of nonlinear rank are connected to deep ReLU network representation costs.
All three notions of rank (index, Jacobian, and bottleneck) capture different kinds of nonlinear low-dimensional structure.}
Notably,
both the Jacobian and bottleneck ranks require that any function $f$ mapping to a scalar must be rank-1, regardless of any latent structure in $f$, and so only vector-valued functions can have rank greater than 1. In contrast, our definition assigns scalar-valued functions different ranks depending on the dimension of its principal subspace. We also discuss the extension of index rank to vector-valued functions in \Cref{sec:vector valued functions}.

Finally, we also note that learning an index-rank-$r$ function can be very different from the common practice of first reducing the dimension of the training features by projecting them onto the top $r$ principal components of the training features and then feeding the reduced-dimension features into a neural network; that is, the principal subspace of the EGOP may be quite different from the features' PCA subspace. Furthermore, as we detail in later sections, assuming the data is generated according to a multi-index model, the representation cost associated with adding linear layers promotes learning low-index-rank functions whose principal EGOP subspace is aligned with the central subspace; this is illustrated in \Cref{fig:PCA_compare}.  
\begin{figure}[ht!]
    \centering
    \subfigure[]{
    \includegraphics[width=0.20\columnwidth]{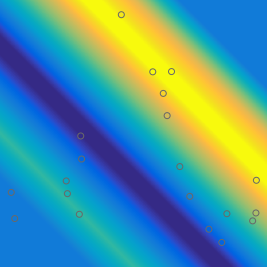}} \hfill
    \subfigure[]{
    \includegraphics[width=0.20\columnwidth]{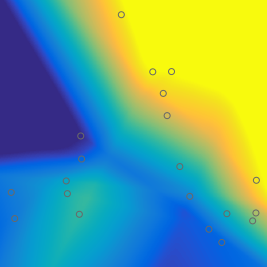}} \hfill
    \subfigure[]{
    \includegraphics[width=0.20\columnwidth]{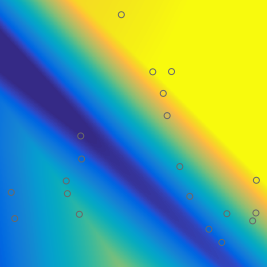}} \hfill
    \subfigure[]{
    \includegraphics[width=0.20\columnwidth]{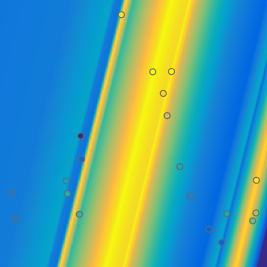}}
    \caption{\small \textbf{Illustration of learning a low-index-rank function.} (a) Heatmap of a rank-1 data generating function $f:\R^2 \rightarrow \R$ and locations of training samples. (b) Interpolant learned with $L=2$ layers, which does not exhibit index-rank-1 structure. (c) Interpolant learned with $L=4$ layers, which closely approximates the index-rank-1 structure of the data-generating function. (d) Result of performing PCA on training features to reduce their dimension to one, followed by learning with $L=2$ layers. Because the PCA subspace depends on the geometry of the training features and not on the geometry of the function, PCA cannot discover the correct principal subspace.} This illustration highlights how the addition of linear layers promotes learning single-index models with a central subspace that may differ significantly from the features' PCA subspace. 
    \label{fig:PCA_compare}
\end{figure}

\subsection{Mixed variation of a function}
\label{sec:mixed}
Performing an eigendecomposition on $\gradcov{f}$ and discarding small eigenvalues yields an eigenbasis for a low-dimensional subspace that captures directions along which $f$ has large variation. 
If the columns of a matrix $\mV \in \R^{d\times r}$ represent this eigenbasis, then 
$f(\vx) \approx f(\vx+\vu)$ for all 
$\vx\in \densitysupp$ and all
$\vu \in \range(\mV)^\perp$. 
Such functions are ``approximately low-index-rank''. In this section, we introduce a notion of \textit{mixed variation} to formalize and quantify this idea. 

\textit{Mixed variation function spaces} are informally defined in \citep{donoho2000high} to contain functions that are more regular in some directions than in others, and Parhi and Nowak
\cite{parhi2022near} provide examples of neural networks adapting to a type of mixed variation. In this paper, we formally define the mixed variation of a function
as follows:
\begin{definition}[Mixed variation] 
\label{def:mv}
For any $q > 0$, define the \emph{order $q$ mixed variation of $f$} to be the Schatten-$q$ (quasi-)norm of the matrix square-root of the EGOP:
\begin{equation}
\mixedvar{f}{q} := \|\gradcov{f}^{1/2}\|_{S^q}
\end{equation}
\end{definition}
Note that by defining the mixed variation in terms of the square root of the EGOP matrix we ensure the mixed variation is a 1-homogenous functional, i.e., $\mixedvar{\alpha f}{q} = |\alpha| \mixedvar{f}{q}$ for all $\alpha \in \R$. Also, since for any matrix $\mM$ we have $\|\mM\|_{S^q}^q \rightarrow \text{rank}(\mM)$ as $q\rightarrow 0$, we see that $\mixedvar{f}{q}^q \rightarrow \funcrank{f}$ as $q \rightarrow 0$. 

As illustrated in \Cref{fig:MV_illustration}, functions may be full-index-rank according to \Cref{def:rank} but still have small mixed variation when they are ``close'' to having lower index rank 
because they vary significantly more in one direction than another, consistent with the notions from \citep{donoho2000high,parhi2022near}.

\begin{figure}[ht!]
    \centering
    \subfigure[]{
    \includegraphics[width=0.20\columnwidth]{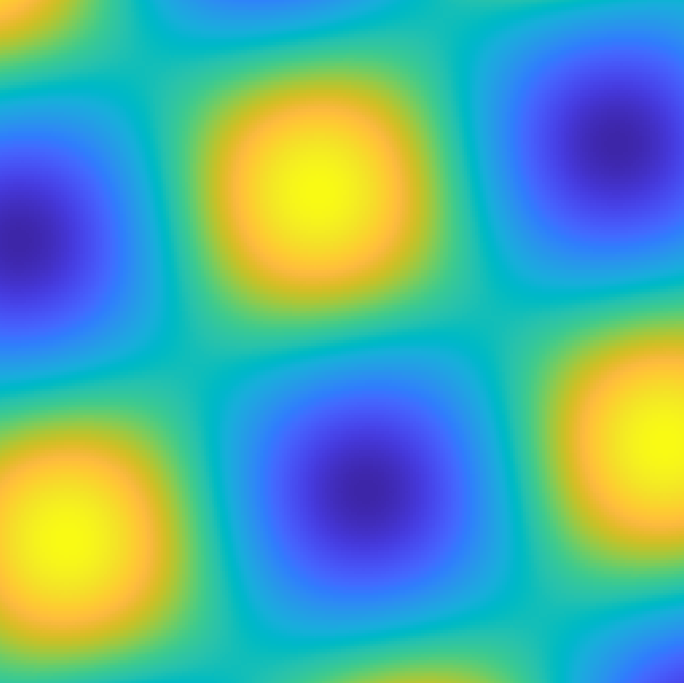}} \hfill
    \subfigure[]{
    \includegraphics[width=0.20\columnwidth]{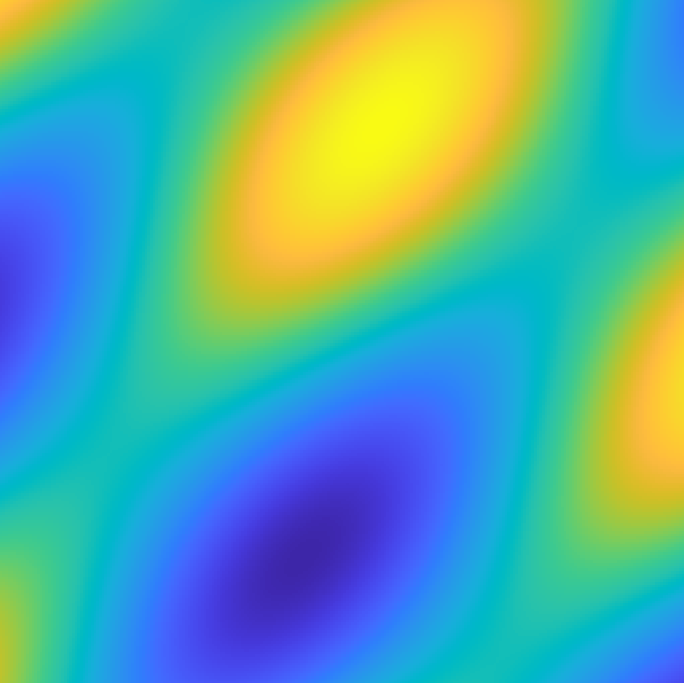}}\hfill
    \subfigure[]{
    \includegraphics[width=0.20\columnwidth]{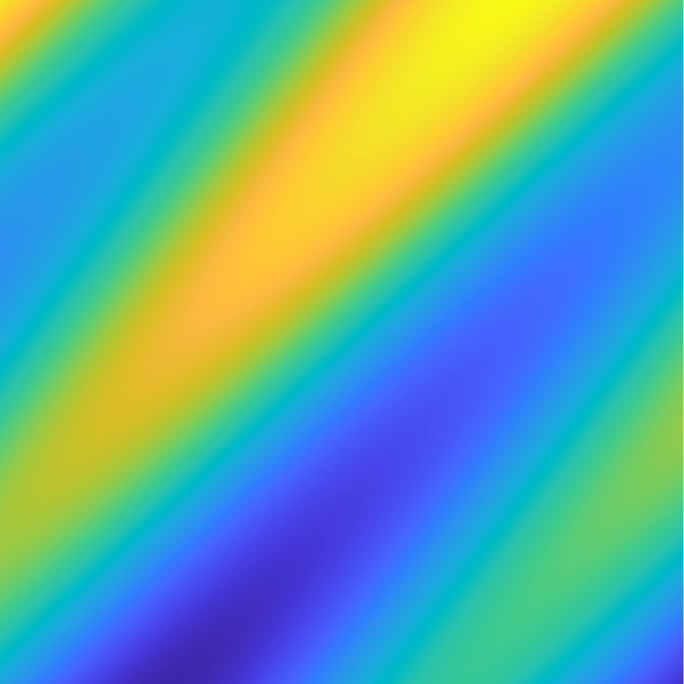}}\hfill
    \subfigure[]{
    \includegraphics[width=0.20\columnwidth]{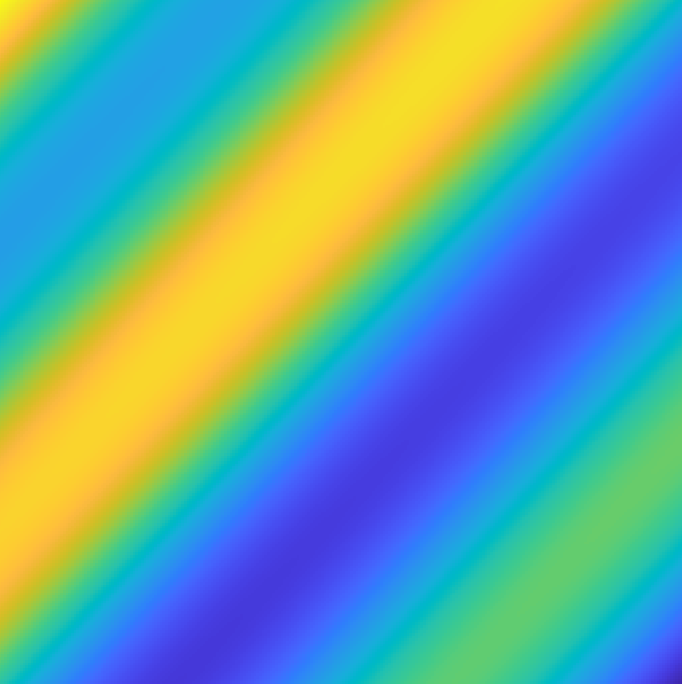}}
    \caption{\small \textbf{Illustration of four functions $f:\R^2 \rightarrow \R$ with mixed variation (\Cref{def:mv}) decreasing from left to right.} All four functions are index rank 2 according to \Cref{def:rank}, but the functions on the right with smaller mixed variation are closer to being index rank 1 
    because they vary significantly more in one direction than another.}
    \label{fig:MV_illustration}
\end{figure}

\section{The inductive bias of the $R_L$ cost}
\label{sec:function_space}
In this section, we show that minimizing the $R_L$ cost promotes learning functions that 
are nearly low-index-rank and are ``smooth'' along their principal subspace. 
Specifically, \Cref{thm:relationship between RL R2 index rank and mixed var} highlights how the relative importance of low dimensional structure versus smoothness changes with the number of linear layers. Thus, the number of linear layers in a model should be treated as a tunable hyperparameter at training time. 
In \Cref{cor:limiting cost theorem,cor:low high theorem} we further analyze how the $R_L$ cost increasingly prioritizes low-rank structure as $L$ increases. In \Cref{thm:effective rank bound} we provide bounds on the effective index rank of networks trained by minimizing the $R_L$ cost. Omitted proofs of the results in this section can be found in \Cref{app:proofs of function space results}.

\subsection{Index rank, mixed variation, and the $R_L$-cost}
We begin by establishing a theorem that relates the $R_L$ cost of a function $f$ to its index rank, mixed variation, and $R_2$ cost. This theorem underscores that low-rank structure and smoothness both influence the $R_L$ cost, but their relative importance depends on $L$. In this context, we measure low-rank structure by the index rank or mixed variation of a function, and we measure smoothness via the $R_2$ cost.
\begin{theorem}
\label{thm:relationship between RL R2 index rank and mixed var}
Let $f \in \setofnns{\densitysupp}$ and $L\geq 2$. Then
\begin{equation*}
    \max\left(
        \mixedvar{f}{\tfrac{2}{L}}^{2/L}, 
        R_2(f)^{2/L}
    \right)
    \leq R_L(f)
    \leq \funcrank{f}^{\frac{L-2}{L}} R_2(f)^{2/L}.
\end{equation*}
\end{theorem}
The proof of this theorem is given in \Cref{sec:proof of central lemma}.
The upper bound tells us that a function $f$ with both low index rank and low $R_2$ cost will have a low $R_L$ cost. Consider methods that explicitly learn a single-index or multi-index model to fit training data \citep{bietti2022learning,cohen2012capturing,ganti2015matrix,ganti2017learning,liu2020learning,mousavi2022neural,yin2008successive,zhu2006fourier}; such methods, by construction, ensure that $f$ has low index rank and has a smooth link function. Thus \Cref{thm:relationship between RL R2 index rank and mixed var} shows that such methods also control the $R_L$ cost of their learned functions.
Furthermore, the lower bound guarantees that if we minimize the $R_L$ cost during training, then the corresponding $R_2$ cost and mixed variation cannot be too high. That is, $R_L$-cost minimizers will be smooth in the $R_2$ sense, and will have low mixed variation.

Observe that the relative importance of the $R_2$ cost and the index rank or mixed variation in the bounds above changes with $L$: as $L$ increases, the terms $\gM\gV(f,\frac{2}{L})^{2/L}$ and $\funcrank{f}^{(L-2)/L}$ both tend towards $\funcrank{f}$, while $R_2^{2/L}$ tends to one.
This suggests that low-index-rank structure greatly influences the $R_L$ cost as $L$ increases. 
In fact, taking the limit as $L$ tends to infinity, we have the following direct corollary of \Cref{thm:relationship between RL R2 index rank and mixed var}:
\begin{corollary}\label{cor:limiting cost theorem}
    Let $f \in \setofnns{\densitysupp}$. Then
    \begin{equation}
        \lim_{L\rightarrow\infty} R_L(f) =\funcrank{f}.
    \end{equation}
\end{corollary}

Even without taking limits, 
given a low-index-rank function and a high-index-rank function, for large enough $L$
the low-index-rank function will have lower $R_L$ cost. This idea is formalized in the following corollary of \Cref{thm:relationship between RL R2 index rank and mixed var}.
\begin{corollary}\label{cor:low high theorem}
    For all $f_l, f_h \in \setofnns{\densitysupp}$ such that $\funcrank{f_l} < \funcrank{f_h}$, there is a value $L_0$ such that $L > L_0$ implies $R_L(f_l) < R_L(f_h)$.
\end{corollary}

Note that \Cref{cor:low high theorem} holds even when $R_2(f_h) < R_2(f_l)$. 
This has implications for interpolating $R_L$-cost minimizers. For example, suppose $f_l$ and $f_h$ both interpolate the training data, with $\funcrank{f_l} < \funcrank{f_h}$, but $f_h$ is an $R_2$-minimizing interpolant. Then \Cref{cor:low high theorem} implies there exists an $L_0$ such that for all $L > L_0$ we have have $R_L(f_h) > R_L(f_l)$, which implies $f_h$ \emph{cannot} be an $R_L$-minimizing interpolant for all $L \geq L_0$. In the next subsection, we describe this effect more quantitatively by providing bounds on the (effective) index rank of interpolating $R_L$-cost minimizers. 

\subsection{Trained networks have low effective index rank}\label{subsec:interp}
\Cref{thm:relationship between RL R2 index rank and mixed var} has implications for the decay of the singular values of EGOP of trained networks, and thus for their \emph{effective} index rank. In this section, we focus on networks that interpolate the data and minimize the $R_L$ cost, but generalize to other idealized learning rules based on finding (near-)global minimizers in \Cref{sec:extension of rank bounds}. 

To simplify the statement of our results, we first define the \emph{singular values of a function $f:\gX\rightarrow \R$}, as $\sigma_k(f) = \sigma_k(\mC_f^{1/2})$ for all $k\in[d]$, i.e., we identify the singular values of a function with the singular values of the square root of the EGOP matrix. Note that the index rank of $f$ is the number of non-zero singular values of $f$, while the order $q$ mixed variation of $f$ is the $\ell^q$ (quasi-)norm of the singular values of $f$. We also define the $\varepsilon$-effective index rank of $f$ in terms of its singular values as follows:

\begin{definition}[Effective index rank]
    Given a function $f:\densitysupp \rightarrow \R$ and a threshold $\varepsilon > 0$, define the $\varepsilon$-effective index rank of $f$, denoted by $\epsfuncrank{f}$, to be the number of singular values of $f$ larger than $\varepsilon$. That is, 
    \begin{equation}
    \epsfuncrank{f} := \left|\{k : \sigma_k(f) > \varepsilon\}\right|.
    \end{equation}
\end{definition}

Below, we bound the effective index rank of minimum $R_L$-cost interpolating solutions, which applies even when the data is not generated by a low-index-rank function. To do so, we define the \emph{interpolation cost} associated with a collection of training data:
\begin{definition}[Interpolation cost]
    Given training data $\gD = \{(\vx_i,y_i)\}_{i=1}^n$ and a rank cutoff $s$, define its rank-$s$ interpolation cost by
    \begin{equation}
    \interpcost_s(\gD) = \min_{f \in \setofnns{\densitysupp}} R_2(f)\st \funcrank{f}\leq s,~f(\vx_i) = y_i \; \forall i \in [n].
    \end{equation}
    i.e., $\interpcost_s(\gD)$ is the minimum $R_2$-cost needed to interpolate the data with a function of index rank at most $s$.
\end{definition}

Provided the training features $\{\vx_i\}_{i=1}^n$ are distinct, the interpolation cost $\gI_s(\gD)$ is always well-defined for all $s \in [d]$. This is because an interpolant of index-rank one always exists. See \Cref{fig:low rank interpolants exist} for an example, and \Cref{sec:rank-r interpolant always exists} for proof of this claim. 

Now we give our main theorem in this section, which shows that interpolants minimizing the $R_L$ cost have effective index ranks that decay with $L$.
\begin{theorem}[Effective index ranks of minimal-cost interpolants.]\label{thm:effective rank bound}
    Assume that $\hat{f}_L$ is an $R_L$-minimal interpolant of the training data $\gD$ for some $L\geq 2$ (i.e., $\hat{f}_L$ is a minimizer of \eqref{eq:opt3}).
    Then given any $\varepsilon > 0$, we have the following bound on the $\varepsilon$-effective index rank of $\hat{f}_L$:
    \begin{equation}\label{eq:effrankbnd}
        \epsfuncrank{\hat{f}_L} \le 
        \min_{s \in [d]}
        \left\lfloor 
        s 
        \left(\frac{\interpcost_s(\mathcal D)}{\varepsilon s}\right)^{\frac{2}{L}}
        \right\rfloor.
    \end{equation}
    Additionally, there exists an $\varepsilon^*>0$ independent of $L$ such that for all $0 < \varepsilon \leq \varepsilon^*$ we have $\epsfuncrank{\hat{f}_L} \geq 1$ for all $L \geq 2$.
\end{theorem}
Generalizations of \Cref{thm:effective rank bound} to interpolating functions that are near minimizers of $R_L$-cost and to functions that minimize the $R_L$-regularized empirical risk \eqref{eq:opt2} are given in \Cref{sec:extension of rank bounds}.

\begin{figure}[ht!]
\label{fig:low rank interpolants exist}
    \centering
    \includegraphics[width=\columnwidth]{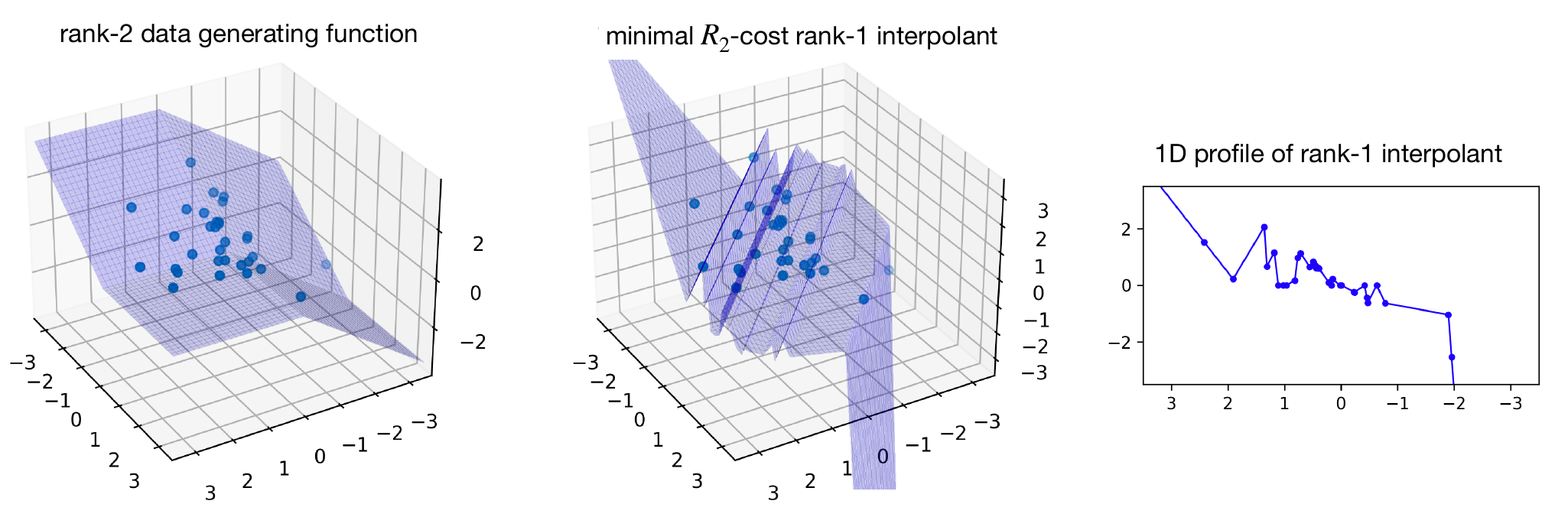}
    \caption{\textbf{Existence of rank deficient interpolants}. Left panel shows 32 training samples generated by the index-rank-2 function $f^*(x_1,x_2) = [x_1]_+-[x_2]_+$, for which $R_2(f^*) = 2$. Middle panel shows $f_{1}$, the estimated minimal $R_2$-cost index-rank-one interpolant of the training samples, for which $R_2(f_1) \approx 287.5$. Right panel shows the 1D profile of the rank-one interpolant in the middle panel.}
\end{figure}
The bounds in \Cref{thm:effective rank bound} have several implications in the case that the data is generated by a function $f^*$ that has index rank $r$ and finite $R_2$ cost. 
First, by considering the case $s=r$ in the bound \eqref{eq:effrankbnd} and using the fact that $\interpcost_r(\gD) \le R_2(f^*)$, we have the following direct corollary of \Cref{thm:effective rank bound}:
\begin{corollary}\label{cor:lowrank}
Suppose the training data $\gD$ is generated by a function $f^* \in \setofnns{\densitysupp}$ with $\funcrank{f^*} = r$.
Let $\hat{f}_L$ be an $R_L$-minimal interpolant of the training data $\gD$. Fix any $\varepsilon > 0$. Suppose $L\geq 2$ is such that $R_2(f^*) < \varepsilon r \left(1 + \frac{1}{r}\right)^{\frac{L}{2}}$.
Then $\epsfuncrank{\hat{f}_L} \leq r$.
\end{corollary}

The above corollary shows that for sufficiently large $L$, the minimum $R_L$-cost interpolant always has effective index rank bounded above by the index rank of the data generating function, independent of the number of training samples.
However, if the number of training samples is small, it is possible that an interpolant of index rank $s < r$ and small $R_2$ cost exists. In this case, the bounds in \Cref{thm:effective rank bound} imply $\epsfuncrank{\hat{f}_L} < r$ for sufficiently large $L$, i.e., $\hat{f}_L$ is rank deficient in the sense that its effective index rank is smaller than the index rank of the data generating function. In fact, \Cref{thm:effective rank bound} implies that for all sufficiently small values of $\varepsilon$ there exists a sufficiently large $L$ such that $\epsfuncrank{\hat{f}_L} = 1$, regardless of the index rank of the data generating function. 

Nevertheless, in our experiments training with standard gradient-based optimization techniques and using moderate values of $L$ (e.g., between 3 and 9), we never observed rank-deficient models; see \Cref{sec:experiments} for more details. 
Instead, we frequently observed that trained models had an effective index rank between the true rank of the task and the ambient dimension. Moreover, models trained with small amounts of label noise and a well-chosen depth $L$ almost always had effective index ranks exactly equal to the true rank; see \Cref{fig:active subspace err} for illustration.

\subsection{Index rank separation of $R_2$-cost minimizers and $R_L$-cost minimizers}
    The above results show that adding linear layers biases representation cost minimizers towards low-index-rank functions, and if the number of layers is sufficiently large, to an index-rank-one function. However, from these results alone it is unclear whether representation cost minimizers without additional linear layers (i.e., $R_2$-cost minimizers) will also exhibit this bias when the labels are generated by a low-index rank function. 
    Applying \cref{thm:relationship between RL R2 index rank and mixed var} with $L=2$ implies that the $R_2$-cost is bounded below by the mixed-variation of order 1. This suggests some amount of bias towards low-index-rank functions. Nevertheless, the examples below show that the bias induced by the $R_2$-cost is not always sufficiently strong to learn a low-index-rank function:
    there are datasets generated by an index-rank-one function such that the interpolating $R_2$-cost minimizer is not index-rank-one, while $R_L$-cost minimizers are nearly index-rank-one for large enough $L$.
    
    \begin{example}\label{ex:tworays}
    Consider the dataset $\gD$ consisting of three training pairs
    \[\gD = \{(\bm 0,0),(\vw_+,1),(\vw_-,1)\},\]
    where $\vw_+ =[\cos(\phi),\sin(\phi)]^\T$ and $\vw_- = [-\cos(\phi),\sin(\phi)]^\T$ with $0 < \phi < \pi/6$. Notice that $\gD$ is generated by the index-rank-one function $f^*(\vx) = |x_1|/\cos(\phi)$ with $R_2(f^*) = 2/\cos(\phi)$. However,  $\hat{f}_2(\vx) = [\vw_+^\T\vx]_+ + [\vw_-^\T\vx]_+$ is the unique minimal $R_2$-cost interpolant (see \Cref{sec:sm:tworays}, for proof), which has index rank 2. Additionally, if the domain $\gX \subseteq \R^2$ is a Euclidean ball centered at the origin or all of $\R^2$, and $\rho$ is any radially symmetric probability density function on $\gX$, direct calculations show $\sigma_2(\hat{f}_2) > \sin(\phi)$ (see \Cref{sec:sm:tworays}). Hence, for any $0 < \varepsilon \leq \sin(\phi)$ we have $\epsfuncrank{\hat{f}_2} = 2$, while the bound in \Cref{cor:lowrank} implies $\epsfuncrank{\hat{f}_L} = 1$ for all $L>3\log(\tfrac{1}{\varepsilon})+4$.
    \end{example}

    \begin{example} Another example is provided by results in \cite{ardeshir2023intrinsic} which studies $R_2$-cost of functions that interpolate samples of the \emph{parity function} $\chi:\{-1,1\}^d\rightarrow \R$ defined by $\chi(\vx) = \prod_i x_i$. The parity function is realizable as an index-rank-one shallow ReLU network of the form $\chi(\vx) = \phi(\bm 1^\T\vx)$ where $\bm 1 \in \R^d$ is the vector of all ones and $\phi \in \gN_2(\R)$ is a sawtooth function. In \cite{ardeshir2023intrinsic} it is proved that any index-rank-one interpolant of the parity dataset $\gD = \{(\vx,\chi(x)) : \vx \in \{-1,1\}^d\}$ must have $R_2$ cost scaling as $\Theta(d^{3/2})$, but there exist shallow ReLU networks interpolating the parity dataset with $R_2$ cost scaling as $\Theta(d)$. Therefore, for sufficiently large dimensions $d$, no interpolating $R_2$-cost minimizer $\hat{f}_2$ of the parity dataset can be index-rank-one. In particular, there exists an $\varepsilon_0 > 0$ such that for all $\varepsilon \leq \varepsilon_0$ we have $\epsfuncrank{\hat{f}_2} > 1$. On the other hand, \Cref{cor:lowrank} implies that any interpolating $R_L$-cost $ \hat{f}_L$ satisfies $\epsfuncrank{\hat{f}_L} = 1$ for sufficiently large $L$.
    \end{example}
    
\section{Numerical Experiments}
\label{sec:experiments}
To understand the practical consequences of the theoretical results in the previous section, we perform numerical experiments in which we train neural networks of the form \eqref{eq:L layers nn model} with varying values of $L$ on simulated data where the ground truth is a low-index-rank function. 
More specifically, we create an index-rank-$r$ function $f: \R^{20} \rightarrow \R$ by randomly generating $\va,\vb \in \R^{21}$ and a rank-$r$ matrix $\mW \in \R^{21\times20}$. 
Under this setup, the function 
$f(\vx) = \va^\T[\mW\vx + \vb]_+$ is an index-rank-$r$ function whose principal subspace is $\range(\mW^\top)$ (or, one could also say that $f$ is a single- or multi-index model with central subspace $\range(\mW^\top)$).
For $r=1, 2$ and $5$, we generate training datasets $\{(\vx_i, f(\vx_i) + \sigma \varepsilon_i)\}_{i=1}^n$ of size $n$ where $\vx_i \sim \uniform([-\frac{1}{2},\frac{1}{2}]^{20})$, $\varepsilon_i \sim N(0,1)$, and the label noise standard deviation $\sigma$ is either $0,0.25,0.5,$ or $1$. For several different values of training samples $n$, we train neural networks of the form \eqref{eq:L layers nn model} by minimizing the $\ell_2$-regularized empirical risk \eqref{eq:opt1} with a mean-squared error loss function $\ell(z,y) = |z-y|^2$. We tune the hyperparameters of depth ($L$) and $\ell_2$-regularization strength ($\lambda$) on a separate validation set. We compare against shallow ReLU networks without linear layers (i.e., $L=2$) trained in the same way and with the hyperparameter $\lambda$ tuned in the same way. See \Cref{app:experiment details} for more training details.

We test the performance of the trained networks on $m=2048$ new test samples of the form $(\vx_i, f(\vx_i) + \sigma \varepsilon_i)$ where either $\vx_i\sim \uniform([-\frac{1}{2},\frac{1}{2}]^{20})$ to measure in-distribution generalization (\Cref{fig:Generalization MSE}) or $\vx_i\sim \uniform([-1,1]^{20})$ to measure out-of-distribution generalization (\Cref{fig:OOD MSE}).
In \Cref{fig:Generalization MSE,fig:OOD MSE}, we see that the regularization induced by adding linear layers helps improve in- and out-of-distribution generalization in this setting; models with linear layers approach the irreducible error of $\sigma^2$ with fewer samples than models without linear layers.\footnote{Because of the label noise, the expected squared-error of any model will be at least $\sigma^2$.}

Models trained with extra linear layers are better able to adapt to the multi-index model underlying the data because they have a low effective index rank. 
We estimate the EGOP singular values of the trained networks $\hat f$ using the \emph{average} gradient outer product (AGOP) matrix computed over the in-distribution test set:
\begin{equation}
    \hat{\mC}_{\hat f} := \frac{1}{m} \sum_{i=1}^{m} \nabla \hat f(\vx_i) \nabla \hat f(\vx_i)^\T.
    \label{eq:AGOP}
\end{equation}
As shown in \cite{constantine2015active}, the AGOP is a good estimate of the EGOP with high probability. 
Thus, the singular values of $\hat f$ can be well approximated by the singular values of the square root of the AGOP.
The singular values for each model $\hat f$ are shown in \Cref{fig:trained singular values}. We observe that adding linear layers leads to trained networks with smaller singular values and lower effective index rank; the singular values $\sigma_k$  for larger $k$ of networks with extra linear layers are often many orders of magnitude smaller than their counterparts without linear layers. 

We also see that models with linear layers generalize better in our experiments because of improved alignment with the principal subspace of the ground truth function.
We use the AGOP to estimate the alignment between the principal subspace of the trained model and the true central subspace of $f$. We measure the alignment between two $r$-dimensional subspaces $\gU,\gV$ by their largest principal angle $\angle(\gU, \gV) = \arcsin(\|\mP_\gU-\mP_\gV\|_{op})$, where $\mP_\gU$ and $\mP_{\gV}$ denote the orthogonal projectors onto $\gU$ and $\gV$, respectively \cite{knyazev2002principal}.
In \Cref{fig:active subspace err} we show the largest principal angle between the principal subspace of $f$ and the principal subspace of the trained models, estimated as the span of the top $r$ eigenvectors of the AGOP. We also show the estimates of the effective index rank of the trained networks at the $\varepsilon = 10^{-3}$ tolerance level.
For models that have many singular values that are far from zero, including those trained without linear layers, the truncation to exactly $r$ eigenvectors in computing the principal angle can give an overly generous estimate of the agreement between the learned principal subspace and principal subspace of the function used to generate the data. Even using this generous estimate, models with linear layers demonstrate better alignment.

\begin{figure}[ht!]
    \centering
    \includegraphics[width=\columnwidth]{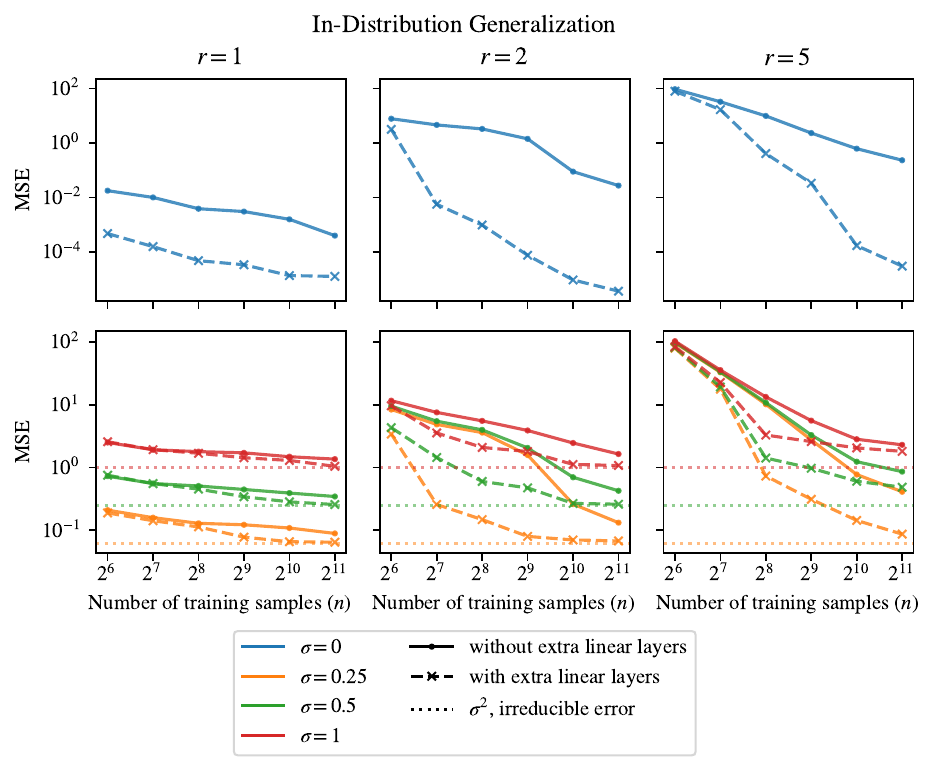}
    \caption{\textbf{Adding linear layers improves generalization on multi-index models.} In-distribution generalization performance of networks trained with or without extra linear layers on data from a single-index model (left) or multi-index model (center, right) with varying amounts of label noise. Models trained with extra linear layers demonstrate significantly improved generalization in this setting. (Bottom) Even in the presence of label noise ($\sigma > 0$), the generalization error of models with extra linear layers quickly approaches the irreducible error $\sigma^2$ as the number of training samples ($n$) increases. See \Cref{sec:experiments} and \Cref{app:experiment details} for training details.}
    \label{fig:Generalization MSE}
\end{figure}

\begin{figure}[ht!]
    \centering
    \includegraphics[width=\columnwidth]{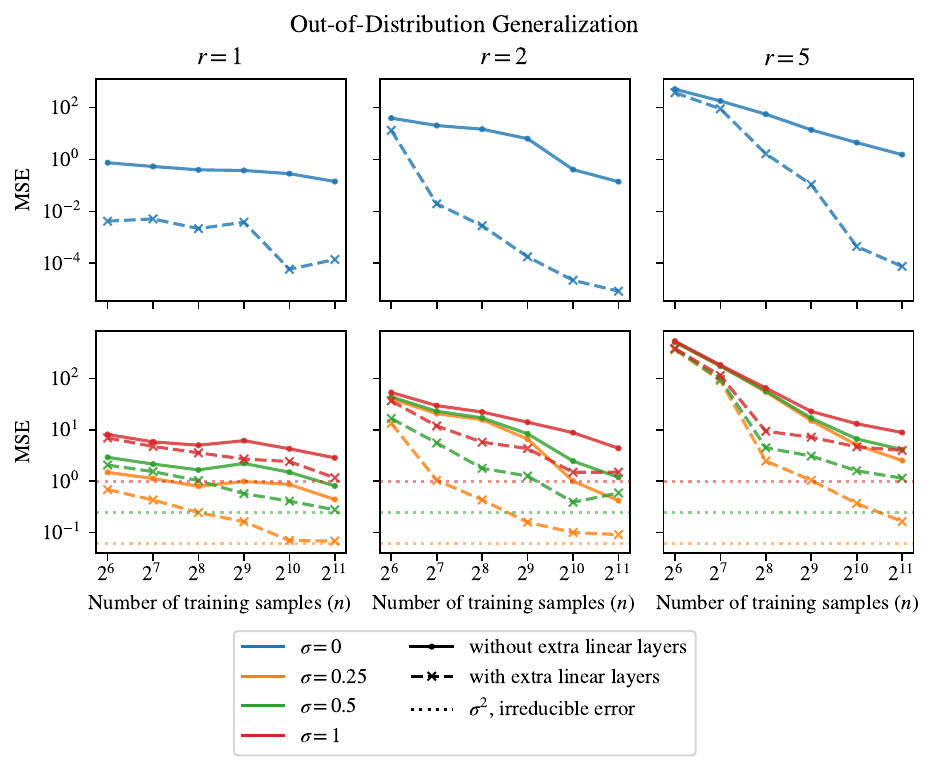}
    \caption{\textbf{Adding linear layers improves performance outside of the training distribution.} Out-of-distribution generalization performance of networks trained with or without extra linear layers on data from a single-index model (left) or multi-index model (center, right) with varying amounts of label noise. See \Cref{sec:experiments} and \Cref{app:experiment details} for training details.}
    \label{fig:OOD MSE}
\end{figure}

\begin{figure}[hp!]
    \centering
    \includegraphics[height=\dimexpr \textheight - 5\baselineskip\relax]{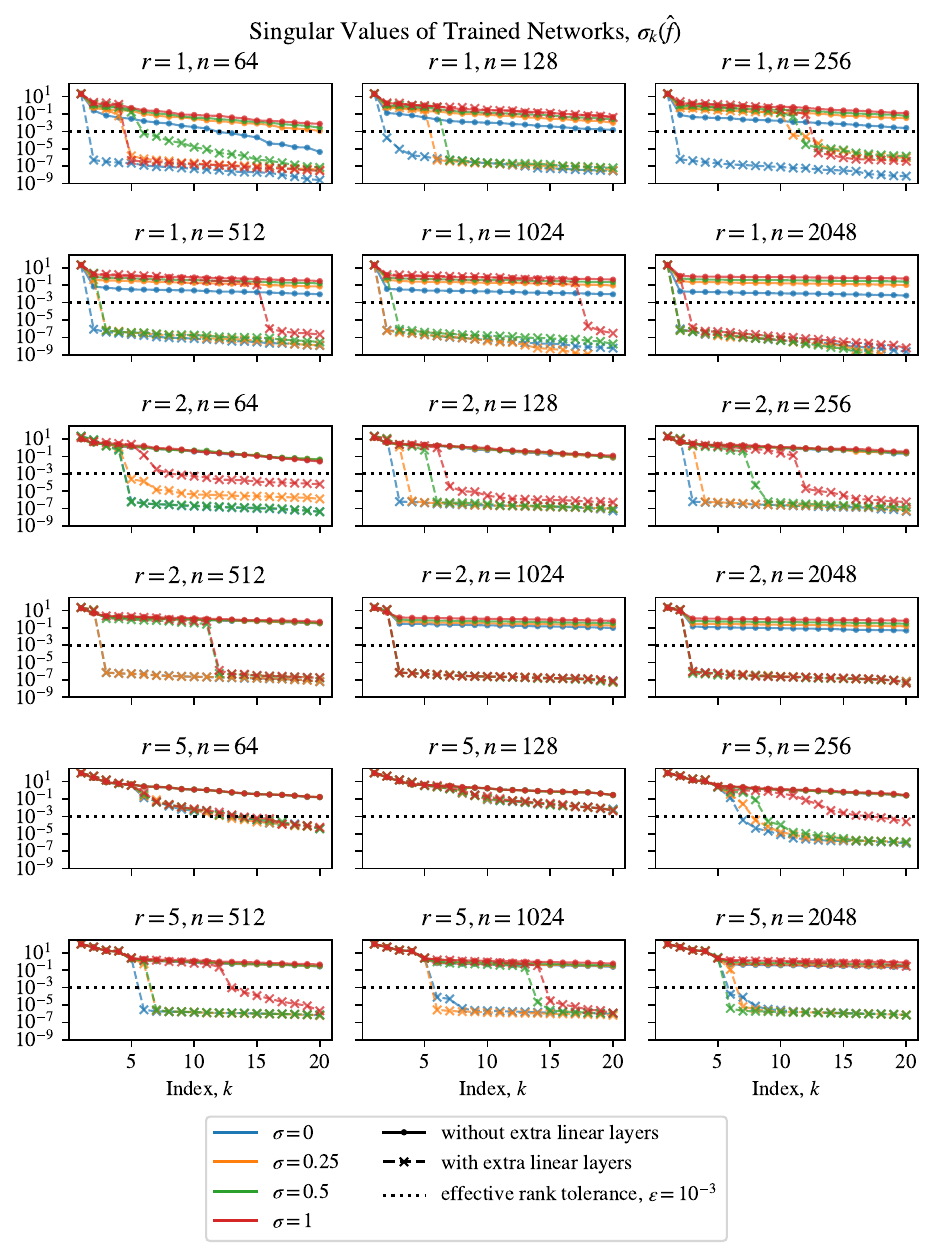}
    \caption{\textbf{Adding linear layers decreases the singular values of trained networks.} Singular values of trained networks trained with or without extra linear layers on data from a single-index model or multi-index model with varying amounts of label noise. Models with extra linear layers exhibit sharper singular value dropoff and have a smaller effective index rank at the $\varepsilon = 10^{-3}$ tolerance level than models without linear layers. See \Cref{sec:experiments} and \Cref{app:experiment details} for training details.}
    \label{fig:trained singular values}
\end{figure}

\begin{figure}[ht!]
    \centering
    \includegraphics[width=\columnwidth]{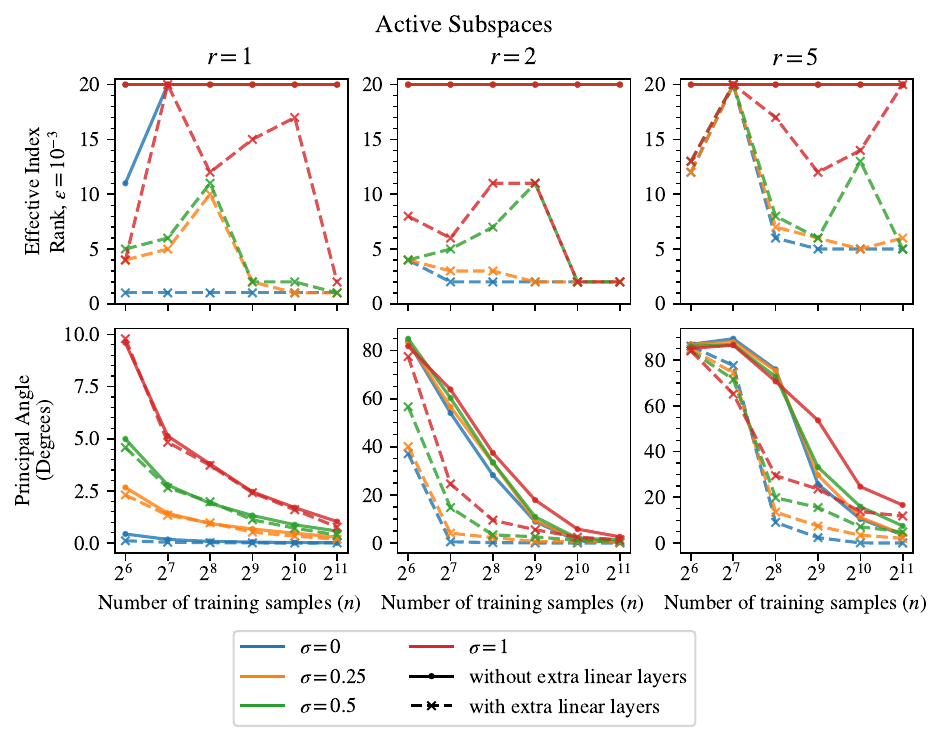}
    \caption{\textbf{Adding linear layers helps find networks with low effective index rank that are aligned with the true principal subspace.} Estimates of the effective index rank and principal subspace alignment of networks trained with or without extra linear layers on data from a single-index model (left) or multi-index model (center, right) with varying amounts of label noise. (Top) The effective index rank using a tolerance of $\varepsilon = 10^{-3}$. (Bottom) The largest principal angle between the principal subspace of $f$ and the span of the top $r$ eigenvectors of the AGOP of the trained model. See \Cref{sec:experiments} and \Cref{app:experiment details} for training details.}
    \label{fig:active subspace err}
\end{figure}

\section{Discussion, Limitations, and Future Directions}
\label{sec:discussion}
The representation cost expressions we derive offer new, quantitative insights into multi-layer networks trained using $\ell_2$-regularization. 
Specifically, we show that training a ReLU network with additional linear layers on the input side with $\ell_2$-regularization implicitly seeks a \textit{low-dimensional} subspace such that after projecting the training data into this subspace it can be fit with a \textit{smooth} function (in the sense of having a low two-layer representation cost). To characterize the representation cost in function space, we provide a formal definition of mixed variation
(\Cref{def:mv}) consistent with past usage \citep{donoho2000high,parhi2021banach}.

While we do not provide generalization bounds, our numerical experiments suggest that if low-index-rank structure is present in the data, adding linear layers induces a bias that is helpful for generalization, particularly with small sample sizes when two-layer networks have larger generalization errors. 
As Bach \citep{bach2017breaking} showed, shallow networks minimizing the $L=2$ representation cost can achieve the minimax generalization rate, which depends principally on the dimension of the latent central subspace
even in high-dimensional settings \citep{liu2020learning}.
An interesting direction for future work is to 
see if networks minimizing the $L>2$ representation cost can achieve improved generalization over shallow networks without constraining the network architecture (as in \citep{bietti2022learning}) or training (as in \citep{mousavi2022neural}) to explicitly seek the central subspace. While adding linear layers cannot improve rates as the sample size tends to infinity (since $L=2$ is already minimax optimal), it is possible that linear layers improve constants in generalization rates. Such improvements can have a substantial impact in practice, especially when the sample size is moderate, as in our experiments.
An additional benefit of this ability to adapt to latent single- and multi-index structure is that networks with low-index-rank are inherently compressible \citep{mousavi2022neural}.

It is important to note that the number of linear layers to add should be treated as a tunable hyperparameter;
\Cref{thm:effective rank bound} implies that adding too many linear layers with a fixed number of training samples can result in global minimizers that underestimate the index rank of the ground truth function. However, the number of linear layers at which such rank-deficient solutions occur may be large. In our experiments, we never observed rank-deficient solutions when training with a moderate number of linear layers ($L \le 9$) and using standard initialization schemes and optimization techniques.

One limitation of our theoretical analysis is the focus on properties of global or near-global 
minimizers. We do not analyze the dynamics of specific optimization algorithms, in contrast with \cite{Damian_Lee_Soltanolkotabi_2022}, which provides generalization bounds in terms of sample complexity of a shallow network trained with a modified form of gradient descent. An interesting extension of this work would be to analyze whether adding linear layers allows specific optimization algorithms to converge to functions with small $R_L$ cost, as we observe in experiments. In that case, \Cref{thm:effective rank bound} suggests that these solutions would have small effective index rank. 

Finally, a key limitation of the current work is that our analysis framework does not extend easily to deep networks with multiple nonlinear layers. 
The inductive bias studied in this work is not directly indicative of the inductive bias of deep ReLU networks. Specifically, the inductive bias of deep ReLU networks does not appear to inherently favor functions with low mixed-variation; see \Cref{sec:different architecture experiments} for a numerical study of training deep ReLU networks on data from a single-index model. These experiments show that adding ReLU layers does not produce functions with low mixed variation and does not enhance generalization in this setting.
Progress towards understanding the inductive bias of deep ReLU networks is found in \cite{jacot2022implicit,jacot2024bottleneck}, but more fully understanding the representation costs of general nonlinear deep networks remains a significant open problem for the community.

\section*{Acknowledgements}
R. Willett gratefully acknowledges the support of AFOSR grant FA9550-18-1-0166 and NSF grant DMS-2023109.
G. Ongie was supported by NSF CRII award CCF-2153371.
S. Parkinson was supported by the National Science Foundation Graduate Research Fellowship under Grant No. DGE:2140001. Any opinions, findings, and conclusions or recommendations expressed in this material are those of the authors and do not necessarily reflect the views of the National Science Foundation.

\newpage
\bibliographystyle{apalike}
\bibliography{refs}

\newpage
\appendix
\section{Rescaling invariant form of the representation cost}\label{sec:simplify}

Part of the difficulty in interpreting the expression for the $R_L$ cost in \eqref{eq:nucnormex} is that it varies under different sets of parameters realizing the same function. In particular, consider the following rescaling of parameters:  for any vector $\vlambda \in \R^K$ with positive entries, by the 1-homogeneity of the ReLU activation we have  
\begin{equation}\label{eq:rescaling trick}
\va^\top [\mW\vx + \vb]_+ + c= (\mD_{\vlambda}^{-1}\va)^\top[\mD_{\vlambda}\mW\vx + \mD_{\vlambda}\vb]_+ + c.
\end{equation}
However, the value of the objective in \eqref{eq:nucnormex} may vary between the two parameter sets realizing the same function. To account for this scaling invariance, we define a new loss function $\Phi_L$ by optimizing over all such ``diagonal'' rescalings of units.
Using the AM-GM inequality and a change of variables, one can prove that $\Phi_L$ depends only on $\mW$ and $\va$ only through the $K \times d$ matrix $\mD_{\va}\mW$. 
This leads us to the following result.
\begin{lemma}\label{lem:simplified expression for RL}
For any $f \in \setofnns{\densitysupp}$, we have
\begin{equation}
\label{eq:RL expressed w PhiL}
    R_L(f) = \inf_{\theta \in \Theta_2} \Phi_L(\mD_{\va}\mW) \st f = h_\theta^{(2)}|_\densitysupp.
\end{equation}
where the function $\Phi_L$ given a matrix $\mM$ is defined as
\begin{equation}
    \Phi_L(\mM) = \inf_{\substack{\|\vlambda\|_2= 1 \\ \lambda_k > 0,\forall k}} \|\mD_{\vlambda}^{-1}\mM\|_{\Sc^{2/(L-1)}}^{2/L}.
    \label{eq:phiL}
\end{equation}
\end{lemma}
\begin{proof}
Fix any parameterization $f = h_\theta^{(2)}|_\densitysupp$ with $\theta = (\mW,\va,\vb,c)$. 
Without loss of generality, assume $\va$ has all nonzero entries. By positive homogeneity of the ReLU, for any vector $\vlam \in \R^K$ with positive entries (which we denote by $\vlam > 0$) the rescaled parameters $\theta' = (\mD_\vlam^{-1}\mW,\mD_\vlam\va,\mD_\vlam^{-1}\vb,c)$ also satisfy $f = h_{\theta'}^{(2)}|_\densitysupp$. Therefore, by \Cref{lem:schatten} we have
\begin{align}
    R_L(f) & = \inf_{\theta \in \Theta_2} \tfrac{1}{L}\|\va\|_2^2 
+ \tfrac{L-1}{L}\|\mW\|^{2/(L-1)}_{\Sc^{2/(L-1)}} \st f = h_\theta^{(2)}|_\densitysupp\\
    & = \inf_{\theta \in \Theta_2} \inf_{\vlam > 0} \tfrac{1}{L}\|\mD_{\bm\lambda}\va\|_2^2 
    + \tfrac{L-1}{L}\|\mD_{\bm\lambda}^{-1}\mW\|^{2/(L-1)}_{\Sc^{2/(L-1)}} \st f = h_\theta^{(2)}|_\densitysupp.
\end{align}
Additionally, for any fixed $\bm\lambda > 0$, we may separately minimize over all scalar multiples $c\bm\lambda$ where $c>0$, to get 
\begin{align}
\inf_{\vlam > 0} \tfrac{1}{L}\|\mD_{\bm\lambda}\va\|_2^2 
+ & \tfrac{L-1}{L}\|\mD_{\bm\lambda}^{-1}\mW\|^{2/(L-1)}_{\Sc^{2/(L-1)}} \\
&  = \inf_{\bm\lambda > 0} 
\left( 
\inf_{c > 0}  
c^2 \tfrac{1}{L}\|\mD_{\bm\lambda}\va\|_2^2 
+ c^{-2/(L-1)}\tfrac{L-1}{L}\|\mD_{\bm\lambda}^{-1}\mW\|^{2/(L-1)}_{\Sc^{2/(L-1)}}
\right)\\
& = \inf_{\bm\lambda > 0} \left(\|\mD_{\bm\lambda}\va\|_2\|\mD_{\bm\lambda}^{-1}\mW\|_{\Sc^{2/(L-1)}}\right)^{2/L}
\end{align}
where the final equality follows by the weighted AM-GM inequality: for all $a,b>0$, it holds that $\tfrac{1}{L} a + \tfrac{L-1}{L} b \geq (ab^{L-1})^{1/L}$, which holds with equality when $a = b$. Here we have $a = (c \|\mD_{\bm\lambda}\va\|_2)^2$ and $b = \left(c^{-1}\|\mD_{\bm\lambda}^{-1}\mW\|_{\Sc^{2/(L-1)}}\right)^{2/(L-1)}$, and there exists a $c>0$ for which $a = b$, hence we obtain the lower bound.

Finally, performing the invertible change of variables $\vlam \mapsto \vlam'$ defined by $\lambda_{k}' = |a_k|\lambda_{k}$ for all $k=1,...,K$, we have 
\begin{align}
\inf_{\bm\lambda > 0} \left(\|\mD_{\bm\lambda}\va\|_2\|\mD_{\bm\lambda}^{-1}\mW\|_{\Sc^{2/(L-1)}}\right)^{2/L} & = \inf_{\bm\lambda' > 0} \left(\|\bm\lambda'\|_2\|\mD_{\bm\lambda'}^{-1}\mD_{\va}\mW\|_{\Sc^{2/(L-1)}}\right)^{2/L}\\
& = \inf_{\substack{\bm\lambda' > 0\\\|\bm\lambda'\|_2 = 1}} \|\mD_{\bm\lambda'}^{-1}\mD_{\va}\mW\|_{\Sc^{2/(L-1)}}^{2/L}
\end{align}
where we are able to constrain $\bm\lambda'$ to be unit norm since $\|\bm\lambda'\|_2\|\mD_{\bm\lambda'}^{-1}\mD_{\va}\mW\|_{\Sc^{2/(L-1)}}$ is invariant to scaling  $\bm\lambda'$ by positive constants.
\end{proof}

In the case of $L=2$, the infimum in \eqref{eq:phiL} can be computed explicitly as $\Phi_2(\mM) = \sum_{k=1}^K \|\vm_k\|_2$, where $\vm_k$ is the $k$th row of $\mM$. 
Notice that 
$
\Phi_2(\mD_{\va}\mW) = \sum_{k=1}^K |a_k|\|\vw_k\|_2,
$
which agrees with the expression in \eqref{eq:R2} after rescaling so that $\|\vw_k\|_2 = 1$ for all $k$.

When $L > 2$, we are unable to find a closed-form expression for $\Phi_L$. However, the characterization of $\Phi_L$ in \eqref{eq:phiL} still gives some insight into the kinds of functions that have small $R_L$ costs.  Intuitively, due to the presence of the Schatten-$q$ norm, functions with small $R_L$ cost have a low-rank inner-layer weight matrix $\mW$. Additionally, since the Schatten-$q$ norm for all $0 < q \leq 1$ dominates the Frobenius norm, functions with small $R_L$-cost will also have small $R_2$-cost.  These claims are formally strengthened in the following lemma.
\begin{lemma} 
\label{lem:relationship between PhiL and Phi2}
For all $L\geq 2$ and all matrices $\mM$, we have
    \begin{equation}
    \label{eq:bound on Phi_L via 1-2 norm}
        \Phi_2(\mM)^{2/L}
        \le \Phi_L(\mM)
        \le \rank(\mM)^{{(L-2)}/L}\Phi_2(\mM)^{2/L}.
    \end{equation}
    Additionally,
    \begin{equation}
        \label{eq:lb bound on Phi_L that reveals rank}
        \|\mM\|_{S^{2/L}}^{2/L} 
        \le \Phi_L(\mM)
    \end{equation}
\end{lemma}
Since both the upper bound from \eqref{eq:bound on Phi_L via 1-2 norm} and the lower bound from \eqref{eq:lb bound on Phi_L that reveals rank} tend toward the rank of $\mM$ as $L$ goes to infinity, so does $\Phi_L$.
The proof of \Cref{lem:relationship between PhiL and Phi2} is given in \Cref{app:repcost}.

\section{Proofs and Technical Details for Results in \Cref{sec:function_space}}
\label{app:proofs of function space results}
\subsection{Index Ranks of Neural Networks}
Observe that if $f(\vx) = a[\vw^\T\vx + b]_+$ then $\nabla f(\vx) = a \heaviside(\vw^\T\vx + b)\vw$ for almost all $\vx \in \densitysupp$ where $\heaviside$ is the unit step function. This implies that
$
\nabla f(\vx) \nabla f(\vx)^\T = a^2  \heaviside(\vw^\T\vx + b)\vw\vw^\T.
$
Likewise, if $f \in \setofnns{\densitysupp}$ and $f = h_\theta^{(2)}|_\densitysupp$ for some $\theta = (\mW,\va,\vb,c)$, 
then 
for almost all $\vx \in \densitysupp$,
\begin{align}
    \nabla f(\vx) \nabla f(\vx)^\T 
    = \sum_{k=1}^K \sum_{j=1}^K a_k a_j \heaviside(\vw_k^\T\vx + b_k) \heaviside(\vw_j^\T\vx + b_{j}) \vw_k \vw_j^\T
    = (\mD_{\va} \mW)^\top \mheaviside{\theta}(\vx) \mD_{\va} \mW
\end{align}
where $\mheaviside{\theta}(\vx)$ is the matrix of unit co-activations at $\vx$. That is, the entries of $\mheaviside{\theta}(\vx)$ are of the form $\heaviside(\vw_k^\T\vx + b_k) \heaviside(\vw_j^\T\vx + b_{j})$ and so will be one if and only if both unit $k$ and unit $j$ are active at $\vx$.
Taking expectations gives
\begin{equation}
\label{eq:grad cov of nn}
    \gradcov{f} = (\mD_{\va} \mW)^\top \mathbb{E}_X[\mheaviside{\theta}(X)] \mD_{\va} \mW.
\end{equation}
The expression above for $\gradcov{f}$ allows us to connect $\funcrank{f}$ to $\rank(\mD_{\va} \mW)$. 
We use the following technical lemma, proved in the \Cref{sec:proof of grad ae 0 means constant}.
\begin{lemma}
\label{lem:grad ae 0 means constant}
    Assume $\densitysupp$ is convex and has nonempty interior. 
    Let $f \in \setofnns{\densitysupp}$ and let  $\nabla f$ denote its weak gradient. Let $\vu \in \R^d$. If $\nabla f(\vx)^\T \vu = 0$ for almost all $\vx \in \densitysupp$, then $f(\vx + \vu) = f(\vx)$ for all $\vx \in \densitysupp$ such that $\vx + \vu \in \densitysupp$.
\end{lemma}
This lemma allows us to take the infimum in \eqref{eq:RL expressed w PhiL} over parameterizations of $f$ with the same rank as $f$, as stated in the next lemma.
\begin{lemma}
\label{lem:rank of f is rank of DaW}
    Assume $\densitysupp \subseteq \R^d$ is either a bounded convex set with nonempty interior or else $\densitysupp = \R^d$. Let $f \in \setofnns{\densitysupp}$. Then 
    \begin{equation}
        R_L(f) = \inf_{\theta \in \Theta_2} \Phi_L(\mD_{\va}\mW) \st f = h_\theta^{(2)}|_\densitysupp \mathand \funcrank{f} = \rank(\mD_{\va}\mW).
    \end{equation}
\end{lemma}
\begin{proof}
    By \eqref{eq:grad cov of nn}, any parameterization $\theta= (\mW,\va,\vb,c) \in \Theta_2$ of $f$ satisfies $\funcrank{f} \le \rank(\mD_{\va}\mW)$. 
    It suffices to show that for all $\theta = (\mW,\va,\vb,c) \in \Theta_2$ such that $f = h_\theta^{(2)}|_\densitysupp$, there is some $\theta' = (\mW',\va',\vb',c') \in \Theta_2$ such that $f = h_{\theta'}^{(2)}|_\densitysupp$, $\funcrank{f} \ge \rank(\mD_{\va'}\mW')$, and $\Phi_L(\mD_{\va'}\mW') \le \Phi_L(\mD_{\va}\mW)$.
    
    Fix a parameterization $\theta = (\mW,\va,\vb,c)$ of $f$ so that $f = h_\theta^{(2)}|_\densitysupp$.
    Let $\mP$ denote the orthogonal projector onto the range of $\gradcov{f}$. If $\densitysupp = \R^d$, then choosing $\theta' = (\mW\mP,\va,\vb,c)$ suffices; for any $\vx \in \R^d$, we have 
    \begin{equation}
        h_{\theta'}^{(2)}(\vx)
        = h_{\theta}^{(2)}(\mP\vx)
        = f(\mP\vx)
        = f(\vx)
    \end{equation}
    because \Cref{lem:grad ae 0 means constant} implies that $f$ is constant along the nullspace of $\gradcov{f}$.
    Additionally, notice that $\funcrank{f} = \rank(\mP) \ge \rank(\mD_{\va}\mW\mP)$.
    Finally, because multiplying by a projection matrix can only decrease singular values, we get
    $
        \Phi_L(\mD_{\va}\mW)
        \ge \Phi_L(\mD_{\va}\mW\mP).
    $
    If $\densitysupp$ is a bounded convex set, the choice of $\theta'$ becomes more delicate and is reserved for \Cref{sec:bounded cvx set technical details}.
\end{proof}
\subsection{Proof of \Cref{thm:relationship between RL R2 index rank and mixed var}}
\label{sec:proof of central lemma}
\begin{proof}
    Let $f\in \setofnns{\densitysupp}$ and $L \ge 2$. From the characterization of $R_L$ in terms of $\Phi_L$ from \Cref{lem:simplified expression for RL} and the bounds on $\Phi_L$ from \Cref{lem:relationship between PhiL and Phi2}, we get
    \begin{equation}
    \label{eq:initial R2 RL bound}
        R_2(f)^{2/L} \le R_L(f) \le
        \inf_{\theta : f = h_\theta^{(2)}|_\densitysupp} \rank(\mD_{\va}\mW)^{(L-2)/L} \Phi_2(\mD_{\va}\mW)^{2/L}.
    \end{equation}
    By \Cref{lem:rank of f is rank of DaW}, the characterization of $R_L$ in terms of $\Phi_L$ can be considered over just those parameterizations of $f$ where $\rank(\mD_{\va}\mW)$ matches $\funcrank{f}$. This allows us to upper bound the right-hand side in \eqref{eq:initial R2 RL bound} as follows:
    \begin{align}
        \funcrank{f}&^{(L-2)/L} R_2(f)^{2/L} \\
        &= \funcrank{f}^{(L-2)/L} \inf_{\substack{\theta : f = h_\theta^{(2)}|_\densitysupp\\\funcrank{f} = \rank(\mD_{\va}\mW)}} \Phi_2(\mD_{\va}\mW)^{2/L} \\
        &= \inf_{\substack{\theta : f = h_\theta^{(2)}|_\densitysupp\\\funcrank{f} = \rank(\mD_{\va}\mW)}} \rank(\mD_{\va}\mW)^{(L-2)/L} \Phi_2(\mD_{\va}\mW)^{2/L} \\
        &\ge \inf_{\theta : f = h_\theta^{(2)}|_\densitysupp} \rank(\mD_{\va}\mW)^{(L-2)/L} \Phi_2(\mD_{\va}\mW)^{2/L}.
    \end{align}
    Therefore 
    \begin{equation}
         R_L(f) \le \funcrank{f}^{(L-2)/L} R_2(f)^{2/L}.
    \end{equation}

Now we prove 
\begin{equation}
    \mixedvar{f}{\tfrac{2}{L}}^{2/L} \leq R_L(f).
\end{equation}
Assume $f = h_\theta^{(2)}|_\densitysupp$ for some $\theta = (\mW,\va,\vb,c)$. 
Let $\mathbb{E}_X[\mheaviside{\theta}(X)]^{1/2}$ be a matrix square root of $\mathbb{E}_X[\mheaviside{\theta}(X)]$. 
By \eqref{eq:grad cov of nn}, a nonsymmetric square root of $\gradcov{f}$ is given by $\gradcov{f}^{1/2} = \mathbb{E}_X[\mheaviside{\theta}(X)]^{1/2}\mD_{\va} \mW$, and so we have $\mixedvar{f}{q} = \|\mathbb{E}_X[\mheaviside{\theta}(X)]^{1/2}\mD_{\va} \mW\|_{S^q}$.
Fix any vector $\vlambda > 0$ such that $\|\vlambda\|_2 = 1$. Then we have 
\begin{align}
\mixedvar{f}{\frac{2}{L}} 
& = \|\mathbb{E}_X[\mheaviside{\theta}(X)]^{1/2}\mD_{\va} \mW\|_{S^\frac{2}{L}}\\
& = \|\mathbb{E}_X[\mheaviside{\theta}(X)]^{1/2}\mD_{\vlambda}\mD_{\vlambda}^{-1}\mD_{\va} \mW\|_{S^\frac{2}{L}}\\
& \leq \|\mathbb{E}_X[\mheaviside{\theta}(X)]^{1/2}\mD_{\vlambda}\|_{F}\|\mD_{\vlambda}^{-1}\mD_{\va} \mW\|_{S^\frac{2}{L-1}}, \label{eq:mv bound by prod of schatten norms}
\end{align}
where in the final step we used 
the fact that
for any matrices $\mA$ and $\mB$ with compatible dimensions, any $0 < p \le 1$, and any $p_1, p_2 > 0$ such that $1/p = 1/p_1 + 1/p_2$, we have $\|\mA \mB\|_{S^p} \le \|\mA\|_{S^{p_1}} \|\mB\|_{S^{p_2}}$; this is a direct consequence of \cite[Theorem 1]{shang2020unified}, here applied with $p = \frac{2}{L}, p_1 = 2, p_2 = \frac{2}{L-1}$.
Next, observe that
\begin{align}
\|\mathbb{E}_X[\mheaviside{\theta}(X)]^{1/2}\mD_{\vlambda}\|_{F}^2
&= \Tr(\mD_{\vlambda}\mathbb{E}_X[\mheaviside{\theta}(X)]\mD_{\vlambda}) \\
&= \sum_{k=1}^K \lambda_k^2 (\mathbb{E}_X[\mheaviside{\theta}(X)])_{kk} \\
&\le \sum_{k=1}^K \lambda_k^2
= 1, \label{eq:bound for mv schatten norm}
\end{align}
because the entries in $\mheaviside{\theta}(X)$ are at most one and $\vlambda$ has unit norm.
Combining \cref{eq:bound for mv schatten norm} with \cref{eq:mv bound by prod of schatten norms}, we see that
\begin{equation}
\mixedvar{f}{\frac{2}{L}} \leq \|\mD_{\vlambda}^{-1} \mD_{\va} \mW\|_{S^{\frac{2}{L-1}}}.
\end{equation}
Since this inequality is independent of the choice of $\vlambda$, we have that 
\begin{equation}
\mixedvar{f}{\frac{2}{L}} \leq \inf_{\substack{\|\vlambda\|_2 = 1 \\ \vlambda > 0}}\|\mD_{\vlambda}^{-1} \mD_{\va} \mW\|_{S^{\frac{2}{L-1}}} = \Phi_L(\mD_{\va}\mW)^{L/2}.
\end{equation}
Finally, since the above inequality holds independent of the choice of parameters $\theta$ realizing $f$, we have
\begin{equation}
\mixedvar{f}{\frac{2}{L}} \leq \inf_{\theta : f = h_\theta^{(2)}|_\densitysupp}\Phi_L(\mD_{\va} \mW)^{L/2} =  R_L(f)^{L/2},
\end{equation}
and taking $(2/L)$-powers of both sides of this inequality gives the claim.
\end{proof}

\subsection{Proof of Corollaries to \Cref{thm:relationship between RL R2 index rank and mixed var}}
\begin{proof}[Proof of \Cref{cor:limiting cost theorem}]
    For ease of notation, denote $\funcrank{f}$ by $r$. By definition $\mixedvarsv{f}{r} > 0$.
    For any $L \ge 2$,
    \begin{equation*}
        \mixedvar{f}{\frac{2}{L}}^{2/L}
        = \sum_{k=1}^d \mixedvarsv{f}{k}^{\frac{2}{L}}
        \geq r \mixedvarsv{f}{r}^{\frac{2}{L}}.
    \end{equation*}
    By \Cref{thm:relationship between RL R2 index rank and mixed var}, it follows that
    \begin{equation}
    \label{eq:ub and lb on RL cost in terms of rank sv and R2 cost}
        r \mixedvarsv{f}{r}^{\frac{2}{L}}
        \leq R_L(f) 
        \leq r^{\frac{L-2}{L}} R_2(f)^{2/L}.
    \end{equation}
    The upper and lower bounds from \Cref{eq:ub and lb on RL cost in terms of rank sv and R2 cost} both tend to $r$ as $L \rightarrow \infty$, so $R_L(f) \rightarrow \funcrank{f}$.
\end{proof}

\begin{proof}[Proof of \Cref{cor:low high theorem}]
    Let $r_l$ and $r_h$ denote the index ranks of $f_l$ and $f_h$, respectively. Choose
    \begin{equation}
    L_0 :=  1 + 2 \frac{\log R_2(f_l) - \log r_l - \log \sigma_{r_h}(f_h)}{\log r_h - \log r_l}.
    \end{equation}
    Then $L > L_0$ implies
    \begin{equation}
    r_l^{\frac{L-2}{2}} R_2(f_l) < r_h^{\frac{L}{2}}\sigma_{r_h}(f_h) \le 
    \mixedvar{f_h}{\tfrac{2}{L}}.
    \end{equation}
    By \Cref{thm:relationship between RL R2 index rank and mixed var}, it follows that $R_L(f_l) < R_L(f_h)$.
\end{proof}

\subsection{Existence of Index Rank-$1$ Interpolants}
\label{sec:rank-r interpolant always exists}
\begin{lemma}
    Given training pairs $\{(\vx_i,y_i)\}_{i=1}^n$ with $\vx_i \in \densitysupp$ and $y_i \in \R$, assume that $\vx_i \ne \vx_j$ whenever $i \ne j$. Then there exists a function $f \in \setofnns{\densitysupp}$ such that $f(\vx_i) = y_i$ for all $i \in [n]$ and $\funcrank{f} = 1$.
\end{lemma}
\begin{proof}
    Let $\gW$ denote the set of all $\vw\in \R^d$ such that $\vw^\T \vx_i = \vw^\T \vx_j$ for some $i \ne j$.
    Let $\vz_1,\ldots,\vz_N$ be an enumeration of all difference vectors $\vx_i-\vx_j$, $i \ne j$. We can write $\gW$ as the set of all $\vw \in \R^d$ such that $\vw^\T \vz_k = 0$ for some $k \in 1,\ldots,N$. Thus, $\gW$ is the union of $N$ different hyperplanes $\gW_k$, where $\gW_k = \{\vw : \vw^\T \vz_k = 0\}$ is the hyperplane normal to $\vz_k$. Each $\gW_k$ is a $d-1$ dimensional hyperplane in $\R^d$ and therefore has Lebesgue measure zero. Hence, their finite union (i.e., $\gW$) must have measure zero as well. We conclude that there is some $\vw_* \in \R^d \setminus \gW$ such that  $\vw_*^\T \vx_i \ne \vw_*^\T \vx_j$ whenever $i \ne j$.

    Consider a univariate function $g:\R \rightarrow \R$ in $\setofnns{\R}$ that interpolates the projected data pairs $\{(\vw_*^\T \vx_i,y_i)\}_{i=1}^n$. For example, we can choose $g(t)$ to be the piecewise linear spline interpolant with knots at $t_i = \vw_*^\T\vx_i$ that is constant for $t < \min_i t_i$ and $t > \max_i t_i$. This function $g$ can be written as a sum of finitely many ReLU units and so belongs to $\setofnns{\R}$. 
    Define $f \in \setofnns{\densitysupp}$ by $f(\vx) := g(\vw_*^\T \vx)$. Then $f$ interpolates the original training pairs $\{(\vx_i,y_i)\}_{i=1}^n$. Further, the weak gradient of $f$ is $\nabla f(\vx) = g'(\vw_*^\T \vx)\vw_*$ where $g'$ is the weak derivative of $g$. This means that the expected gradient outer product of $f$ is equal to the rank one matrix $\mathbb{E}_X[g'(\vw_*^\T X)^2] \vw_* \vw_*^\T$. Therefore $\funcrank{f} = 1$.
\end{proof}

\subsection{Proof of \Cref{thm:effective rank bound}}
We begin with a lemma about the singular value decay of $\hat f_L$ which is straightforward to prove using algebraic manipulations of \Cref{thm:relationship between RL R2 index rank and mixed var}; see \Cref{sec:proof of singular value decay}.
\begin{lemma}
\label{lem:singular value decay}
    Assume that $\hat f_L$ is an $R_L$-minimal interpolant. Then
    for all $t \in [d]$, 
    \begin{equation}
        \mixedvarsv{\hat{f}_L}{t} \le 
        \min_{s \in [d]} \frac{\interpcost_s(\mathcal{D})}{s} \left(\frac{s}{t}\right)^{\frac{L}{2}}.
    \end{equation}
\end{lemma}
Using this lemma, we now prove \Cref{thm:effective rank bound}.
\begin{proof}
    Assume to the contrary that
    \begin{equation}
        \epsfuncrank{\hat f_L} > 
        \min_{s \in [d]}
        s 
        \left(\frac{\interpcost_s(\mathcal D)}{\varepsilon s}\right)^{\frac{2}{L}}.
    \end{equation}
    Then there is some integer $t$ with
    \begin{equation}
        t > \min_{s \in [d]}
        s 
        \left(\frac{\interpcost_s(\mathcal D)}{\varepsilon s}\right)^{\frac{2}{L}}
        \label{eq:lower bound on t}
    \end{equation}
    such that $\mixedvarsv{\hat{f}_L}{t} > \varepsilon$. 
    Rearranging \Cref{eq:lower bound on t} and applying \Cref{lem:singular value decay}, we conclude that
    \begin{equation}
        \varepsilon 
        >
        \min_{s \in [d]} \frac{\interpcost_s(\mathcal{D})}{s} \left(\frac{s}{t}\right)^{\frac{L}{2}}
        \ge \mixedvarsv{\hat{f}_L}{t}.
    \end{equation}
    This is a contradiction, so 
    \begin{equation}
        \epsfuncrank{\hat f_L} \le 
        \min_{s \in [d]}
        s 
        \left(\frac{\interpcost_s(\mathcal D)}{\varepsilon s}\right)^{\frac{2}{L}}.
    \end{equation}
    Finally, the floor function can be put inside the minimum because $\epsfuncrank{\hat f_L}$ is an integer. 
\end{proof}

Finally, to prove the lower bound on the effective rank given in \Cref{thm:effective rank bound}, we provide the following lemma, which shows that under mild conditions the sum of squared singular values of a minimum $R_L$-cost interpolant for any $L\geq 2$ is uniformly bounded below by a constant depending only on the data. In particular, this result implies that the top singular value of a sequence of minimum $R_L$-cost interpolants cannot vanish as $L\rightarrow \infty$, and so the $\varepsilon$-effective index rank is always at least one for sufficiently small $\varepsilon$. The proof can be found in \Cref{sec:lower bound on singular values proof}.

\begin{lemma}\label{lem:singval_lower_bnd}
Assume that $\hat f_L$ is an $R_L$-minimal interpolant.
Suppose that $\Omega \subseteq \mathcal X$ is any open bounded set with $C^1$ boundary such that $\rho$ is uniformly bounded away from zero on $\Omega$.
Then 
\begin{equation}
\sum_{k=1}^d \sigma_k(\hat f_L)^2 
\geq C 
\frac{(\min_{c\in\mathbb{R}} \max_{i: \vx_i \in \Omega} |y_i-c|)^{d+2}}
{\interpcost_1(\gD)^d}
\end{equation}
where $C>0$ is a constant depending only on $\Omega$, $\rho$ and $d$. In particular, if $\Omega$ contains two points $\vx_i,\vx_j$ whose corresponding labels $y_i,y_j$ are not equal, the lower bound is non-zero.
\end{lemma}

\section{Generalizing to Vector-Valued Functions}
\label{sec:vector valued functions}

While we focus on functions $f: \densitysupp \rightarrow \R$, our results can be naturally generalized to vector-valued functions $f: \densitysupp \rightarrow \R^D$ with $D > 1$. 
In this setting, the $L$-layer representation cost $R_L(f)$ is the minimal cost 
$
    C_L(\theta) = \frac{1}{L}\left(\|\mA\|_F^2 + \sum_{\ell=1}^{L-1}\|\mW_{\ell}\|_F^2\right)
$
required to parameterize $f$ over $\densitysupp$ as 
$
    f(\vx) = \mA^\T[\mW_{L-1}\cdots\mW_2\mW_1 \vx + \vb]_+ + \vc
$
where now $\mA$ is a $D \times K$ matrix and $\vc$ is a vector in $\R^D$. 

Given $f: \densitysupp \rightarrow \R^D$, consider a generalization of the EGOP matrix where $d \times 1$ gradient vectors are replaced by $D \times d$ Jacobian matrices: 
\begin{equation}
\gradcov{f} := \mathbb{E}_X[Jf(X)^\T Jf(X)] = \int_\densitysupp Jf(\vx)^\T Jf(\vx) \rho(\vx)\, d\vx
\label{eq:EJGM}.
\end{equation}
We refer to this as the expected Jacobian Gram matrix (EJGM).
We use the EJGM instead of the EGOP to define the index rank, principal subspace, singular values, and mixed variation of vector-valued functions $f$. 

For example, consider $f = [f_1, f_2]^\T$ where both component functions $f_1, f_2 : \densitysupp \rightarrow \R$ have index-rank 1
with distinct principal subspaces $\vspan(\vv_1)$ and $\vspan(\vv_2)$, respectively. It is straightforward to verify that $\gradcov{f} = \gradcov{f_1} + \gradcov{f_2}$. Using this fact, we can see that the principal subspace of $f$ (i.e., range of $\gradcov{f}$) is $\vspan(\vv_1,\vv_2)$ and the index-rank is 2; note that the active subspace of $f$ is the \emph{sum} of the principal subspaces of $f_1$ and $f_2$ instead of their \emph{union}. 

Using these modified definitions, all the results in \Cref{sec:function_space} hold with only minor changes in their proofs. 
We conclude that minimizing the $R_L$ cost in this setting promotes learning functions $f$ where each component $f_j$, for $j = 1, \ldots, D$, is nearly constant orthogonal to a universal low-dimensional subspace (universal in the sense that the subspace does not depend on $j$) and is smooth along that subspace.
\section{Extensions of \Cref{thm:effective rank bound}}
\label{sec:extension of rank bounds}
In this section, we extend \Cref{thm:effective rank bound} to interpolants that nearly minimize the $R_L$ cost and to functions that nearly minimize the $R_L$-regularized empirical risk. The proofs of these extensions are only slight modifications of the proof of \Cref{thm:effective rank bound}. 
\begin{corollary}[Effective index ranks of near-minimal interpolants.]\label{cor:effective rank bound of near minimal interpolants}
    Assume that $\hat f \in \setofnns{\densitysupp}$ is nearly an $R_L$-minimal interpolant. That is, $\hat f(\vx_i) = y_i$ for all $i \in [n]$, and 
    for some small constant $\alpha \ge 0$, 
    \begin{equation}
        \label{eq:near minimal RL interpolant}
        R_L(\hat f) \le (1+\alpha) \left(\inf_{f \in \setofnns{\densitysupp}}  R_L(f) \st  f(\vx_i) = y_i \;\forall i \in [n]\right).
    \end{equation}
    Then given $\varepsilon > 0$, we have the following bound on the $\varepsilon$-effective index rank of $\hat f$:
    \begin{equation}
        \epsfuncrank{\hat f} \le 
        \min_{1 \le s \le d}
        \left\lfloor 
        (1+\alpha)
        s 
        \left(\frac{\interpcost_s(\mathcal D)}{\varepsilon s}\right)^{\frac{2}{L}}
        \right\rfloor.
    \end{equation}
\end{corollary}
The parameter $\alpha$ in \Cref{cor:effective rank bound for near RERM} controls how close $\hat f$ is to being $R_L$-minimal; if $\alpha = 0$, then $\hat f$ is exactly an $R_L$-minimal interpolant. In the next result, $\alpha$ plays a similar role; it controls how close $\hat f$ is to minimizing the regularized empirical risk. 
\begin{corollary}[Effective index ranks of near-minimizers of the regularized risk.]\label{cor:effective rank bound for near RERM}
    Assume that $\hat f \in \setofnns{\densitysupp}$ (nearly) minimizes the $R_L$-regularized empirical $\ell^2$ risk. That is, for some regularization parameter $\lambda > 0$ and some small constant $\alpha \ge 0$
    \begin{equation}
        \label{eq:near regularized empirical risk minimizer}
        \frac{1}{n}\sum_{i=1}^n |y_i - \hat f(\vx_i)|^2 + \lambda R_L(\hat f) 
        \le (1+\alpha) \left(\inf_{f \in \setofnns{\densitysupp}}\frac{1}{n}\sum_{i=1}^n |y_i - f(\vx_i)|^2 + \lambda R_L(f)\right).
    \end{equation}
    Then given $\varepsilon > 0$, we have the following bound on the $\varepsilon$-effective index rank of $\hat f$:
    \begin{equation}
        \epsfuncrank{\hat f} \le 
        \min_{1 \le s \le d}
        \left\lfloor 
        (1+\alpha)
        s 
        \left(\frac{\interpcost_s(\mathcal D)}{\varepsilon s}\right)^{\frac{2}{L}}
        \right\rfloor.
    \end{equation}
\end{corollary}
The proofs of \Cref{cor:effective rank bound of near minimal interpolants,cor:effective rank bound for near RERM} are essentially identical to the proof of \Cref{thm:effective rank bound}, but use a slightly modified version of \Cref{lem:singular value decay}, as follows.
\begin{lemma}
\label{lem:modified singular value decay}
    Assume that $\hat f$ satisfies \Cref{eq:near minimal RL interpolant} or \Cref{eq:near regularized empirical risk minimizer}. Then
    for all $t \in [d]$, 
    \begin{equation}
        \mixedvarsv{\hat{f}}{t} \le 
        (1+\alpha)^{\frac{L}{2}}
        \min_{s \in [d]} \frac{\interpcost_s(\mathcal{D})}{s} \left(\frac{s}{t}\right)^{\frac{L}{2}}.
    \end{equation}
\end{lemma}
The proof of this lemma is shown in \Cref{sec:proof of singular value decay}.

\section{Proof of \Cref{lem:relationship between PhiL and Phi2}}\label{app:repcost}
\subsection{Proof of \Cref{lem:relationship between PhiL and Phi2}: \Cref{eq:bound on Phi_L via 1-2 norm}}
\begin{proof}
Let $q \in (0,1]$ and $\vsig \in \R^n$. Since the function $t \mapsto t^{2/q}$ is convex, we can use Jensen's inequality to see that
\begin{equation}
n^{-\frac{2}{q}} \|\vsig\|_q^2 
= \left(\frac{1}{n}\sum_{i=1}^n \sigma_i^q\right)^{\frac{2}{q}}
\le \frac{1}{n}\left(\sum_{i=1}^n \sigma_i^2\right)
= n^{-1} \|\vsig\|_2^2.
\end{equation}
Thus
\begin{equation}
    \|\vsig\|_2 \le \|\vsig\|_q \le n^{\frac{1}{q} - \frac{1}{2}}\|\vsig\|_2.
\end{equation}
When $q = \frac{2}{L-1}$, we have $\frac{1}{q} - \frac{1}{2} = \frac{L-2}{2}$. Extending this result to Schatten norms and raising all expressions to the $2/L$ power, we see that for any rank-$r$ matrix $\mM$,
\begin{equation}
    \|\mM\|_{F}^{2/L} \le \|\mM\|_{\gS^q}^{2/L} \le r^{\frac{L-2}{L}}\|\mM\|_{F}^{2/L}.
\end{equation}
Therefore,
\begin{equation}
\label{eq:PhiL in terms of Phi2 and rank}
\Phi_2(\mM)^{2/L}
\le \Phi_L(\mM)
\le (\rank{\mM})^\frac{L-2}{L} \Phi_2(\mM)^{2/L}.
\end{equation}
\end{proof}

\subsection{Proof of \Cref{lem:relationship between PhiL and Phi2}: \Cref{eq:lb bound on Phi_L that reveals rank}}
\begin{proof}
In \cite{wang1997} it is shown that given matrices $\mA \in \R^{d \times K}, \mB \in \R^{K \times K}$ and a constant $q > 0$, 
\begin{equation}
\|\mA\mB\|_{\gS^q}^q= 
\sum_{k=1}^K \sigma_{k}^q (\mA\mB)
\ge \sum_{k=1}^K \sigma_{k}^q (\mB) \ \sigma_{K-k+1}^q (\mA).
\end{equation}
We apply this result to $\mD_{\vlam}^{-1}\mM$ where $\vlam > 0$:
\begin{align}
 \|\mD_{\vlam}^{-1}\mM\|_{S^{\frac{2}{L-1}}}^{\frac{2}{L-1}}
 &\ge \sum_{k=1}^K \sigma_{k}^{\frac{2}{L-1}} (\mM) \ \sigma_{K-k+1}^{\frac{2}{L-1}} (\mD_{\vlam}^{-1}) \\
 &= \sum_{k=1}^K \sigma_{k}^{\frac{2}{L-1}} (\mM) \ \sigma_{k}^{-\frac{2}{L-1}} (\mD_{\vlam}).
\end{align}

Next, we take the infimum over both sides and replace $\vlambda$ with its ordered version, $\vmu$:
\begin{align}
\Phi_L(\mM)^{\frac{L}{L-1}} 
&\ge \inf_{\substack{\|\vlam\|_2 = 1,\\ \lambda_k > 0, \forall k}} \sum_{k=1}^K \sigma_{k}^{\frac{2}{L-1}} (\mM) \ \sigma_{k}^{-\frac{2}{L-1}} (\mD_{\vlam}).\\
&\ge \min_{\substack{\|\vmu\|_2 = 1,\\ \mu_1 \ge \mu_2 \ge \ldots \ge \mu_K \ge 0}} \sum_{k=1}^K \sigma_{k}^{\frac{2}{L-1}} (\mM) \ \mu_{k}^{-\frac{2}{L-1}}
\end{align}
Using Lagrange multipliers, we find that
\begin{equation}
    \min_{\substack{\|\vmu\|_2 = 1,\\ \mu_1 \ge \mu_2 \ge \ldots \ge \mu_K \ge 0}} \sum_{k=1}^K \sigma_{k}^{2/(L-1)} (\mM) \ \mu_{k}^{-2/(L-1)}
    = \|\mM\|_{S^{2/L}}^{2/(L-1)}
\end{equation}
Therefore, 
\begin{equation}
\label{eq:lower bound on Phi_L}
\Phi_L(\mM)
\ge \|\mM\|_{S^{2/L}}^{2/L}.
\end{equation}
\end{proof}

\section{Additional Proofs and Lemmas for Results in \Cref{sec:function_space}}
\subsection{Proof of \Cref{lem:grad ae 0 means constant}}
\label{sec:proof of grad ae 0 means constant}
\begin{proof}
If $f \in \setofnns{\densitysupp}$, then $f$ is a continuous piecewise linear function with finitely many linear regions. Let $\Omega_1,...,\Omega_N \subseteq \densitysupp$ denote a disjoint partition of $\densitysupp$ so that $f$ is piecewise linear over each $\Omega_j$ and each $\Omega_j$ has positive measure. Let $\chi_{\Omega_j}$ denote the indicator function for $\Omega_j$. There exist some $\vv_j\in\R^{d}$ and $c_j\in \R$ for $j = 1,\ldots N$ such that $f(\vx) = \sum_{j=1}^N (\vv^\T_j\vx + c_j) \chi_{\Omega_j}(\vx)$ for all $\vx\in\densitysupp$. Observe that $\nabla f(\vx) = \sum_j \vv_j \chi_{\Omega_j}(\vx)$ is the weak gradient of $f(\vx)$. Since each $\Omega_j$ has positive measure, we see that $\nabla f(\vx)^\T \vu = 0$ for almost all $\vx \in \Omega_j$ implies $\vv^\T_j \vu = 0$ for all $j$.

Now assume $\vx,\vx + \vu \in \densitysupp$. Since $\densitysupp$ is convex, for all $t \in [0,1]$ we have $\vx + t\vu \in \densitysupp$. Consider the cardinality of the range of the continuous function $t \mapsto f(\vx + t\vu)$. First, 
\begin{align*}
    \left|\left\{f(\vx + t\vu) : t \in [0,1]\right\}\right|
    &= \left|\left\{\sum_{j=1}^N (\vv^\T_j(\vx + t\vu) + c_j) \chi_{\Omega_j}(\vx + t\vu) : t \in [0,1]\right\}\right| \\
    &= \left|\left\{\sum_{j=1}^N (\vv^\T_j\vx + c_j) \chi_{\Omega_j}(\vx + t\vu) : t \in [0,1]\right\}\right|
\end{align*}
because $f$ is the continuous version of the expression in the right-hand side; on the boundaries between regions, the expression in the right-hand side is equal to zero. 
Next, observe that 
\begin{equation}
\left|\left\{\sum_{j=1}^N (\vv^\T_j\vx + c_j) \chi_{\Omega_j}(\vx + t\vu) : t \in [0,1]\right\}\right| \le 2^N
\end{equation}
because any term in the sum can take on one of two values. A continuous function with finite range and connected domain must be constant, so $f(\vx) = f(\vx + t\vu)$ for all $t \in [0,1]$. In particular, $f(\vx) = f(\vx + \vu)$.
\end{proof}
\subsection{Proof of \Cref{lem:rank of f is rank of DaW} when $\densitysupp$ is a bounded convex set}
\label{sec:bounded cvx set technical details}
    As before, we need to show that
    for all $\theta = (\mW,\va,\vb,c) \in \Theta_2$ such that $f = h_\theta^{(2)}|_\densitysupp$, there is some $\theta' = (\mW',\va',\vb',c') \in \Theta_2$ such that $f = h_{\theta'}^{(2)}|_\densitysupp$, $\funcrank{f} \ge \rank(\mD_{\va'}\mW')$, and $\Phi_L(\mD_{\va'}\mW') \le \Phi_L(\mD_{\va}\mW)$.
    When $\densitysupp = \R^d$, the new parameterization $\theta'$ is obtained by projecting the weight matrix $\mW$ onto the range of $\gradcov{f}$. This is not quite enough when $\densitysupp$ is a bounded convex set, primarily because of units whose active set boundaries are outside $\densitysupp$. 
    Instead, the strategy in creating $\theta'$ when $\densitysupp$ is a bounded convex set is to combine the problematic units into one affine piece and then apply the following technical lemma: 
    \begin{lemma}
    \label{lem:weights of canonical form}
        Assume $\densitysupp$ is convex and has nonempty interior. Suppose 
        \begin{equation}
            f(\vx) = \sum_{k = 1}^K a_k [\vw_k^\T\vx + b_k]_+ + \vv^\T\vx + c,~~\forall\vx \in \densitysupp.
        \end{equation}
        Assume that for every unit $k \in [K]$, $a_k \ne 0$ and the active set boundaries $H_k = \{\vx : \vw_k^\T\vx + b_k = 0\}$ are distinct and intersect the interior of $\densitysupp$.
        Then $\vv \in \range(\gradcov{f})$ and $\vw_k \in \range(\gradcov{f})$ for all $k \in [K]$.
    \end{lemma}
        \begin{proof}
        It suffices to show that $\vw_1,\ldots,\vw_K$ and $\vv$ lie in $ \nullspace(\gradcov{f})^\perp$, so we fix a vector $\vu \in \nullspace(\gradcov{f})$ and show that $\vu$ is orthogonal to $\vw_1,\ldots,\vw_K$ and $\vv$. 

        Fix a unit $k \in [K]$. First, we pick a point on the active set boundary $H_k$. Let $\densitysupp^\mathrm{o}$ denote the interior of $\densitysupp$.
        Since the active set boundaries all intersect $\densitysupp^\mathrm{o}$ and are distinct,
        there is an $\vx_k \in H_k \cap \densitysupp^\mathrm{o}$ such that $\vx \not \in H_j$ whenever $j \ne k$.
        
        Next, we consider small perturbations of $\vx_k$ in the direction of $\pm \vw_k$.
        Pick $\varepsilon > 0$ sufficiently small so that $\vx_k \pm \varepsilon \vw_k \in \densitysupp^\mathrm{o}$ and $\varepsilon |\vw_j^\T \vw_k| < |\vw_j^\top \vx_k + b_j|$ whenever $j \ne k$.
        This implies that 
        \begin{enumerate}\setlength{\itemindent}{0.4in}
            \item $\vw_k^\top (\vx_k + \varepsilon \vw_k) + b_k > 0$,
            \item $\vw_k^\top (\vx_k - \varepsilon \vw_k) + b_k < 0$, and 
            \item $\sign(\vw_j^\top (\vx_k \pm \varepsilon \vw_k) + b_j) = \sign(\vw_j^\top \vx_k + b_j)$ for all $j \ne k$.
        \end{enumerate} 
        Thus, the points $\vx_k \pm \varepsilon \vw_k$ lie on opposite sides of $H_k$, and for $j \ne k$, the points $\vx_k \pm \varepsilon \vw_k$ are on the same side of $H_j$ as $\vx_k$.
        
        We now consider small perturbations of $\vx_k \pm \varepsilon \vw_k$ in the direction of $\vu$. Choose $\delta > 0$ sufficiently small so that $\vx_k \pm \varepsilon \vw_k +\delta \vu \in \densitysupp^\mathrm{o}$, $\delta |\vw_k^\T\vu| < \varepsilon \|\vw_k\|_2^2$, and $\delta |\vw_j^\T \vu| < |\vw_j^\top (\vx_k \pm \varepsilon \vw_k) + b_j|$ whenever $j \ne k$.
        This guarantees that 
        \begin{enumerate}\setlength{\itemindent}{0.4in}
            \item $\vw_k^\top (\vx_k + \varepsilon \vw_k + \delta \vu) + b_k > 0$, 
            \item $\vw_k^\top (\vx_k - \varepsilon \vw_k + \delta \vu) + b_k < 0$, and 
            \item $\sign(\vw_j^\top (\vx_k \pm \varepsilon \vw_k + \delta \vu) + b_j) = \sign(\vw_j^\top \vx_k + b_j)$ for all $j \ne k$.
        \end{enumerate} 
        That is, for every unit $j\in [K]$, the points $\vx_k \pm \varepsilon \vw_k + \delta$ are on the same side of $H_j$ as $\vx_k \pm \varepsilon \vw_k$. Additionally, \Cref{lem:grad ae 0 means constant} implies that $f(\vx_k \pm \varepsilon \vw_k + \delta \vu) = f(\vx_k \pm \varepsilon \vw_k)$.

        Because of this, it is straightforward to verify that
        \begin{equation}
        \label{eq:negative side}
            0 = f(\vx_k - \varepsilon \vw_k + \delta \vu) - f(\vx_k - \varepsilon \vw_k)
            = \sum_{\substack{j \in [K] \\ \vw_j^\T\vx_k + b_j > 0}} \delta a_j \vw_j^\T \vu + \delta \vv^\T\vu.
        \end{equation}
        On the other hand, $\vx_k + \varepsilon \vw_k + \delta \vu$ and $\vx_k + \varepsilon \vw_k$ are also active on unit $k$, and so
        \begin{equation}
        \label{eq:positive side}
            0 = f(\vx_k + \varepsilon \vw_k + \delta \vu) - f(\vx_k + \varepsilon \vw_k)
            = \sum_{\substack{j \in [K] \\ \vw_j^\T\vx_k + b_j \ge 0}} \delta a_j \vw_j^\T \vu  + \delta \vv^\T\vu.
        \end{equation}
        Subtracting \Cref{eq:negative side} from \Cref{eq:positive side} yields $0 = \delta a_k \vw_k^\T \vu$. 
        Hence, $\vw_k^\T \vu = 0$. Since $\vu$ was arbitrary, we get $\vw_k \in \nullspace(\gradcov{f})^\perp$. 
        Since this holds for all $k \in [K]$, it follows from \eqref{eq:positive side} that $\vv$ lies in $\nullspace(\gradcov{f})^\perp$ as well. 
    \end{proof}
Using \Cref{lem:weights of canonical form}, we now finish the proof of \Cref{lem:rank of f is rank of DaW} when $\densitysupp$ is a bounded convex set by choosing a suitable $\theta'$.
\begin{proof}
    If $\densitysupp$ is a bounded convex set, we rewrite the parameterization
    \begin{equation}
    \label{eq:param of f}
        f(\vx) = h_\theta^{(2)}(\vx) = \sum_{k = 1}^K a_k [\vw_k^\T\vx + b_k]_+ + c,~~\forall\vx \in \densitysupp.
    \end{equation}
    in a way that allows us to apply \Cref{lem:weights of canonical form}. 
    For convenience, we assume without loss of generality that $\|\vw_k\|_2 = 1$ for all $k$. (We may always rescale $a_k$ and $\vw_k$ to ensure that this is true without changing the matrix $\mD_{\va} \mW$.)
    We consider several types of units in $\theta$, and partition $[K]$ accordingly as follows.
    \begin{itemize}\setlength{\itemindent}{0.4in}
        \item $\Gamma_1 = \{k\in [K]: H_k \cap \densitysupp^\mathrm{o} \ne \emptyset\}$: These units have active sets that intersect the interior of $\densitysupp$ and can be combined into units with distinct active set boundaries plus an affine term.
        \item $\Gamma_2 = \{k \in [K]: \vw_k^\T \vx + b_k \ge 0 \;\forall \vx \in \densitysupp\}$: These units are active on the entirety of $\densitysupp$ and so can be combined into an affine term.
        \item $\Gamma_3 = \{k \in [K]: \vw_k^\T \vx + b_k \le 0 \;\forall \vx \in \densitysupp\}$: These units are active on none of $\densitysupp$ and so are immediately discarded. 
    \end{itemize}
    
    We further distinguish between different units in $\Gamma_1$ based on which ones share an active set boundary, whether units that share an active set boundary cancel out, and which side of shared active set boundaries are active. Formally, define the equivalence relation $\sim$ on $\Gamma_1$ by $j \sim k$ if $H_k = H_j$. Each equivalence class modulo $\sim$ contains units that share an active set boundary. Define
    \begin{itemize}\setlength{\itemindent}{0.4in}
        \item $\Gamma_1^0 = \{k \in \Gamma_1  : \sum_{j \sim k} a_j = 0\}$
        \item $\Gamma_1^1 = \{k \in \Gamma_1  : \sum_{j \sim k} a_j \ne 0\}$
    \end{itemize}
    We denote the set of equivalence classes of $\Gamma_1^1$ modulo $\sim$ by  $\Gamma_1^1/\sim$. Let $T^1$ be a transversal of $\Gamma_1^1/\sim$. Since the weights $\vw_k$ are all normalized so that $\|\vw_k\|_2 = 1$, note that $j \sim k$ if and only if $(\vw_j,b_j) = \pm(\vw_k,b_k)$. To distinguish between the $(\vw_j,b_j) = (\vw_k,b_k)$ and $(\vw_j,b_j) = (-\vw_k,-b_k)$ cases, we write 
    \begin{itemize}\setlength{\itemindent}{0.4in}
        \item $\Gamma_1^{1+} = \{j \in \Gamma_1^1  : (\vw_j,b_j) = (\vw_k,b_k) \text{ for some } k \in T^1\}$
        \item $\Gamma_1^{1-} = \{j \in \Gamma_1^1  : (\vw_j,b_j) = (-\vw_k,-b_k) \text{ for some } k \in T^1\}$
    \end{itemize}
    We similarly define $T^0$, $\Gamma_1^{0+}$ and $\Gamma_1^{0-}$. 

    Now given $\vx \in \densitysupp$, we use the identity $[-t]_+ = [t]_+ - t$ to see that
    \begin{align}
        \sum_{k \in \Gamma_1^1} a_k [\vw_k^\T\vx + b_k]_+
        &= \sum_{k \in T^1} \sum_{j \sim k} a_j [\vw_j^\T\vx + b_j]_+ \label{eq:gamma11 term less simplified}\\
        &= \sum_{k \in T^1} \left(
              \sum_{\substack{j \sim k \\ \vw_j = \vw_k}} a_j [\vw_k^\T\vx + b_k]_+ 
            + \sum_{\substack{j \sim k \\ \vw_j = -\vw_k}} a_j [-\vw_k^\T\vx - b_k]_+ \right)\\
        &= \sum_{k \in T^1} \left(
              \sum_{\substack{j \sim k}} a_j [\vw_k^\T\vx + b_k]_+ 
            - \sum_{\substack{j \sim k \\ \vw_j = -\vw_k}} a_j (\vw_k^\T\vx + b_k) \right)\\
        &= \left(\sum_{k \in T^1} 
              \sum_{\substack{j \sim k}} a_j [\vw_k^\T\vx + b_k]_+ \right)
        + \sum_{j \in \Gamma_1^{1-}} a_j \vw_j^\T\vx + C \label{eq:gamma11 term}
    \end{align}
    where $+C$ denotes a term that is constant with respect to $\vx$.\footnote{Note that the value of $C$ may change from line to line in this proof.}
    A nearly identical derivation shows that 
    \begin{equation}
        \sum_{k \in \Gamma_1^0} a_k [\vw_k^\T\vx + b_k]_+ = \sum_{j \in \Gamma_1^{0-}} a_j \vw_j^\T\vx + C. \label{eq:gamma10 term}
    \end{equation}
    Additionally, since the units in $\Gamma_2$ are active on the entirety of $\densitysupp$, 
    \begin{equation}
        \sum_{k \in \Gamma_2} a_k [\vw_k^\T\vx + b_k]_+ = \sum_{k \in \Gamma_2} a_k \vw_k^\T\vx + C. \label{eq:gamma2 term}
    \end{equation}

    Using \Cref{eq:gamma10 term,eq:gamma11 term,eq:gamma2 term}, we rewrite \Cref{eq:param of f} as follows:
    \begin{equation}
    \label{eq:expression to which canonical form lemma}
        f(\vx) = 
        \sum_{k \in T^1} 
              \left(\sum_{\substack{j \sim k}} a_j\right) [\vw_k^\T\vx + b_k]_+ 
        + \sum_{k \in \Gamma_1^{0-} \cup \Gamma_1^{1-} \cup \Gamma_2} a_k \vw_k^\T\vx 
        + C,~~\forall\vx \in \densitysupp.
    \end{equation}
    \Cref{lem:weights of canonical form} applies to this form and tells us that the vectors $\vw_k$ for $k \in T^1$ lie in the range of $\gradcov{f}$. For any $j \in \Gamma_1^1$, the vector $\vw_j$ is co-linear with some vector $\vw_k$ with $k \in T^1$, and so $\vw_j$ lies in the range of $\gradcov{f}$ as well.
    \Cref{lem:weights of canonical form} also tells us that the vector $\sum_{k \in \Gamma_1^{0-} \cup \Gamma_1^{1-} \cup \Gamma_2} a_k \vw_k$ lies in the range of $\gradcov{f}$. Since the vectors $\vw_k$ corresponding to $\Gamma_1^{1-}$ are in the range of $\gradcov{f}$, we may subtract them from the sum. This allows us to conclude that $\sum_{k \in \Gamma_1^{0-} \cup \Gamma_2} a_k \vw_k$ is in the range of $\gradcov{f}$, though it is possible that some individual vectors in the sum are not. 

    The equation \Cref{eq:expression to which canonical form lemma} is very close to the parameterization $\theta'$ that we want. However, it is convenient to ensure that the matrix $\mD_{\va'} \mW'$ corresponds to a subset of the rows of $\mD_{\va} \mW \mP$ so that, similarly to the $\densitysupp = \R^d$ case, we can establish that $\rank(\mD_{\va'} \mW') \le \funcrank{f}$ and $\Phi_L(\mD_{\va'} \mW') \le \Phi_L(\mD_{\va}\mW)$. Additionally, the parameterization $h_{\theta'}^{(2)}$ must not include skip connections, so we need to convert the skip connection from \Cref{eq:expression to which canonical form lemma} into ReLU units. Since $\densitysupp$ is bounded, there is some $B \in \R$ such that $\|\vx\|_2 \le B$ for all $\vx \in \densitysupp$. Each $\vw_k$ term has norm 1, and so $\vw^\T P\vx + B \ge 0$ for all $\vx \in \densitysupp$.
    Putting this all together with our knowledge that certain vectors lie in the range of $\gradcov{f}$, we use \Cref{eq:gamma10 term,eq:gamma2 term} to observe that
    \begin{align}
        f(\vx) 
        &= \sum_{k \in \Gamma_1^1} a_k [\vw_k^\T\vx + b_k]_+
        + \sum_{k \in \Gamma_1^{0-} \cup \Gamma_2} a_k \vw_k^\T\vx 
        + C \\
        &= \sum_{k \in \Gamma_1^1} a_k [\vw_k^\T\mP\vx + b_k]_+
        + \sum_{k \in \Gamma_1^{0-} \cup \Gamma_2} a_k \vw_k^\top\mP\vx 
        + C \\
        &= \sum_{k \in \Gamma_1^1} a_k [\vw_k^\T\mP\vx + b_k]_+
        + \sum_{k \in \Gamma_1^{0-} \cup \Gamma_2} a_k[\vw_k^\top\mP\vx + B]_+
        + C \label{eq:final parameterization}
    \end{align}
    Choosing $\theta'$ be this final parameterization from \Cref{eq:final parameterization} means that the matrix $\mD_{\va'} \mW'$ corresponds to a subset of the rows of $\mD_{\va} \mW \mP$, so $\rank(\mD_{\va'} \mW') \le \funcrank{f}$ and $\Phi_L(\mD_{\va'} \mW') \le \Phi_L(\mD_{\va}\mW)$ just as in the proof of the $\densitysupp = \R^d$ case.
\end{proof}

\subsection{Proofs of \Cref{lem:singular value decay} and \Cref{lem:modified singular value decay}}
\label{sec:proof of singular value decay}
\begin{proof}
    Fix $s,t \in [d]$. The lower bound from \Cref{thm:relationship between RL R2 index rank and mixed var} tells us that for any $f \in \setofnns{\densitysupp}$, 
    \begin{equation}
        R_L(f) 
        \geq \mixedvar{f}{\frac{2}{L}}^{2/L}
        = 
            \sum_{i=1}^d \mixedvarsv{f}{i}^{\frac{2}{L}}
        \ge t \mixedvarsv{f}{t}^{\frac{2}{L}}. \label{eq:lower bound via sv}
    \end{equation}
    On the other hand,
    \begin{align}
        &\inf_{f \in \setofnns{\densitysupp}} R_L(f)\st~f(\vx_i) = y_i~\forall i \in [n] \\
        &\qquad\le \inf_{f \in \setofnns{\densitysupp}} R_L(f)\st~f(\vx_i) = y_i~\forall i \in [n] \text{~and~} \funcrank{f}\leq s\\
        &\qquad\le s^{\frac{L-2}{L}} \inf_{f \in \setofnns{\densitysupp}} R_2(f)^{2/L}\st~f(\vx_i) = y_i~\forall i \in [n] \text{~and~} \funcrank{f}\leq s\\
        &\qquad= s^{\frac{L-2}{L}} \interpcost_s(\mathcal{D})^{2/L} \label{eq:upper bound on cost of interpolant}
    \end{align}
    where the second inequality comes from the upper bound in \Cref{thm:relationship between RL R2 index rank and mixed var}.

    Now for \Cref{lem:singular value decay}, if $\hat f$ is an $R_L$-minimal interpolant, then 
    \begin{equation}
        R_L(\hat f) = \inf_{f \in \setofnns{\densitysupp}} R_L(f)\st~f(\vx_i) = y_i~\forall i \in [n].
    \end{equation} 
    Using \Cref{eq:lower bound via sv,eq:upper bound on cost of interpolant}, we may conclude that
    \begin{equation}
    \label{eq:singular value bound}
        t \mixedvarsv{\hat f}{t}^{\frac{2}{L}} \le s^{\frac{L-2}{L}} \interpcost_s(\mathcal{D})^{2/L}.
    \end{equation} 
    For \Cref{lem:modified singular value decay}, if $\hat f$ satisfies \Cref{eq:near minimal RL interpolant}, then similarly \Cref{eq:lower bound via sv,eq:upper bound on cost of interpolant} imply that
    \begin{equation}
    \label{eq:singular value bound with alpha}
        t \mixedvarsv{\hat f}{t}^{\frac{2}{L}} \le (1+\alpha) s^{\frac{L-2}{L}} \interpcost_s(\mathcal{D})^{2/L}.
    \end{equation} 
    If $\hat f$ satisfies \Cref{eq:near regularized empirical risk minimizer} then,
    \begin{align}
        \lambda R_L(\hat f) 
        &\le 
        \frac{1}{n}\sum_{i=1}^n |y_i - \hat f(\vx_i)|^2 + \lambda R_L(\hat f) \\
        &\le (1+\alpha) \left(\inf_{f \in \setofnns{\densitysupp}}\frac{1}{n}\sum_{i=1}^n |y_i - f(\vx_i)|^2 + \lambda R_L(f)\right) \\
        &\le (1+\alpha) \left(\inf_{\substack{f \in \setofnns{\densitysupp} \\ f(\vx_i) = y_i \forall i \in [n]}}\frac{1}{n}\sum_{i=1}^n |y_i - f(\vx_i)|^2 + \lambda R_L(f)\right) \\
        &= \lambda (1+\alpha) \left(\inf_{\substack{f \in \setofnns{\densitysupp} \\ f(\vx_i) = y_i \forall i \in [n]}} R_L(f)\right).
    \end{align}
    With \Cref{eq:lower bound via sv,eq:upper bound on cost of interpolant}, this implies \Cref{eq:singular value bound with alpha} holds as well in this case.
    In all cases, \Cref{lem:singular value decay,lem:modified singular value decay} follow from rearranging \Cref{eq:singular value bound,eq:singular value bound with alpha}, respectively, and minimizing over $s$.
\end{proof}

\subsection{Proof of \Cref{lem:singval_lower_bnd}}
\label{sec:lower bound on singular values proof}
The proof is a consequence of the following key lemma that lower bounds the sum of squares of the singular values of any Lipschitz continuous data interpolating function. Below we use $\mathrm{Lip}(f)$ to denote the minimum Lipschitz constant of $f$ on $\gX$, i.e., the infimum over all constants $c\geq 0$ such that $|f(\vx)-f(\vy)| \leq c\|\vx-\vy\|$ for all $\vx,\vy\in \gX$.
\begin{lemma}\label{lem:singval_lower_bnd_interpolant}
    Let $\Omega \subset \gX$ be as in \Cref{lem:singval_lower_bnd}. Then for any Lipschitz function $f:\gX\rightarrow \R$ that interpolates the data (i.e., $f(\vx_i) = y_i$ for all $i$) we have
    \begin{equation}
        \sum_{k=1}^d \sigma_k(f)^2 
        \geq C 
        \frac{(\min_{c\in\mathbb{R}} \max_{i: \vx_i \in \Omega} |y_i-c|)^{d+2}}
        {\mathrm{Lip}(f)^d},
    \end{equation}
    where $C>0$ is a universal constant depending on $\Omega$, $\rho$, and $d$, but independent of $f$ and the data.
\end{lemma}
\begin{proof}
    First, it is straightforward to verify that
    \begin{equation}
        \sum_{k=1}^d \sigma_k(f)^2 = tr(\mC_{f,\rho}) = \int_{\mathcal X} \|\nabla f(\vx)\|_2^2 \rho(\vx) d\vx,
    \end{equation}
    Also, by assumption, there exists a constant $C_1 > 0$ such that $\rho(\vx) \geq C_1$ for all $\vx \in \Omega$, and so
    \begin{equation}
        \int_{\mathcal X} \|\nabla f(\vx)\|_2^2 \rho(\vx) d\vx
    \ge C_1 \int_{\Omega} \|\nabla f(\vx)\|_2^2 d\vx.
    \end{equation}
    
    Hence, it suffices to lower bound $\|\nabla f\|_{L^2(\Omega)}^2 = \int_{\Omega} \|\nabla f(\vx)\|_2^2 d\vx$. 
    
    Towards this end, define $\overline{f}_\Omega = \frac{1}{|\Omega|}\int_\Omega f(\vx)d\vx$ where $|\Omega|$ denotes the Lebesgue measure of $\Omega$. By a Sobolev inequality (see, e.g., \cite{evans2010partial} Section 5.6.2) we have
    \begin{equation}
        \|f-\overline{f}\|_{L^\infty(\Omega)} \lesssim_{\Omega,d} \|f-\overline{f}\|_{L^{d+2}(\Omega)} + \|\nabla f\|_{L^{d+2}(\Omega)},
    \end{equation}
    where the notation $A \lesssim_{\Omega,d} B$ indicates $A \leq C B$ for a universal constant $C$ depending only on $\Omega$ and $d$.
    Furthermore, by Poincar\'e's inequality (see, e.g., \cite{evans2010partial} Section 5.8.1), we have
    \begin{equation}
        \|f-\overline{f}\|_{L^{d+2}(\Omega)} \lesssim_{\Omega,d} \|\nabla f\|_{L^{d+2}(\Omega)},
    \end{equation}
    and so combining the two inequalities above gives
    \begin{equation}
    \|f-\overline{f}\|_{L^\infty(\Omega)} \lesssim_{\Omega,d} \|\nabla f\|_{L^{d+2}(\Omega)}.
    \end{equation}
    Next, since $2 < d+2 < \infty$, an $L^p$-norm interpolation inequality gives
    \begin{equation}
        \|\nabla f\|_{L^{d+2}(\Omega)} \leq \|\nabla f\|_{L^2(\Omega)}^{\frac{2}{d+2}} \|\nabla f\|_{L^{\infty}(\Omega)}^{\frac{d}{d+2}}.
    \end{equation}
    Also, since $\Omega \subset \gX$ we have $\|\nabla f\|_{L^{\infty}(\Omega)} \leq \|\nabla f\|_{L^{\infty}(\gX)}$, while Rademacher's Theorem gives $\|\nabla f\|_{L^{\infty}(\gX)} = \textrm{Lip}(f)$. Therefore, we have shown
    \begin{equation}
    \|f-\overline{f}\|_{L^\infty(\Omega)} \lesssim_{\Omega,d} \textrm{Lip}(f)^{\frac{d}{d+2}}\|\nabla f\|_{L^2(\Omega)}^{\frac{2}{d+2}},
    \end{equation}
    which implies
    \begin{equation}
    \frac{\|f-\overline{f}\|_{L^\infty(\Omega)}^{d+2}}{\textrm{Lip}(f)^d} \lesssim_{\Omega,d} \|\nabla f\|_{L^2(\Omega)}^2.
    \end{equation}
    Finally, since $f$ satisfies $f(\vx_i) = y_i$ for all $i$, we have
    \begin{equation}
        \|f-\overline{f}\|_{L^\infty(\Omega)} \geq \max_{i: \vx_i \in \Omega} |y_i-\overline{f}| \geq \min_{c\in\R} \max_{i: \vx_i \in \Omega}|y_i-c|
    \end{equation}
    Combining this inequality with the one above gives the claim.
    \end{proof}

To finish the proof of \Cref{lem:singval_lower_bnd}, all that remains is to bound the Lipschitz constant of minimal $R_L$-cost interpolants. This is achieved with the next two lemmas.

\begin{lemma}\label{lem:lip_bnd} Suppose $f \in \gN_2(\gX)$. Then $\mathrm{Lip}(f) \leq R_2(f)$.
\end{lemma}
\begin{proof} Suppose $f(\vx) = \sum_{k=1}^K a_k[\vw_k^\T\vx + b_k]_+ + c$ is any parameterization of $f$ such that $\|\vw_k\|=1$ for all $k = 1,...,K$. Then a weak gradient of $f$ is given by
\begin{equation}
    \nabla f(\vx) = \sum_k \heaviside(\vw_k^\T\vx + b_k)a_k\vw_k
\end{equation}
where $\heaviside(\cdot)$ is the unit step function. Also, for any $\vx \in \gX$ we have
\begin{equation}
    \|\nabla f\|_{L^\infty(\gX)} \leq \sum_k \|H(\vw_k^\T\vx + b_k)a_k\vw_k\| \leq \sum_k |H(\vw_k^\T\vx + b_k)||a_k|\|\vw_k\| \leq \sum_k |a_k|
\end{equation}
Therefore, by taking the infimum over all such parameterizations of $f$, and using the characterization of the $R_2$-cost given in \eqref{eq:opt2}, we see that 
\begin{equation}
    \|\nabla f\|_{L^\infty(\gX)}\leq R_2(f),
\end{equation}
Finally, by Rademacher's Theorem, we have $\mathrm{Lip}(f) = \|\nabla f\|_{L^\infty(\gX)}$, which gives the claim.
\end{proof}

\begin{lemma}\label{lem:lip_bnd_RLmin} Let $\hat{f}$ be a minimum $R_L$-cost interpolant of the data $\gD$. Then $\mathrm{Lip}(\hat{f}) \leq \interpcost_1(\gD)$.
\end{lemma}
\begin{proof}
Let $f_1$ be a minimum $R_2$-cost index-rank-one interpolant of the data $\gD$, such that $\interpcost_1(\gD) = R_2(f_1)$. Since $\hat{f}$ is a $R_L$-cost minimizer, we have
\begin{equation}
R_L(\hat{f})^{L/2} \leq R_L(f_1)^{L/2} = R_2(f_1) = \interpcost_1(\gD)
\end{equation}
where the equality $R_L(f_1)^{L/2} = R_2(f_1)$ follows from \Cref{thm:relationship between RL R2 index rank and mixed var} and noting that $f_1$ has index-rank one. On the other hand, by \Cref{lem:lip_bnd,thm:relationship between RL R2 index rank and mixed var} we have 
\begin{equation}
\text{Lip}(\hat{f}) \leq R_2(\hat{f}) \leq R_L(\hat{f})^{L/2}.
\end{equation}
Combining the two inequalities above gives the desired result.
\end{proof}

\Cref{lem:singval_lower_bnd} now follows directly from \Cref{lem:singval_lower_bnd_interpolant} with $f = \hat{f}$ and using the bound $\mathrm{Lip}(\hat{f}) \leq \interpcost_1(\gD)$ given in \Cref{lem:lip_bnd_RLmin}.

\subsection{Proofs of claims in \Cref{ex:tworays}}\label{sec:sm:tworays}
\begin{figure}[ht!]
\includegraphics[width=\textwidth]{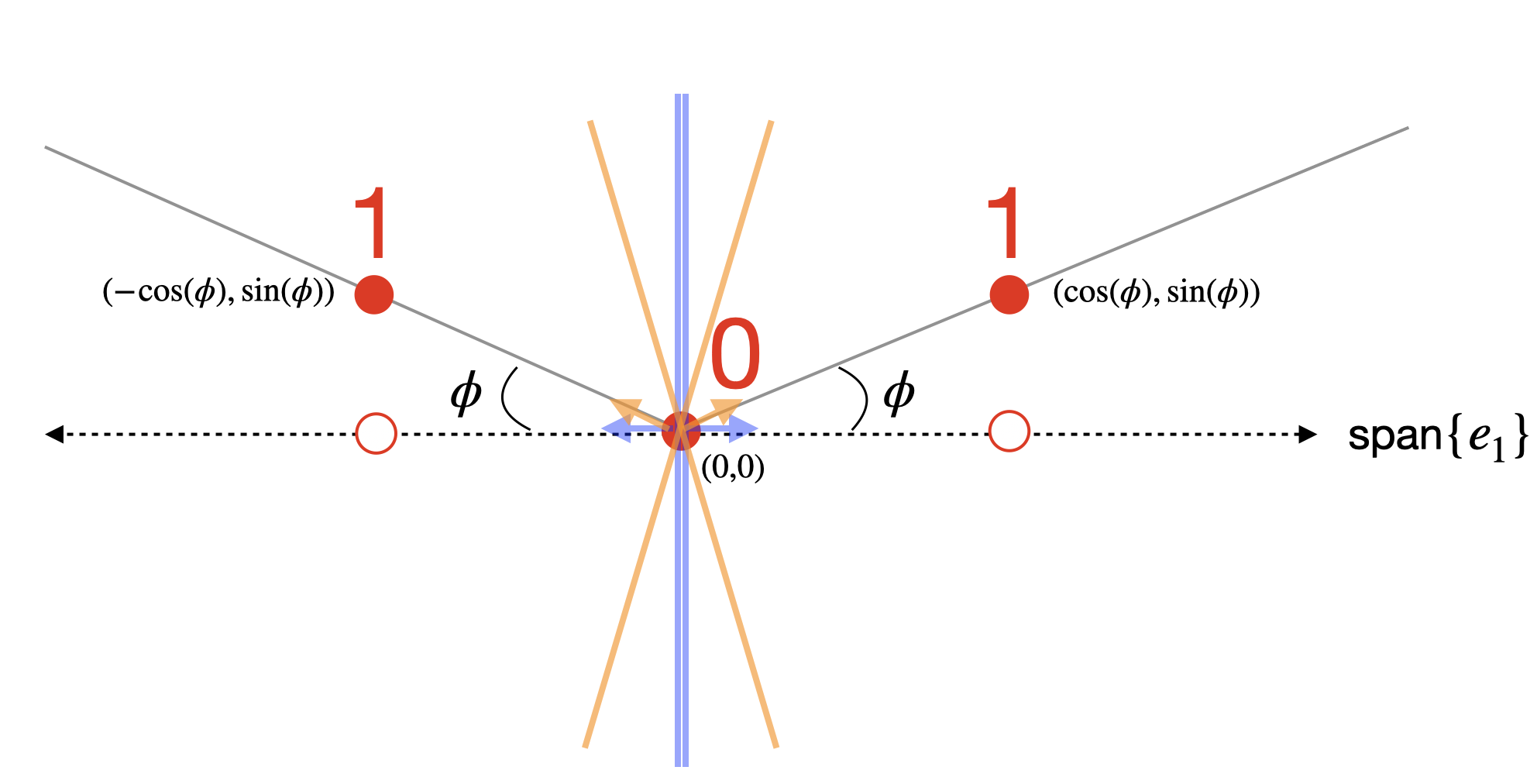}
\caption{Dataset $\gD$ consisting of three training points in $\R^2$. Closed red circles indicate the training features with the corresponding labels in red. Orange lines show the boundaries of the two ReLU units of the minimum $R_2$-cost interpolant. Blue lines indicate the ReLU boundaries of the index rank one data generating function $f^*(\vx) = \cos(\phi)^{-1}|x_1| =  \cos(\phi)^{-1}([\ve_1^\T \vx]_+ + [-\ve_1^\T \vx]_+)$.}\label{fig:tworays}
\end{figure}
Let $\phi$ be any angle in the range $(0,\pi/6)$, and define $\vw_+ = [\cos(\phi),\sin(\phi)]$ and $\vw_- = [-\cos(\phi),\sin(\phi)]$. Consider the dataset consisting of three training pairs: 
\[
\gD = \{(\bm 0, 0), (\vw_+,1), (\vw_-,1)\}
\]
as shown in \Cref{fig:tworays}. We prove the following:
\begin{prop}
 The function $\hat{f}_2(\vx) = [\vw_+^\T\vx]_+ + [\vw_-^\T\vx]_+$ is the unique minimum $R_2$-cost interpolant of the dataset $\gD$. Furthermore, assuming the domain $\gX$ is either a Euclidean ball centered at the origin or all of $\R^2$, and $\rho$ is a radially symmetric probability density function, we have $\sigma_2(\hat{f}_2) > \sin(\phi)$.
\end{prop}
\begin{proof}
Let $f$ be any $R_2$-cost minimizer that interpolates $\gD$. Then $f$ can be written as
\[
f(\vx) = \sum_{k=1}^K a_k[\vw_k^\top\vx-b_k]_+ + c
\]
where $\|\vw_k\|_2 = 1$ and $a_k \neq 0$ for all $k\in [K]$, with $R_2(f) = \sum_{k=1}^K |a_k|$.

First, we prove that $f$ being a $R_2$-cost minimizer implies certain geometric constraints on its ReLU units. Below, we use $u$ to denote a generic ReLU unit in $f$, i.e., $u(\vx) = a_k [\vw_k^\top\vx-b_k]_+$ for some $k$. We define the active set of $u$ to be the set of inputs $\vx$ such that $u(\vx) \neq 0$, and consider the following cases:

\vspace{0.5em}

\noindent\emph{Case 1: The active set of $u$ contains none of the training points.}

If this were the case, the unit $u$ could be removed while strictly reducing the $R_2$-cost while still satisfying the interpolation constraints, contradicting the fact that $f$ is an $R_2$-cost minimizer. Therefore, this case is impossible.

\vspace{0.5em}

\noindent\emph{Case 2: The active set of $u$ contains one training point.} 

Suppose the active set of $u$ contains only $\vw_+$. Let $u(\vw_+) = \alpha \ne 0$. Consider the ReLU unit $u_0(\vx) = \alpha[\vw_+^\top\vx]_+$, which also satisfies $u_0(\vw_+) = \alpha$ and vanishes over the other two training points. We prove that $u = u_0$. Recall that $u(\vx) = a[\vw^\top\vx - b]_+$ where $\|\vw\|_2=1$. The constraint $u(\vw_+) = \alpha$ implies $|a| = |\alpha|/(\vw^\T\vw_+-b)$. Also, since $0 = u(\bm 0) =[-b]_+$, we see that $b\geq0$, which implies $\vw^\T\vw_+>0$. By the Cauchy-Schwarz inequality, we have $0 < \vw^\T\vw_+-b \leq 1$ with equality if and only if $b=0$ and $\vw = \vw_+$. This shows $R_2(u) = |a| \geq |\alpha| = R_2(u_0)$ where equality holds if and only if $b=0$ and $\vw = \vw_+$, or equivalently, $u=u_0$. Therefore, if it were the case that $u\neq u_0$, then $R_2(f) > R_2(f-u+u_0)$, contrary to our assumption that $f$ was a $R_2$-cost minimizer, and so it must be the case that $u=u_0$.

An argument parallel to the above shows that if the active set of $u$ contains only $\vw_-$, and $u(\vw_-) = \alpha$, then $u(\vx) = \alpha[\vw_-^\top\vx]_+$.

The last case to consider is where the active set of $u$ contains only $\bm 0$. Let $u(\bm 0) = \alpha$. Consider the ReLU unit  $u_0(\vx) = \alpha (\sin\phi)^{-1} [-\ve_2^\top\vx + \sin\phi]_+$. Then $u_0(\bm 0) = \alpha$, while $u_0(\bm \vw_+) = u_0(\vw_-) = 0$. A similar argument to the above shows that $u_0$ is the unique $R_2$-cost minimizer under the constraints that the active set of $u$ contains only $\bm 0$ and $u(\bm 0) = \alpha$.

\vspace{0.5em}

\noindent\emph{Case 3: The active set of $u$ contains two training points.}

First, suppose the active set of $u$ contains both $\vw_+$ and $\vw_-$, but not $\bm 0$. We show this is not possible, since such a unit could be replaced with two units at a lower $R_2$-cost. In particular, let $u(\vw_+) = \alpha_+$ and $u(\vw_-) = \alpha_-$. Note that $\alpha_+$, $\alpha_-$ must have the same sign. Without loss of generality, we assume $\alpha_+,\alpha_->0$. Consider the units $u_+(\vx) = \alpha_+[\vw_+^\top\vx]_+$ and $u_-(\vx) = \alpha_-[\vw_-^\top\vx]_+$, and let $u_0 = u_+ + u_-$. Then $R_2(u_0) = \alpha_- + \alpha_+$, and $u_0$ matches the output of $u$ over the training points. We show the $R_2$-cost of $u$ must be greater than $u_0$. 

Recall $u(\vx) = a[\vw^\T\vx-b]_+$ with $\|\vw\|_2=1$. Define $\tilde{\vw} = a\vw$ and $\tilde{b} = ab$, so that $u(\vx)=[\tilde{\vw}^\T\vw-\tilde{b}]_+$ and $R_2(u) = \|\tilde{\vw}\|_2$. Then the constraints $u(\vw_+) = \alpha_+$ and $u(\vw_-) = \alpha_-$ imply
\begin{align*}
\tilde{\vw}^\top\vw_+ - \tilde{b} & = \alpha_+,\\
\tilde{\vw}^\top\vw_- - \tilde{b} & = \alpha_-,
\end{align*}
and adding the equations above gives 
\begin{align*}
\tilde{\vw}^\top(\vw_++\vw_-) - 2\tilde{b} = 2\tilde{w}_2\sin(\phi) - 2\tilde{b} = \alpha_+ + \alpha_- & \iff \tilde{w}_2 = \tfrac{\alpha_++\alpha_-+2\tilde{b}}{2\sin(\phi)}.
\end{align*}
This gives the lower bound
\[
    R_2(u) = \|\tilde{\vw}\|_2 \geq |\tilde{w}_2| = \frac{|\alpha_++\alpha_- + 2\tilde{b}|}{2\sin(\phi)} > \alpha_++\alpha_- +  2\tilde{b} \geq \alpha_++\alpha_-,
\]
where the strict inequality above follows from our assumption that $0 < \sin(\phi) < 1/2$, and the final inequality holds since $\tilde{b}\geq0$ because $u(\bm 0) = 0$. This shows the $R_2(u) > \alpha_+ + \alpha_- = R_2(u_0)$ contradicting the fact that $f$ is an $R_2$-cost minimizer.

Next, suppose the active set of $u$ contains both $\bm 0$ and $\vw_+$, but not $\vw_-$. Let $u(\bm 0) = \alpha_0$ and $u(\vw_+) = \alpha_+$. Again, without loss of generality, we assume $\alpha_0,\alpha_+>0$. We will prove that for $u(\vx) = a[\vw^\T\vx - b]_+$ must be the case that $\vw = \vw_+$ or $b = -\vw^\T\vw_-$.

Again, define $\tilde{\vw} = a\vw$ and $\tilde{b} = ab$, so that $u(\vx) = [\tilde{\vw}^\T\vx -\tilde{b}]_+$. The constraint $u(\bm 0) = \alpha_0$ implies $\tilde{b} = -\alpha_0$. We show that interpolation constraints determine $\vw$ up to a single free parameter $\delta \in (0,1]$. In particular, define $\delta = \|\vx_0\|_2$ where $\vx_0$ is the unique intersection point of the ReLU boundary $\{\vx : \tilde{\vw}^\top\vx = \tilde{b}\}$ and the ray $\{\beta \vw_- : \beta > 0\}$. Then we have $\beta \tilde{\vw}^\top\vw_- = \tilde{b}$, or equivalently $\beta = \tilde{b}/(\tilde{\vw}^\top\vw_-) = -\alpha_0/(\tilde{\vw}^\top\vw_-)$, and so $\vx_0 = -\frac{\alpha_0}{\tilde{\vw}^\top\vw_-} \vw_-$. This gives $\delta = \|\vx_0\|_2 = \frac{\alpha_0}{\tilde{\vw}^\top\vw_-}$, or equivalently,
\[
\tilde{\vw}^\top\vw_- = -\frac{\alpha_0}{\delta}.
\]
Also, from the constraint $u(\vw_+) = \alpha_+$ we have
\[
\tilde{\vw}^\top\vw_+ - \tilde{b}  =  \tilde{\vw}^\top\vw_+ + \alpha_0 = \alpha_+ \iff \tilde{\vw}^\top \vw_+ = \alpha_+ - \alpha_0.
\]
Adding and subtracting equations above gives
\begin{align*}
\tilde{\vw}^\top(\vw_+-\vw_-) = 2\tilde{w}_1\cos(\phi) = \alpha_+- \alpha_0 +\alpha_0/\delta & \iff \tilde{w}_1 = \tfrac{\alpha_+-\alpha_0 + \alpha_0/\delta}{2\cos(\phi)},\\
\tilde{\vw}^\top(\vw_++\vw_-) = 2\tilde{w}_2\sin(\phi) = \alpha_+-\alpha_0-\alpha_0/\delta & \iff \tilde{w}_2 = \tfrac{\alpha_+-\alpha_0-\alpha_0/\delta}{2\sin(\phi)}.
\end{align*}
Therefore, 
 \[
\phi(\delta) := R_2(u)^2 = \|\tilde{\vw}\|^2 = \tilde{\vw}_1^2 + \tilde{\vw}_2^2 = \left(\tfrac{\alpha_+-\alpha_0 + \alpha_0/\delta}{2\cos(\phi)}\right)^2 + \left(\tfrac{\alpha_+-\alpha_0 - \alpha_0/\delta}{2\cos(\phi)}\right)^2.
\]
Observe that $\phi$ is a smooth function of $\delta \in (0,\infty)$. Basic calculus shows that $\phi$ has a unique critical point $\delta^*$ given by
\[
\delta^* = \frac{\alpha_0}{(\alpha_+-\alpha_0)\left(\cos^2\phi-\sin^2\phi\right)}.
\]
In the event that $\delta^* \in (0,1]$, then is easy to prove $\delta^*$ is the unique minimizer of $\phi$. Plugging in the value $\delta = \delta^*$ into the expressions for $\tilde{w}_1$ and $\tilde{w}_2$, we have
\[
\tilde{w}_1 = (\alpha_+-\alpha_0)\frac{(1 + \cos^2\phi - \sin^2\phi)}{2\cos\phi} = (\alpha_+-\alpha_0)\cos(\phi)
\]
and 
\[
\tilde{w}_2 = (\alpha_+-\alpha_0)\frac{(1 - \cos^2\phi + \sin^2\phi)}{2\sin\phi} = (\alpha_+-\alpha_0)\sin(\phi).
\]
This shows $\tilde{\vw} = (\alpha_+-\alpha_0)\vw_+$. Therefore, $u$ has the form $u(\vx) = a[\vw_+^\top \vx -b]_+$ where $\vw_+^\top\vw_- < b<0$.

On the other hand, when $\delta^* > 1$, the minimum of $\phi(\delta)$ for $\delta \in (0,1]$ occurs at $\delta = 1$. This implies the ReLU boundary of $u$ contains the point $\vw_-$, and $u$ has the form $u(\vx) = a[\vw(\vx-\vw_-)]_+$.

A parallel argument shows that if the active set of $u$ contains both $\bm 0$ and $\vw_-$, but not $\vw_+$, then either $u(\vx) = a[\vw_-^\T\vx -b]_+$ with $-1 < \vw_+^\T\vw_- < b < 0$, or $u(\vx) = a[\vw^\T(\vx-\vw_+)]_+$.

\vspace{0.5em}

\noindent\emph{Case 4: The active set of $u$ contains all three training points.}

In this case, cannot make any further simplifications to the form of $u$.

\vspace{0.5em}

The cases above show that $f$ must have the form
\begin{equation}\label{eq:frep}
f(\vx) = a_1 u_1(\vx) + a_2 u_2(\vx) + a_3 u_3(\vx) + \sum_{k=1}^M a_{2,k}u_{2,k}(\vx) + \sum_{j=1}^N a_{3,j}u_{3,j}(\vx) + c, 
\end{equation}
where $u_1(\vx) = [\vw_+^\top \vx]_+$, $u_2(\vx) = [\vw_-^\top \vx]_+$, $u_3(\vx) = [-\ve_2^\top \vx+\sin(\phi)]_+$, and each $u_{2,k}$ is a distinct ReLU unit of the form $[\vw_k^\T\vx - b_k]_+$ with $\|\vw_k\|_2 = 1$ whose active set contains either $\{\bm 0,\vw_+\}$ or $\{\bm 0,\vw_-\}$ and has the form specified in Case 3 above, while each $u_{3,j}$ is a distinct ReLU unit of the form $[\vw_j^\T\vx- b_j]_+$ whose active set contains all three points $\{\bm 0,\vw_+,\vw_-\}$.

Additionally, the coefficients $\va = (a_1,a_2,a_3,a_{2,1},...,a_{2,M},a_{3,1},...,a_{3,N}) \in \R^K$ in \eqref{eq:frep} are a minimizer of the convex optimization problem:
\begin{equation}\label{eq:primal}
p^* = \min_{c\in\R} \left(\min_{\va \in \R^K} \|\va\|_1 ~~s.t.~~\mV\va = \vy-c\bm 1\right),
\end{equation}
where $\va \in \R^K$ is the vector of all outer-layer weights, $\vy = [1, 1, 0]^\top$, $\bm 1 = [1,1,1]^\top \in \R^3$ and $\mV \in \R^{3\times W}$ is the matrix whose columns are the evaluations of one of the units at the three training points, so that $\mV\va = [f(\vw_+)-c,f(\vw_-)-c,f(\bm 0)-c]^\top$. In particular, if we sort the columns of $\mV$ such that first three columns correspond to units $u_1$, $u_2$, $u_3$, and the next $M$ columns correspond to units active over two training points (denoted by $u_{2,k}$), and the final $N$ columns correspond to units active over three training points (denoted by $u_{3,j}$), we have
\[
\mV = 
\begin{bmatrix}
1 & 0 & 0 & u_{2,1}(\vw_+) & \cdots & u_{2,M}( \vw_+) & u_{3,1}(\vw_+) & \cdots & u_{3,N}( \vw_+)  \\
0 & 1 & 0 & u_{2,1}(\vw_-) & \cdots & u_{2,M}(\vw_-) & u_{3,1}(\vw_-) & \cdots & u_{3,N}(\vw_-) \\
0 & 0 & \sin(\phi) & u_{2,1}(\bm 0) & \cdots & u_{2,M}(\bm 0) & u_{3,1}(\bm 0) & \cdots & u_{3,N}(\bm 0) 
\end{bmatrix}.
\]
Consider the pair $\va_0 = [1, 1, 0, \cdots, 0]$, $c_0=0$, which corresponds to the interpolant $f_0(\vx) = [\vw_+^\T\vx]_+ + [\vw_-^\T\vx]_+$, hence is feasible for \eqref{eq:primal}. We prove that $(\va_0,c_0)$ is the unique minimizer of \eqref{eq:primal}, which implies $f=f_0$, i.e., $f_0$ is the unique interpolating $R_2$-cost minimizer.

To do so, we make use of the following lemma, which shows that the existence of a specific vector in the row space of $\mV$ (known as a \emph{dual certificate} in the compressed sensing literature \cite{candes2014mathematics}) is sufficient to guarantee a feasible pair $(\va_*,c_*)$ is the unique minimizer of \eqref{eq:primal}. Below, we use $\mathrm{supp}(\va) = \{i \in [K] : a_i \neq 0\}$ to denote the non-zero support of the vector $\va$. Also, given an index set $\gJ \subset [K]$ we define $\mV_\gJ$ to be the submatrix obtained by restricting $\mV$ to columns indexed by $\gJ$, and for any vector $\vh \in \R^K$ we let $\vh_\gJ \in \R^{|\gJ|}$ denote the restriction of $\vh$ to its entries indexed by $\gJ$.

\begin{lemma}\label{lem:dual_cert}
Suppose $(\va_*,c_*)$ is feasible for \eqref{eq:primal}, i.e., $\mV\va_* = \vy - c_*\bm 1$. Let $\gI = \mathrm{supp}(\va_*)$. Assume $\mV_\gI$ is full rank, and $\bm 1 \not\in \mathrm{range}(\mV_\gI)$. Further, suppose there exists a vector $\vz_* \in \R^3$ with $\vz_*^\top \bm 1 = 0$ such that $\vq = \mV^\T\vz_*$ satisfies $q_i = \mathrm{sign}(\va_{*,i})$ for all $i \in \gI$, and $|q_i| < 1$ for all $i \in \gI^C$. Then $(\va_*,c_*)$ is the unique minimizer of \eqref{eq:primal}.
\end{lemma}
\begin{proof}
Suppose $(\va,c) \neq (\va_*,c_*)$ is feasible for \eqref{eq:primal}, i.e., $\mV\va = \vy - c\bm 1$. Define $\vh = \va-\va_*$. First, we show that $\vh_{\gI^C} \neq \bm 0$. By way of contradiction, suppose $\vh_{\gI^C} = \bm 0$. Then we have
\[
\mV_I \vh_I = \mV\vh = (c^*-c)\bm 1
\]
but by the assumption $\bm 1 \not\in \mathrm{range}(\mV_I)$, the only possibility is that $(c^*-c)\bm 1 = 
\bm 0$, or equivalently $c^* = c$. And by the assumption that $\mV_I$ is full rank, we must have $\vh_I = \bm 0$, which implies $\vh = \bm 0$, or equivalently, $\va = \va_*$. Hence, $(\va,c) = (\va_*,c_*)$, a contradiction.

Next, we have
\begin{align*}
\|\va\|_1 & = \|\va_* + \vh_\gI\|_1 + \|\vh_{\gI^C}\|_1\\
& >  \langle \va_* + \vh_\gI, \vq\rangle + \langle\vh_{\gI^C},\vq\rangle\\
& = \|\va_*\|_1 + \langle \vh,\vq\rangle\\
& = \|\va_*\|_1 + \langle \mV\vh,\vz_*\rangle\\
& = \|\va_*\|_1 + (c_*-c)\langle \bm 1,\vz_*\rangle\\
& = \|\va_*\|_1
\end{align*}
where the strict inequality comes from the fact that $\|\vh_{\gI^C}\|_1 > 0$ and $\|\vh_{\gI^C}\|_1 > \langle\vh_{\gI^C},\vq\rangle$ since we assume $|q_i| < 1$ for all $i\in \gI^C$. Therefore, $\|\va\|_1 > \|\va_*\|_1$ for all feasible $\va \neq \va_*$. Finally, if $\va = \va_*$, then $\vh = \bm 0$, and so $\bm 0 = \mV \vh = (c_*-c)\bm 1$, which implies $c = c_*$, showing $(\va_*,c_*)$ is the unique minimizer.
\end{proof}

First, observe that for $\gI = \text{supp}(\va_0) = \{1,2\}$, the submatrix $\mV_\gI$ is full rank, and $\bm 1 \notin\mathrm{range}(\mV_\gI)$. Next, we identify a vector $\vz_0$ that satisfies the conditions of \Cref{lem:dual_cert}.

Let $\vz_0 = [1,  1,  -2]^\top$ and $\vq = \mV^\top\vz_0$. Then 
\[
q_1 = 1, q_2 = 1, q_3 = -2\sin(\phi)
\]
where $|q_3| < 1$ by our assumption that $0 < \sin(\phi) < 1/2$. The remaining entries of $\vq$ have the form
\[
u_{m,k}(\vw_+)+u_{m,k}(\vw_-)-2u_{m,k}(\bm 0)
\]
for $m =2,3$. Now we show each of these entries must have absolute value strictly less than one.

Consider the case $m=2$, corresponding to units active over exactly two training points (i.e., units active over $\vw_+$ and $\bm 0$, or $\vw_-$ and $\bm 0$). For simplicity, let us write $u = u_{2,k}$, which has the form $u(\vx) = [\vw^\top\vx - b]_+$ with $\|\vw\|_2 = 1$. Without loss of generality, assume $u$ is active over $\vw_+$ and $\bm 0$ only. Therefore, we need to bound the quantity
\[
q = u(\vw_+)-2u(\bm 0).
\]
Previously, we identified two cases for the unit $u$: either $u(\vx) = [\vw_+^\top \vx-b]_+$ with $-1 < \vw_+^\top\vw_- < b < 0$, or $u(\vx) = [\vw^\top(\vx-\vw_-)]_+$. In the first case $q = (1-b)-2(-b) = 1+b$, and so $|q| = |1+b| < 1$. In the second case, we have $q = \vw^\top(\vw_+-\vw_-) - 2(-\vw^\top\vw_-) = \vw^\top(\vw_++\vw_-) = 2 w_2 \sin(\phi)$, and so $|q| = 2|w_2|\sin(\phi) \leq 2\sin(\phi) < 1$, since we assume $\|\vw\|_2  = 1$ and $0 < \sin(\phi) < 1/2$.

Now consider the case $m=3$, corresponding to units active over all three training points. For simplicity, let us write $u = u_{3,k}$, which has the form $u = [\vw^\top\vx - b]_+$ with $\|\vw\|_2 = 1$. Since $u$ is active over all three training points we have
\[
q = (\vw^\top\vw_+ - b) + (\vw^\top\vw_- - b) - 2(-b) = \vw^\top(\vw_++\vw_-) = 2w_2\sin(\phi)
\]
and so we have $|q| < 1$ by the same argument as above. Therefore, $\vz_0$ satisfies the requirements of Lemma \ref{lem:dual_cert}, which proves $(\va_0,c_0)$ is the unique minimizer of \eqref{eq:primal}, as claimed.

Finally, we compute the singular values of $f(\vx) = [\vw_+^\T\vx]_+ + [\vw_-^\T\vx]_+$ assuming the domain $\gX$ is a ball centered at the origin $\gX = \{\vx\in\R^d : \|\vx\|_2 \leq R\}$ or $\gX = \R^d$ and $\rho:\gX\rightarrow \R$ is any radially symmetric probability density function. 

First, we have
\[
\nabla f(\vx) = H(\vw_+^\T\vx)\vw_+ + H(\vw_-^\T\vx)\vw_-
\]
and so 
\begin{multline*}
\nabla f(\vx) \nabla f(\vx)^\T\\\qquad = H(\vw_+^\T\vx)\vw_+\vw_+^\T +  H(\vw_-^\T\vx)\vw_-\vw_-^\T + H(\vw_+^\T\vx)H(\vw_-^\T\vx)( \vw_+\vw_-^\T + \vw_-\vw_+^\T).
\end{multline*}
By radial symmetry of $\rho$, we have
\[
\int_\gX H(\vw_+^\T\vx) \rho(\vx) d\vx = \int_\gX H(\vw_-^\T\vx) \rho(\vx) d\vx = \frac{1}{2},
\]
and
\[
\int_\gX H(\vw_+^\T\vx) H(\vw_-^\T\vx)\rho(\vx) d(\vx) = \frac{\phi}{\pi}.
\]
Therefore, the EGOP matrix $\mC_f$ is given by
\begin{align*}
\mC_f & = \int_\gX \nabla f(\vx) \nabla f(\vx)^\T \rho(\vx) d\vx\\ &  = \frac{1}{2}\left(\vw_+\vw_+^\T + \vw_-\vw_-^\T\right) + \frac{\phi}{\pi} ( \vw_+\vw_-^\T + \vw_-\vw_+^\T)\\
 & = \begin{bmatrix} \cos^2\phi & 0 \\ 0 & \sin^2 \phi \end{bmatrix} + \frac{2\phi}{\pi}\begin{bmatrix} -\cos^2\phi & 0 \\ 0 & \sin^2 \phi \end{bmatrix}\\
 & = \begin{bmatrix} (1-\frac{2\phi}{\pi})\cos^2\phi & 0 \\ 0 & (1+\frac{2\phi}{\pi}) \sin^2 \phi \end{bmatrix}.
\end{align*}
and so the singular values of $\mC_f^{1/2}$ are given by
\[
\sigma_1 = \sqrt{1-\frac{2\phi}{\pi}}\cos\phi,~~\sigma_2 = \sqrt{1+\frac{2\phi}{\pi}}\sin\phi.
\]
In particular, $\sigma_2 > \sin\phi$.

\end{proof}

\section{Details of Numerical Experiments}
\label{app:experiment details}
All code can be found 
at the following link: 
\begin{center}
\href{https://github.com/suzannastep/linear_layers_experiments}{https://github.com/suzannastep/linear\_layers\_experiments}.
\end{center}
\paragraph{Data generation process}
We choose a universal training superset $\{\vx_i\}_{i=1}^{2048}$ where each $\vx_i \sim \uniform([-\frac{1}{2},\frac{1}{2}])$. 
For each $r \in \{1,2, 5\}$, 
we create an index-rank-$r$ function $f$ as described in \Cref{sec:experiments} where 
\begin{itemize}
    \item $\mV \; (20 \times r)$ is the first $r$ columns of a random orthogonal matrix, 
    \item $\mU \; (21 \times r)$ is the first $r$ columns of a random orthogonal matrix, 
    \item $\mathbf{\Sigma} \; (r \times r)$ is a diagonal matrix with entries drawn from $\uniform([0,100])$, 
    \item $\mW = \mU \mathbf{\Sigma} \mV^\top \; (21 \times 20)$, and 
    \item $\va$ and $\vb \; (21 \times 1)$ are vectors with entries drawn from the standard normal distribution and $\uniform([-\frac{1}{2},\frac{1}{2}])$, respectively.
\end{itemize}
Then for each label noise standard deviation $\sigma \in \{0,0.25,0.5,1\}$, we create training pairs of the form $\{(\vx_i,f(\vx_i) + \varepsilon_i)\}_{i=1}^{2048}$ where $\varepsilon_i \sim N(0,\sigma^2)$.
We then create training sets of size  $n \in \{64, 128, \ldots, 2048\}$ consisting of $\{(\vx_i,f(\vx_i) + \varepsilon_i)\}_{i=1}^{n}$. This ensures that samples in the training set of size $64$ are a subset of the samples in the training set of size $128$, etc. 

\paragraph{Training and hyperparameter tuning}
For each index rank $r$, dataset size $n$, and label noise standard deviation $\sigma$, we train a model of the form \eqref{eq:L layers nn model} of depth $L$ and with hidden-layer widths all equal to $1000$, starting from PyTorch's default initialization using Adam with a fixed batch size of 64 and the mean-squared error loss. We train with a learning rate of $10^{-4}$ for 60,000 epochs with a weight decay ($\ell_2$-regularization) parameter of $\lambda$ followed by $100$ epochs with a learning rate of $10^{-5}$ and no weight decay. 
This final training period without weight decay ensures the trained networks have small mean-squared error; all models have a final training MSE of no more than $\sigma + 10^{-2}$. The values of the $\ell_2$-regularization term throughout training are plotted in \Cref{fig:traintime}.

\begin{figure}[hp!]
    \centering
    \includegraphics[height=\dimexpr \textheight - 3\baselineskip\relax]{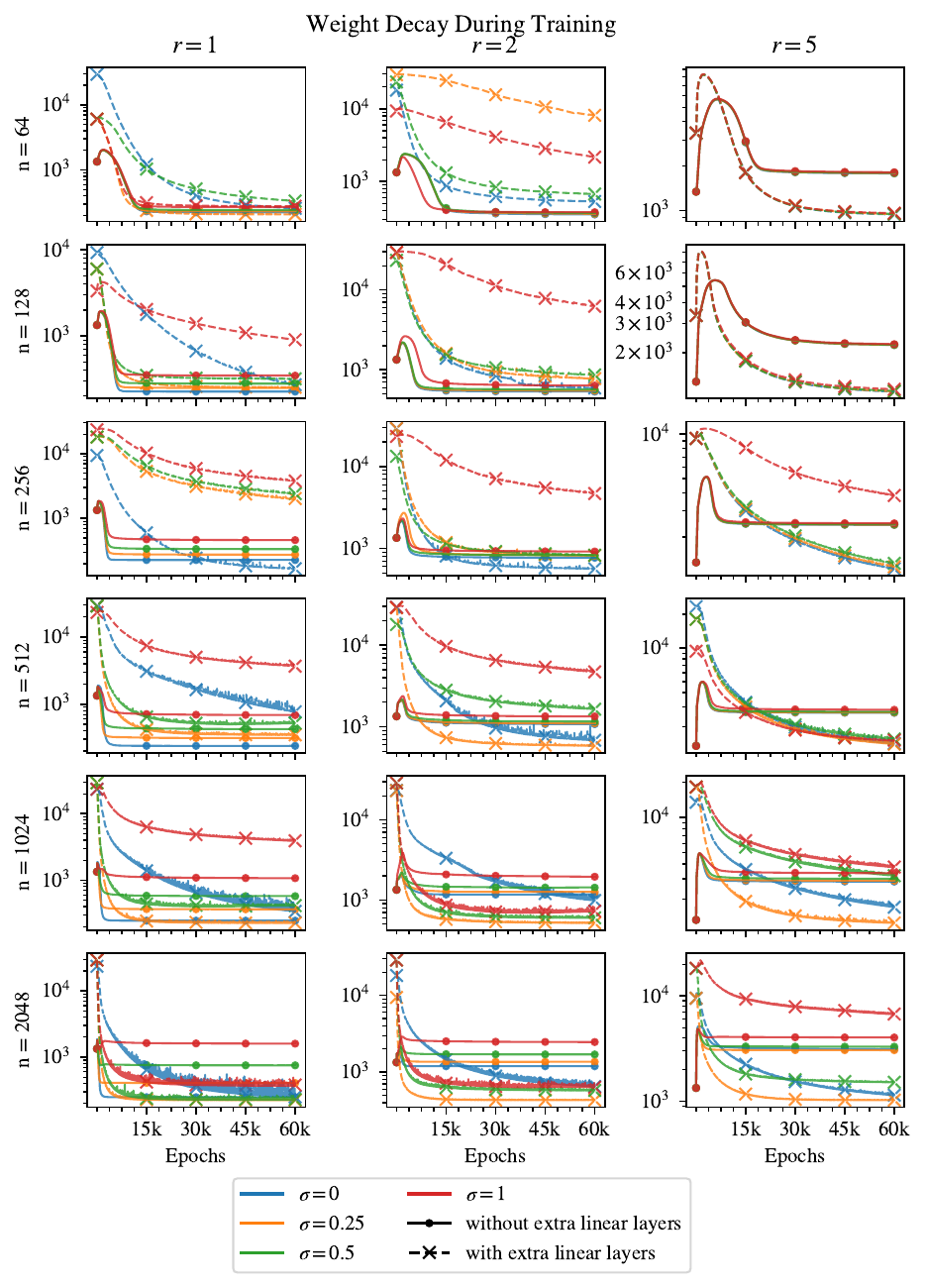}
    \caption{\textbf{Values of the $\ell_2$-regularization term throughout 60,100 training epochs.} Markers are shown every 15k epochs to clarify which lines correspond to models with/without extra linear layers. See \Cref{app:experiment details}.}
    \label{fig:traintime}
\end{figure}

We tune the hyperparameters of depth ($L$) and $\ell_2$-regularization strength ($\lambda$) on a validation set of size 2048 from the same distribution as the training set. 
We use hyperparameters ranges of $L \in \{3,\ldots,9\}$ and $\lambda \in \{10^{-3},10^{-4},10^{-5}\}$.
Models with no linear layers correspond to depth $L = 2$, for which we tune the hyperparameter $\lambda$ in the same way.

\section{Additional Numerical Experiments}

\subsection{Comparison to Training with SGD}
\label{sec:sgd training}
To validate that our empirical findings hold beyond a single training regime, we performed numerical experiments identical to those described in \Cref{sec:experiments} and \Cref{app:experiment details} but using SGD instead of Adam for training. 
We focused only on data from single-index models ($r = 1$) with little to no label noise ($\sigma \in \{0, 0.25\}$).
The same general conclusions hold; in this setting adding linear layers leads to improved generalization (\Cref{fig:SGD Generalization MSE}), a stark singular-value dropoff (\Cref{fig:SGD trained singular values}), and alignment between the principal subspace of the trained model and the true central subspace of $f$ (\Cref{fig:SGD active subspace err}). Interestingly, this is true even though SGD with weight decay does not seem to substantially decrease the $\ell_2$-norm of the parameters during training; see \Cref{fig:SGDtraintime}.

\begin{figure}[ht!]
    \centering
    \includegraphics[width=\columnwidth]{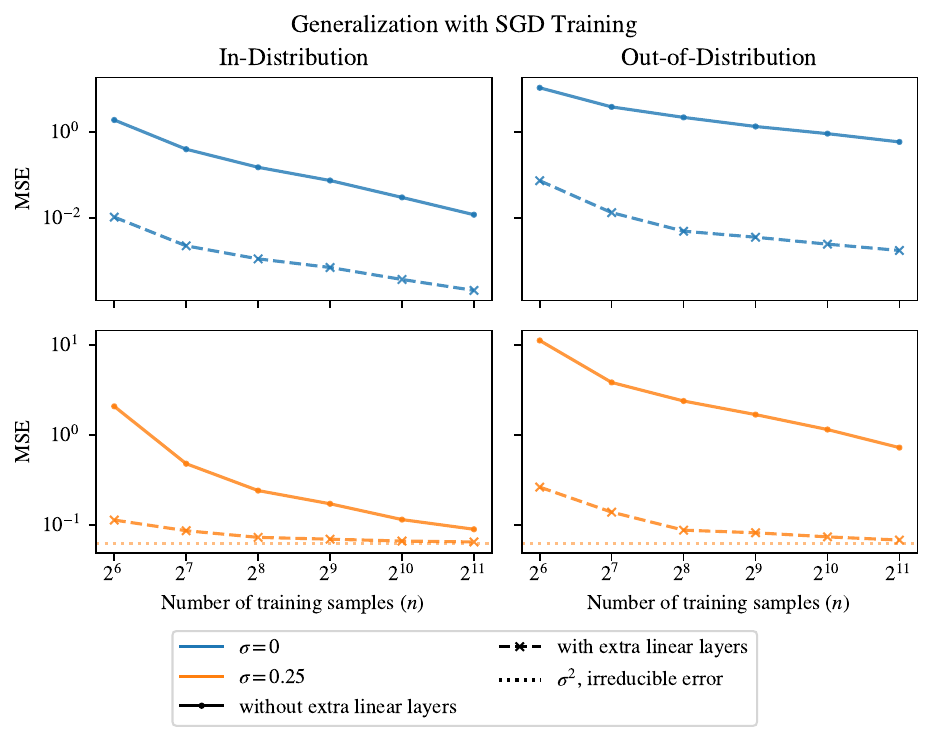}
    \caption{\textbf{Adding linear layers improves generalization on a single-index model when training with SGD.} In-distribution (left) and out-of-distribution (right) generalization performance of networks trained via SGD with or without extra linear layers on data from a single-index model with (bottom) and without (top) label noise. Models trained with extra linear layers demonstrate significantly improved generalization in this setting, even in the presence of label noise. See \Cref{sec:sgd training}.}
    \label{fig:SGD Generalization MSE}
\end{figure}

\begin{figure}[ht!]
    \centering
    \includegraphics[width=\columnwidth]{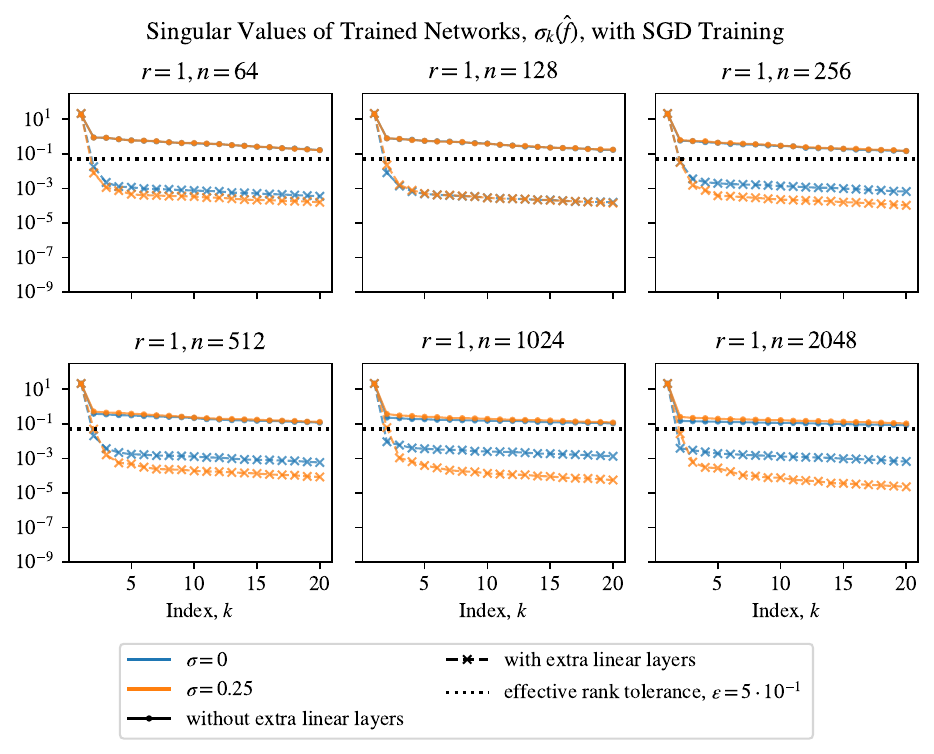}
    \caption{\textbf{Adding linear layers decreases the singular values of trained networks when training with SGD.} Singular values of networks trained via SGD with or without extra linear layers on data from a single-index model with (orange) or without (blue) label noise. Models with extra linear layers exhibit sharper singular value dropoff and have a smaller effective index rank at the $\varepsilon = 5\cdot 10^{-1}$ tolerance level than models without linear layers. Note that the singular value dropoff is less sharp than in models trained with Adam (c.f. \Cref{fig:trained singular values}), and accordingly we use a larger effective rank tolerance in this setting. See \Cref{sec:sgd training}.}
    \label{fig:SGD trained singular values}
\end{figure}

\begin{figure}[ht!]
    \centering
    \includegraphics[width=\columnwidth]{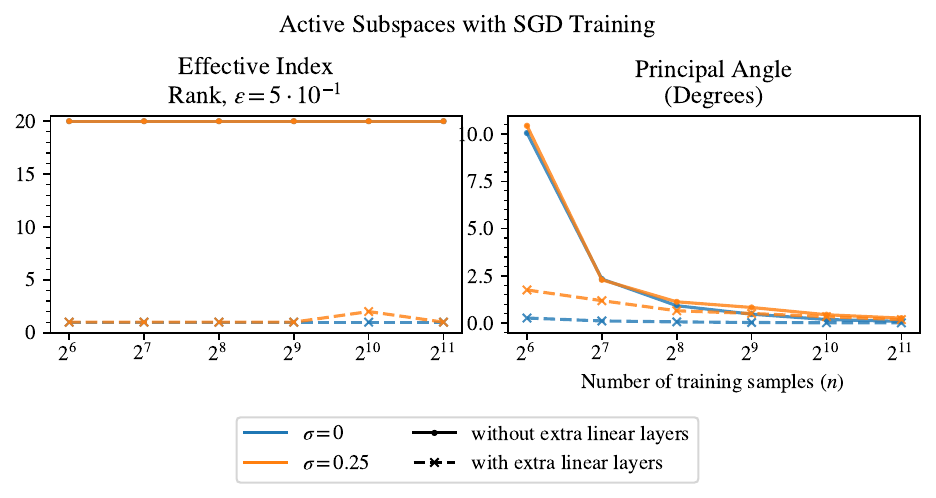}
    \caption{\textbf{When training with SGD, adding linear layers helps find networks with low effective index rank that are aligned with the true principal subspace.} Estimates of the effective index rank (left) and principal subspace alignment (right) of networks trained via SGD with or without extra linear layers on data from a single-index model with (orange) or without (blue) label noise. See \Cref{sec:sgd training}.}
    \label{fig:SGD active subspace err}
\end{figure}

\begin{figure}[ht!]
    \centering
    \includegraphics[width=\columnwidth]{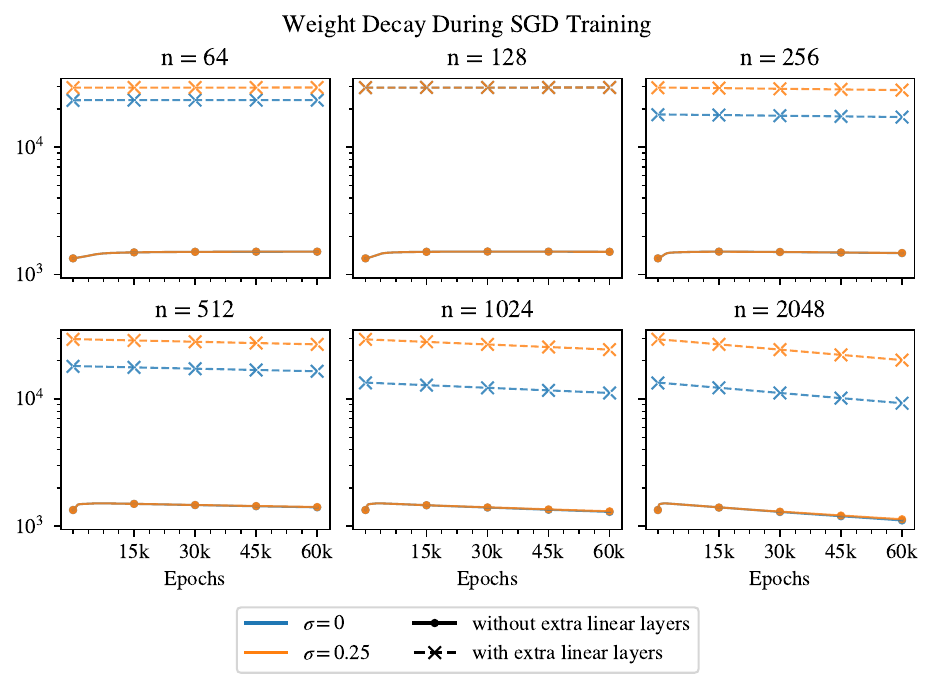}
    \caption{\textbf{Values of the $\ell_2$-regularization term throughout 60,100 training epochs of SGD.} Markers are shown every 15k epochs to clarify which lines correspond to models with/without extra linear layers. See \Cref{sec:sgd training}.}
    \label{fig:SGDtraintime}
\end{figure}

\subsection{Using Deep ReLU Networks on Data From a Single-Index Model}
\label{sec:different architecture experiments}
As discussed in \Cref{sec:discussion}, the inductive bias of adding linear layers to a shallow ReLU network is not directly indicative of the inductive bias of deep ReLU networks. In this section we explore how deep ReLU networks behave when trained on data from a single-index model. We followed the procedure described in \Cref{sec:experiments} and \Cref{app:experiment details} but compare shallow ($L=2$) models and ``linear layers then ReLU" models as studied in this work (i.e., \Cref{eq:L layers nn model}) with deep ReLU models of the form
\begin{align}
\label{eq:L layers deep ReLU model}
    \va^\T[\mW_{L-1}[\cdots [\mW_2[\mW_1 \vx]_+]_+\cdots]_+ + \vb]_+ + c.
\end{align}
We focused only on data from single-index models ($r = 1$) with little to no label noise ($\sigma \in \{0, 0.25\}$).
Interestingly, in this setting deep ReLU models do not experience improved generalization over shallow models (\Cref{fig:othermodels Generalization MSE}) or experience significant EGOP singular-value decay (\Cref{fig:othermodels trained singular values}). 
Though these experiments are fairly small scale, we tentatively conclude that deep ReLU networks do not inherently favor functions with low mixed variation.

However, related work by Jacot \cite{jacot2022implicit,jacot2024bottleneck} has shown that, as depth approaches infinity, the representation cost of deep ReLU networks converges to a distinct notion of non-linear function rank. Empirically, a low-rank structure emerges in such networks, though this low-rankness is not equivalent to the index rank. Understanding the representation costs of general nonlinear deep networks, especially at finite depths, is an open problem.

\begin{figure}[ht!]
    \centering
    \includegraphics[width=\columnwidth]{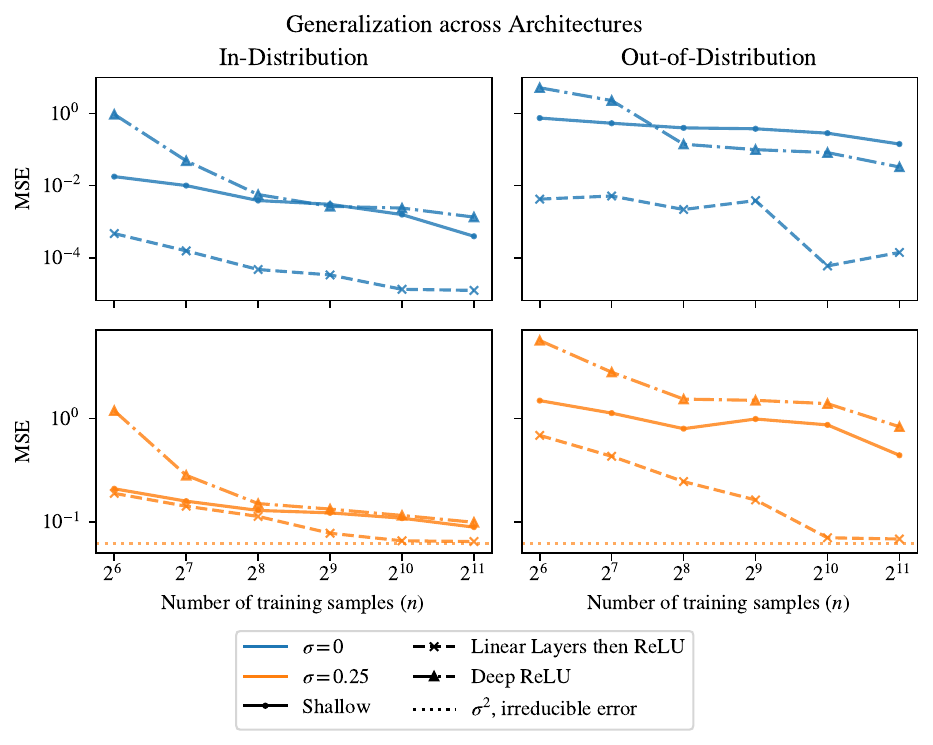}
    \caption{\textbf{Depth does not improve generalization of deep ReLU networks on data from a single-index model.} In-distribution (left) and out-of-distribution (right) generalization performance of a variety of model architectures trained on data from a single-index model with (bottom) and without (top) label noise. Deep ReLU models do not perform better than shallow networks, while the ``linear layers then ReLU" models studied in this work have significantly improved generalization in this setting, even in the presence of label noise. See \Cref{sec:different architecture experiments}.}
    \label{fig:othermodels Generalization MSE}
\end{figure}

\begin{figure}[ht!]
    \centering
    \includegraphics[width=\columnwidth]{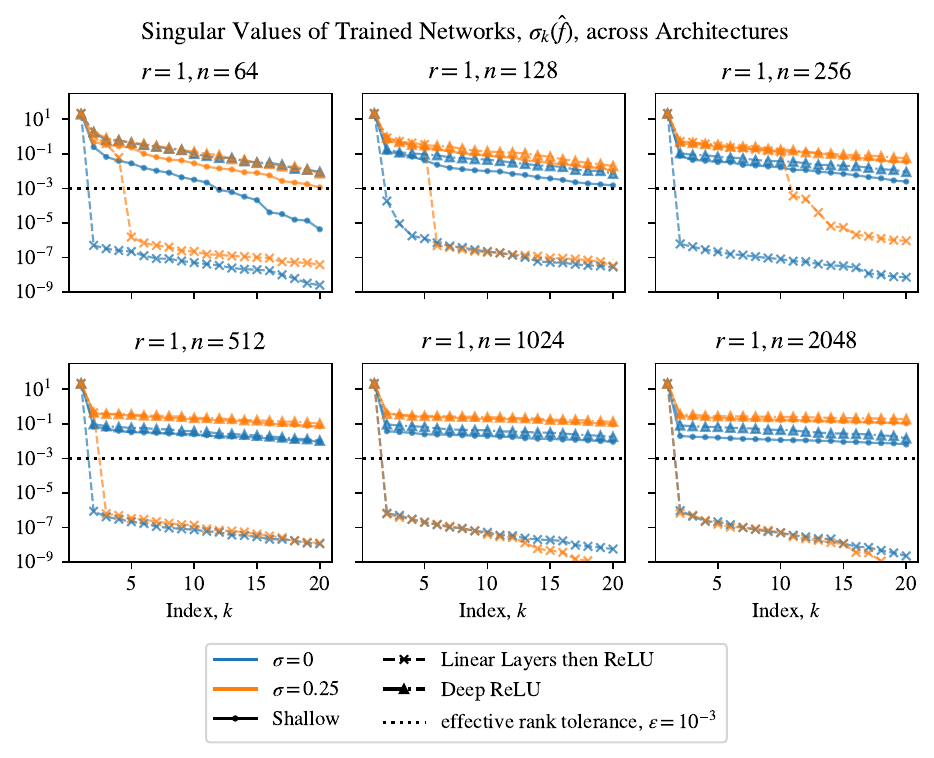}
    \caption{\textbf{Depth does not cause EGOP singular value decay in deep ReLU models.} Singular values of $\gradcov{\hat f}^{1/2}$ for a variety of model architectures trained on data from a single-index model with (orange) and without (blue) label noise. Deep ReLU models do not exhibit dramatic singular value dropoff, but models with extra linear layers do. See \Cref{sec:different architecture experiments}.}
    \label{fig:othermodels trained singular values}
\end{figure}

\begin{figure}[ht!]
    \centering
    \includegraphics[width=\columnwidth]{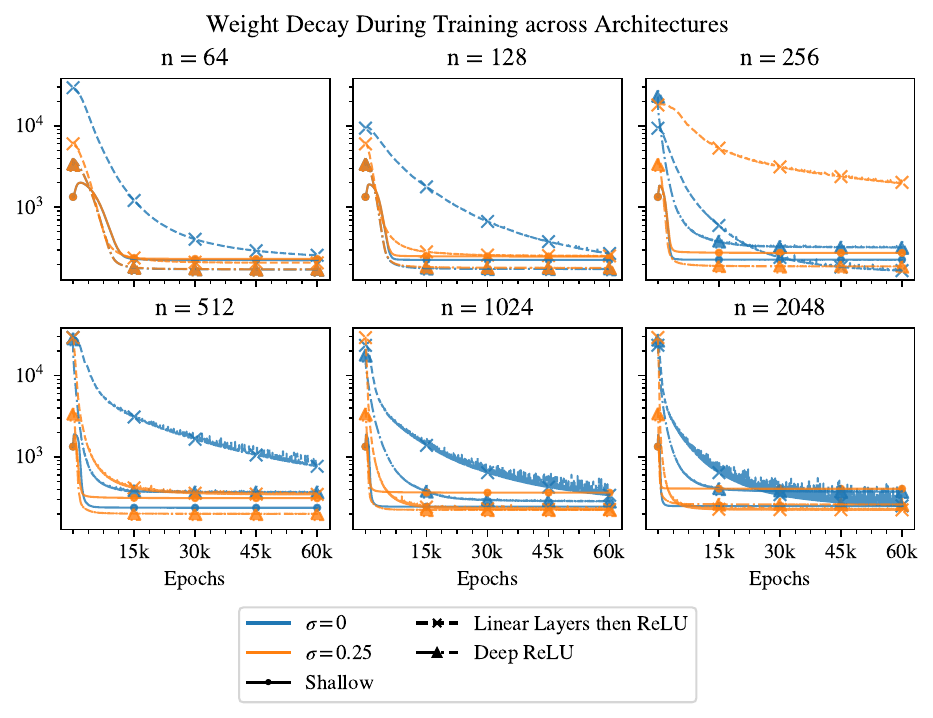}
    \caption{\textbf{Values of the $\ell_2$-regularization term throughout 60,100 training epochs of SGD.} Markers are shown every 15k epochs to clarify which lines correspond to which model architectures. See \Cref{sec:different architecture experiments}.}
    \label{fig:othermodelstraintime}
\end{figure}

\end{document}